\documentclass[final]{cvpr}

\usepackage{times}
\usepackage{epsfig}
\usepackage{graphicx}
\usepackage{amsmath}
\usepackage{amssymb,amsthm}

\usepackage{multirow}
\usepackage{tabularx}
\usepackage{colortbl}
\usepackage{booktabs}
\usepackage[table]{xcolor}
\usepackage{caption}

\usepackage{algorithmicx}
\usepackage{algpseudocode}
\usepackage{algorithm}

\makeatletter
\@namedef{ver@everyshi.sty}{}
\makeatother
\usepackage{pgfplots}
\usepackage{pgfplotstable}
\pgfplotsset{compat=newest}

\usepackage{enumitem}
\usepackage{xr}
\usepackage{xspace}
\usepackage{pifont}

\definecolor{color1}{RGB}{230, 159, 0}
\definecolor{color2}{RGB}{86, 180, 233}
\definecolor{color3}{RGB}{0, 158, 115}
\definecolor{color4}{RGB}{240, 228, 66}
\definecolor{color5}{RGB}{0, 114, 178}
\definecolor{color6}{RGB}{213, 94, 0}
\definecolor{color7}{RGB}{204, 121, 167}

\definecolor{mg}{RGB}{199,246,182}
\definecolor{pa}{RGB}{255,252,209}
\definecolor{cs}{RGB}{209,252,255}

\pgfplotscreateplotcyclelist{colorblind-friendly}{
{color1!60!black,fill=color1},
{color2!60!black,fill=color2},
{color3!60!black,fill=color3},
{color4!60!black,fill=color4},
{color5!60!black,fill=color5},
{color6!60!black,fill=color6},
{color7!60!black,fill=color7},
}

\newcommand{\cmark}{\ding{51}}
\newcommand{\xmark}{\ding{55}}

\newcommand{\CABB}{\textsc{cabb}\xspace}
\newcommand{\Crop}{\textsc{crop}\xspace}
\newcommand{\Full}{\textsc{full}\xspace}

\newcommand{\Su}{\textsc{isus}\xspace}
\newcommand{\Cu}{\textsc{cus}\xspace}

\usepackage{xspace}
\usepackage{bbold}
\usepackage{xcolor}

\newcommand{\vct}[1]{\ensuremath{\boldsymbol{#1}}}

\newcommand{\argmin}{\operatornamewithlimits{\arg\,\min}}

\newtheorem{proposition}{Proposition}

\newcommand{\aka}{\emph{a.k.a.}\xspace}

\usepackage[pagebackref=true,breaklinks=true,colorlinks,bookmarks=false]{hyperref}

\begin{document}

\title{Improving Panoptic Segmentation at All Scales}

\author{Lorenzo Porzi, Samuel Rota Bul\`o, Peter Kontschieder\\
Facebook\\
{\tt\small \{porzi,rotabulo,pkontschieder\}@fb.com}
}

\twocolumn[{%
\renewcommand\twocolumn[1][]{#1}%
\maketitle

\vspace{-20pt}
\begin{center}
  \includegraphics[width=\textwidth]{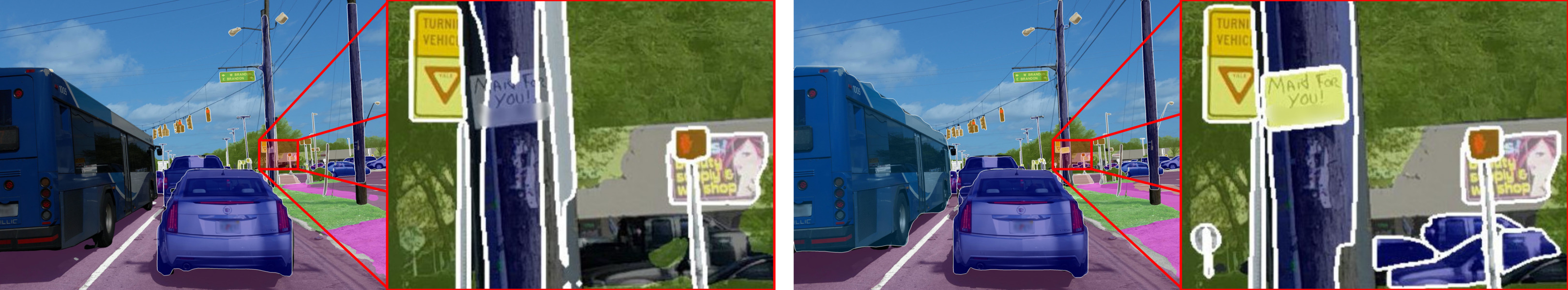}
  \captionof{figure}{Panoptic segmentation on high-resolution natural images is challenged with recognizing objects at a wide range of scales. Standard approaches (left) can struggle when dealing with very small (zoomed detail) or very large objects (bus on the left). By introducing a novel instance scale-uniform sampling strategy and a crop-aware bounding box loss, we are able to improve panoptic segmentation results at all scales (right).}
  \label{fig:teaser}
\end{center}%
}]

\maketitle

\begin{abstract}
Crop-based training strategies decouple training resolution from GPU memory consumption, allowing the use of large-capacity panoptic segmentation networks on multi-megapixel images.
Using crops, however, can introduce a bias towards truncating or missing large objects.
To address this, we propose a novel crop-aware bounding box regression loss (\CABB loss), which promotes predictions to be consistent with the visible parts of the cropped objects, while not over-penalizing them for extending outside of the crop.
We further introduce a novel data sampling and augmentation strategy which improves generalization across scales by counteracting the imbalanced distribution of object sizes.
Combining these two contributions with a carefully designed, top-down panoptic segmentation architecture, we obtain new state-of-the-art results on the challenging Mapillary Vistas (MVD), Indian Driving and Cityscapes datasets, surpassing the previously best approach on MVD by +4.5\% PQ and +5.2\% mAP.
\end{abstract}

\section{Introduction}
\label{sec:intro}

Panoptic segmentation~\cite{Kirillov18} is the task of generating per-pixel, semantic labels for an image, together with object-specific segmentation masks.
It is thus a combination of semantic segmentation and instance segmentation, \ie two long-standing tasks in computer vision that have been traditionally tackled separately.
Due to its importance for tasks like autonomous driving or scene understanding it has recently attracted a lot of interest in the research community.

The majority of deep-learning based panoptic segmentation architectures~\cite{Kirillov19,Por+19_cvpr,Li2018,Xiong_UBER_2019,mohan2020efficientps} proposed a combination of specialized segmentation branches -- one for conventional semantic segmentation and another one for instance segmentation -- followed by a combination strategy to generate a final panoptic segmentation result.
Instance segmentation branches in top-down panoptic architectures are dominantly designed on top of Mask R-CNN~\cite{He2017}, \ie a segmentation extension of Faster R-CNN~\cite{Ren+15} generating state-of-the-art mask predictions for given bounding boxes.
In contrast and more recently, bottom-up panoptic architectures~\cite{cheng2020panoptic,sofiiuk2019adaptis} have emerged but still lag behind in terms of instance segmentation performance.

Panoptic segmentation networks are typically solving multiple tasks (object detection, instance segmentation and semantic segmentation), and are trained on batches of full-sized images.
However, with increasing complexity of tasks and growing capacity of the network backbone, full-image training is quickly inhibited by available GPU memory, despite availability of memory-saving strategies during training like~\cite{RotPorKon18a,micikevicius2018mixed,Gomez2017,mlsys2020_196}.
Obvious mitigation strategies include a reduction of training batch size, downsizing of high-resolution training images, or building on backbones with lower capacity.
These workarounds unfortunately introduce other limitations:
i) Small batch sizes can lead to higher variance in the gradients which will reduce the effectiveness of Batch Normalization~\cite{IofSze15} and consequently the performance of the resulting model.
ii) Reducing the image resolution leads to a loss of fine structures which are known to strongly correlate with objects belonging to the long tail of the label distribution.
Downsampling the images is consequently amplifying already existing performance issues on small and usually underrepresented classes.
iii) A number of recent works~\cite{WangSCJDZLMTWLX19,cheng2019panopticdeeplab,YuanCW20} have shown that larger backbones with sophisticated strategies of maintaining high-resolution features are boosting panoptic segmentation results in comparison to those with reduced capacity.

A possible strategy to overcome the aforementioned issues is to move from full-image-based training to crop-based training.
This was successfully used for conventional semantic segmentation~\cite{RotPorKon18a,Chen2018ECCV,Chen2016}, which is however an easier problem as the task is limited to a per-pixel classification problem.
By fixing a certain crop size the details of fine structures can be preserved, and at a given memory budget, multiple crops can be stacked to form reasonably sized training batches.
For more complex tasks like panoptic segmentation, the simple cropping strategy also affects the performance on object detection and consequently on instance segmentation.
In particular, extracting fixed-size crops from images during training introduces a bias towards truncating large objects, with the likely consequence of underestimating their actual bounding box sizes during inference on full images (see, \eg Fig.~\ref{fig:teaser} left).
Indeed, Fig.~\ref{fig:mvd-areas} (left) shows that the distribution of box sizes during crop-based training on the high-resolution Mapillary Vistas~\cite{Neuhold2017} dataset does not match with the one derived from full-image training data.
In addition, Fig.~\ref{fig:mvd-areas} (right) shows that large objects (based on \# pixels) are drastically underrepresented, which may lead to over-fitting and thus further harming generalization.

In this paper we overcome these issue by introducing two novel contributions: 1) A crop-based training strategy exploiting a \textit{crop-aware loss} function (\CABB) to address the problem of cropping large objects, and 2) Instance scale-uniform (\Su) sampling as data augmentation strategy to combat the imbalance of object scales in the training data.
Our solution enjoys all benefits from crop-based training as discussed above.
In addition, our crop-aware loss incentivizes the model to predict bounding boxes to be consistent with the visible parts of cropped objects, while not over-penalizing predictions outside of the crop.
The underlying intuition is simple: Even if an object bounding box size was modified through cropping, the actual object bounding boxes may be larger than what is visible to the network during training.
By not penalizing hypothetical predictions beyond the visible area of a crop but still within their actual sizes, we can better model the bounding box size distribution given by the original training data.
With \Su we introduce an effective data augmentation strategy to improve feature-pyramid like representations as used for object detection at multiple scales.
It aims at more evenly distributing supervision of object instances during training across pyramid scales, leading to improved recognition performance of instances at all scales during inference.
In the experimental analyses we find that our crop-aware loss function is particularly effective on high-resolution images as available in the challenging Mapillary Vistas~\cite{Neuhold2017}, Indian Driving~\cite{Varma19}, or Cityscapes~\cite{Cordts2016} datasets.

\paragraph{Contributions.} We summarize our contributions to the panoptic segmentation research community as follows.
\begin{itemize}
  \item We introduce a novel, crop-aware training loss applicable to improving bounding box detection in panoptic segmentation networks when training them in a crop-based way.
    At negligible computational overhead ($\sim$10ms per batch) we show how our new loss addresses issues of crop-based training, considerably improving the performance on disproportionately often truncated bounding boxes.
    \item We describe a novel Instance Scale-Uniform Sampling approach to smooth the distribution of object sizes observed by a network at training time, improving its generalization across scales.
    \item We significantly push the state-of-the-art results on the high-resolution Mapillary Vistas dataset, improving on multiple evaluation metrics like Panoptic Quality~\cite{Kirillov18} (+4.5\%) and mean average precision (mAP) for mask segmentation (+5.2\%). We also obtain remarkable performance gains on IDD and Cityscapes, improving PQ by +0.6\% and mAP by +4.1\% and +1.5\%, respectively.
\end{itemize}

\section{Technical Contributions}
\label{sec:method}

In this section we present our main methodological contributions.
In particular, in Sec.~\ref{sec:su} we describe a novel Instance-Scale Uniform Sampling (\Su) approach aimed at reducing the object scale imbalance inherent in high-resolution panoptic datasets.
Sections~\ref{sec:regression},~\ref{sec:cabb} and~\ref{sec:comp_aspects} describe the Crop-Aware Bounding Box (\CABB) loss, which we propose as a mitigation to the bias imposed by crop-based training on the detection of large objects.

\subsection{Instance Scale-Uniform Sampling (\Su)}
\label{sec:su}

\begin{figure*}
\begin{tikzpicture}[baseline]
\pgfplotstableread[col sep=comma, header=false]{data/mvd_crop_iou_by_size.dat}\datatable
\begin{axis}[
    title={Cropped vs. original IoU},
    ylabel=IoU,
    width=0.5\textwidth,
    height=3cm,
    ybar, ymin=0.2, ymax=1,
    ymajorgrids=true,
    ytick={0.2,0.6,1},
    xticklabels from table={\datatable}{0},
    x tick label style={
        rotate=45,
        anchor=east},
    xtick=data,
    xtick pos=left,
    enlargelimits=0.1,
    cycle list name=colorblind-friendly
]
\addplot table [x expr=\coordindex, y=1] {\datatable};
\end{axis}
\end{tikzpicture}
\quad
\begin{tikzpicture}[baseline]
\pgfplotstableread[col sep=comma]{data/mvd_object_sizes.dat}\datatable
\begin{axis}[
    title=Number of objects by scale,
    ylabel=\# objects,
    ytick={0,2e5,4e5},
    xlabel=scale (px),
    xticklabels from table={\datatable}{scale},
    xtick=data,
    ybar,
    ymajorgrids=true,
    enlargelimits=.15,
    width=0.5\textwidth,
    height=3cm,
    cycle list name=colorblind-friendly
    ]
\addplot table [x expr=\coordindex, y expr=\thisrow{number}] {\datatable};
\end{axis}
\end{tikzpicture}
\caption{Left: average intersection over union of cropped bounding boxes \wrt their original extent, computed using the Mapillary Vistas training settings in Sec.~\ref{sec:exp-setup}. Right: distribution of object scales in the Mapillary Vistas training set.}
\label{fig:mvd-areas}
\end{figure*}
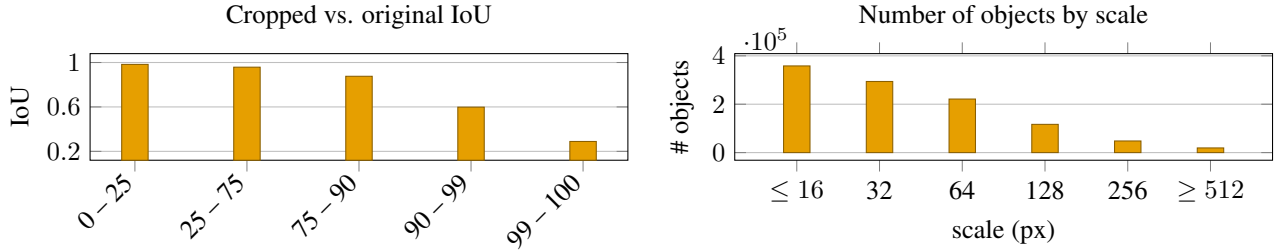

Most top-down panoptic segmentation networks build on top of backbones that produce a ``pyramid'' of features at multiple scales.
At training time, some heuristic rule~\cite{Kirillov19} is applied to split the ground truth instances across the available scales, such that the network is trained to detect small objects using high-resolution features and large objects using low-resolution features.
By sharing the parameters of the prediction modules (\eg the RPN and ROI heads of~\cite{Por+19_cvpr}) across all scales, the network is incentivized to learn scale-invariant features.
When dealing with high-resolution images, however, this approach encounters two major issues: i) the range of object scales can greatly exceed the range of scales available in the feature pyramid, and ii) the distribution of object scales is markedly non-uniform (see Fig.~\ref{fig:mvd-areas}).
While (i) can be partially addressed by adding more feature scales, at the cost of increased memory and computation, (ii) will lead to a strong imbalance in the amount of supervision received by each level of the feature pyramid.

In order to mitigate this imbalance, we propose an extension to the Class-Uniform Sampling (\Cu) approach introduced in~\cite{RotPorKon18a} we coin Instance Scale-Uniform Sampling (\Su).
The standard \Cu data preparation process follows four steps: 1) sample a semantic class with uniform probability; 2) load an image that contains that class and re-scale it such that its shortest side matches a predefined size $s_0$; 3) apply any data augmentation (\eg flipping, random scaling); and 4) produce a random crop from an area of the image where the selected class is visible.
In \Su, we follow the same steps as in \Cu, except that the scale augmentation procedure is made instance-aware.
In particular, when a ``thing'' class is selected in step 1 and after completing step 2, we also sample a random instance of that class from the image and a random feature pyramid level.
Then, in step 3 we compute a scaling factor $\sigma$ such that the selected instance will be assigned to the selected level according to the heuristic adopted by the network being trained.
In order to avoid excessively large or small scale factors, we clamp $\sigma$ to a limited range $r_\text{th}$.
Conversely, when a ``stuff'' class is selected in step 1, we follow the standard scale augmentation procedure, \ie uniformly sample $\sigma$ from a range $r_\text{st}$.
In the long run, \Su will have the effect of smoothing out the object scale distribution, providing more uniform supervision across all scales.

\subsection{Bounding box regression}
\label{sec:regression}

Most top-down panoptic segmentation approaches encode object bounding boxes in terms of offsets with respect to a set of reference boxes~\cite{Li2018,Por+19_cvpr,mohan2020efficientps}.
These reference boxes can be fixed, e.g. the ``anchors'' in the region proposal stage, or be the output of a different network section, e.g. the ``proposals'' in the detection stage.
The goal of a network component that predicts bounding boxes is to regress these offset values given the input image (or derived features thereof).

A ground-truth bounding box $\mathtt G$ is encoded in terms of a center $\vct c_\mathtt G\in\mathbb R^2$ and dimensions $\vct d_\mathtt G\in\mathbb R^2$.
Each ground-truth box is assigned a reference (or anchor) bounding box $\mathtt A$ with center $\vct c_\mathtt A\in\mathbb R^2$ and dimensions $\vct d_\mathtt A\in\mathbb R^2$. The ground truth for the training procedure is then encoded in relative terms and specifically given by $\Delta_\mathtt G=(\vct \delta_\mathtt G, \vct \omega_\mathtt G)$ where
\[
	\vct \delta_\mathtt G = \frac{\vct c_\mathtt G-\vct c_\mathtt A}{\vct d_\mathtt A}\in\mathbb R^2\qquad\text{and}\qquad \vct \omega_\mathtt G = \frac{\vct d_\mathtt G}{\vct d_\mathtt A}\in\mathbb R^2\,.
\]
Here and later, we implicitly assume for notational convenience that operations and functions applied to vectors work element-wise unless otherwise stated.
We will also use the notation $\ominus$ to denote the operation above that returns $\Delta_\mathtt G$ given bounding boxes $\mathtt G$ and $\mathtt A$, \ie $\Delta_\mathtt G=\mathtt G\ominus\mathtt A$.

Similarly, given an anchor bounding box $\mathtt A$ and $\Delta_\mathtt P=(\vct\delta_\mathtt P,\vct\omega_\mathtt P)$, we can recover the predicted bounding box $\mathtt P$ with center $\vct c_\mathtt P$ and dimensions $\vct d_\mathtt P$ as
\[
	\vct c_\mathtt P = \vct c_\mathtt A + \vct \delta_\mathtt P \vct d_\mathtt A\qquad\text{and}\qquad\quad \vct d_\mathtt P = \vct \omega_\mathtt P \vct d_\mathtt A\,.
\]

\paragraph{Standard bounding box loss ~\cite{Ren+15}.}
To train the network, the following per-box loss is minimized over the training dataset:
\begin{equation}
\label{eqn:base-loss}
L_\mathtt{BB}(\Delta_\mathtt P; \Delta_\mathtt G) = \Vert\ell_\beta(\vct \delta_\mathtt P - \vct \delta_\mathtt G) + \ell_\beta(\log \vct \omega_\mathtt P - \log \vct \omega_\mathtt G)\Vert_1\,,
\end{equation}
where $\Vert\cdot\Vert_1$ is the 1-norm and $\ell_\beta$ denotes the Huber (\emph{a.k.a.} smooth-L1) norm with parameter $\beta>0$, \ie
\begin{equation*}
  \ell_\beta(z) = \begin{cases}
    \frac{1}{2\beta} z^2 & |z| \leq \beta \\
    |z| - \frac{\beta}{2} & \text{otherwise,}
  \end{cases}
\end{equation*}
and $|z|$ gives the absolute value of $z$.

\subsection{Crop-Aware Bounding Box (\CABB)}
\label{sec:cabb}

In a standard crop-based training, a ground-truth bounding box $\mathtt G$ from the original image that overlaps with the cropping area $\mathtt C$ is typically cropped yielding a new bounding box denoted by $\mathtt G|_\mathtt C$.
\footnote{When masks are available like in instance or panoptic segmentation, the cropping operation is performed at the mask level and the bounding box is recomputed a posteriori. We implicitly assume that this is the case if a ground-truth mask is available for $\mathtt G$.}
Accordingly, the actual ground-truth $\Delta_\mathtt G$ that is used in the loss \eqref{eqn:base-loss} is the result of $\Delta_\mathtt G=\mathtt G|_\mathtt C \ominus \mathtt A$.
Training with this modified ground-truth, however, poses some issues, namely a bias towards cutting or missing big objects at inference time (see, \eg, Fig.~\ref{fig:teaser}~and~\ref{fig:qual-results}).

The solution we propose in this work consists in relaxing the notion of ground-truth bounding box $\mathtt G$ into a set of ground-truth boxes that coincide with $\mathtt G|_\mathtt C$ after the cropping operation. We denote by $\rho(\mathtt G,\mathtt C)$ the function that computes this set for given ground-truth box $\mathtt G$ and cropping area $\mathtt C$, \ie
\[
	\rho(\mathtt G,\mathtt C)=\{\mathtt X\in\mathcal B\,:\,\mathtt X|_\mathtt C = \mathtt G|_\mathtt C\}\,,
\]
where $\mathtt X$ runs over all possible bounding boxes $\mathcal B$. We refer to $\rho(\mathtt G,\mathtt C)$ as a Crop-Aware Bounding Box (\CABB) that in fact is a set of bounding boxes (see also Fig.~\ref{fig:cabb}).
If the ground-truth bounding box $\mathtt G$ is strictly contained in the crop area then our \CABB boils down to the original ground truth, for $\rho(\mathtt G,\mathtt C)=\{\mathtt G\}$ in that case.%
\footnote{
To simplify the description, we deliberately neglect the fact that a bounding box strictly contained in the original image and touching the boundary of the crop area should not be extended beyond the crop. However, our approach can be easily adapted to address these edge cases.
}
Since we will use a representation for bounding boxes relative to some anchor box $\mathtt A$ we introduce also the notation $\rho_\mathtt A(\mathtt G,\mathtt C)$, which returns the same set as above but with elements expressed relative to $\mathtt A$, \ie $\rho_\mathtt A(\mathtt G,\mathtt C)=\{\mathtt X\ominus\mathtt A\,:\,\,\mathtt X\in\rho(\mathtt G,\mathtt C)\}$.

\begin{figure}[t!]
	\centering
	\includegraphics[width=.6\columnwidth]{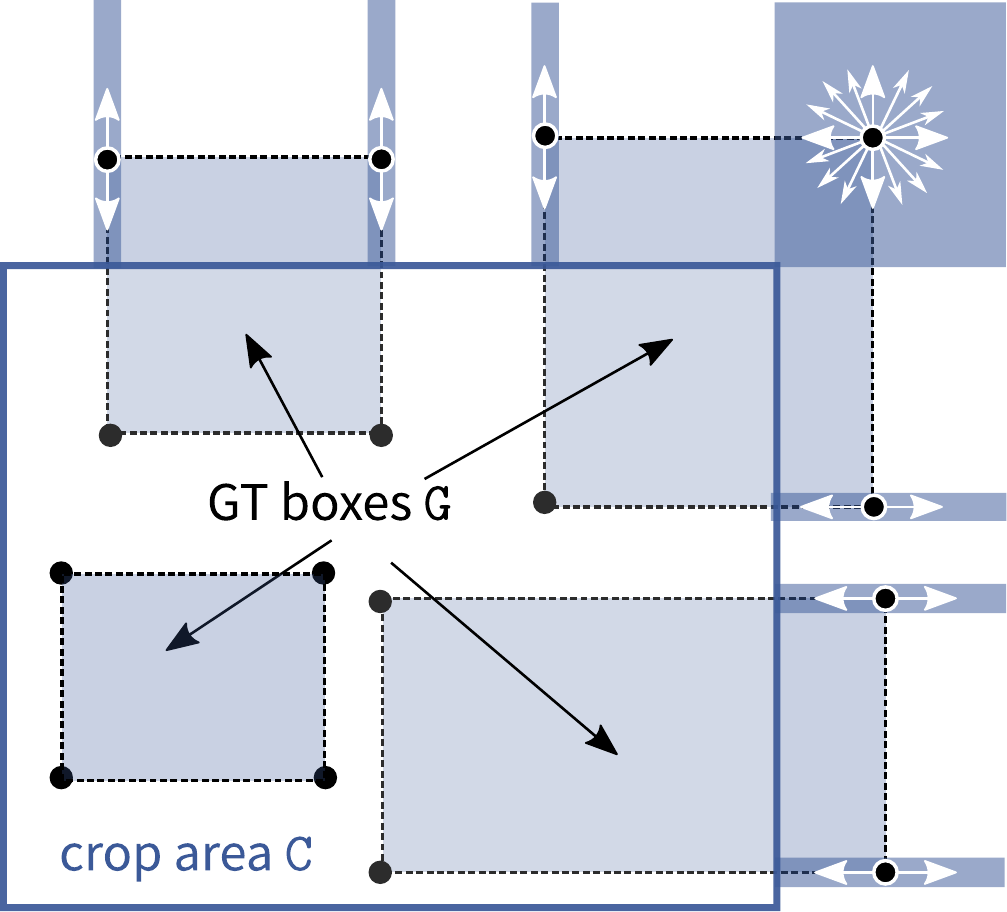}
	\caption{Example of Crop-Aware Bounding Boxes (\CABB). We show $4$ ground-truth boxes, three of which fall partially outside the crop area. The corresponding set $\rho(\mathtt G, \mathtt C)$, \aka \CABB, consists of all rectangular bounding boxes that can be formed by moving the white-bordered corners within the feasible areas (depicted in blue). Note that the areas extend to infinity but are truncated here.}
  \label{fig:cabb}
\end{figure}

\paragraph{Crop-aware bounding box loss.} In order to exploit the proposed, relaxed notion of ground-truth bounding box, we introduce the following new loss function for a given ground-truth box $\mathtt G$, anchor box $\mathtt A$ and crop area $\mathtt C$:
\begin{equation}\label{eq:new_loss}
\begin{aligned}
	L_\mathtt{CABB}({\Delta_\mathtt P}) = & \min_{{\Delta}} L_\mathtt{BB}({\Delta_\mathtt P}; {\Delta})\,, \\
 \text{s.t}\quad & \Delta\in\rho_\mathtt A(\mathtt G,\mathtt C)\,.
\end{aligned}
\end{equation}
Any bounding box in $\rho(\mathtt G,\mathtt C)$ is compatible with the cropped ground-truth box we observe and thus could be potentially a valid prediction. To disambiguate, our new loss favours the solution closer to the actual prediction from the network in order to enforce a smoother training dynamic.
Since the ground-truth box that is typically adopted for the standard loss in \eqref{eqn:base-loss} belongs to the feasible set of the minimization in our new loss, we have that $L_\mathtt{CABB}$ lower bounds $L_\mathtt{BB}$.

\subsection{Computational Aspects}\label{sec:comp_aspects}
This section focuses on the computational aspects of our new loss. In particular, we will address the problem of evaluating it by solving the internal minimization as well as computing the gradient.

The minimization problem that is nested into our new loss has no straightforward solution, since it is
neither convex nor quasi-convex and in general, local, non-global solutions might exist.
Its feasible set is convex in $\Delta=(\vct \delta,\vct \omega)$ since it can be written in terms of linear equalities and inequalities.
Each dimension gives rise to an independent set of constraints and since also the objective function is separable with respect to dimension-specific variables, we have that the whole minimization problem
can be separated into two independent minimization problems involving only dimension-specific variables.

\paragraph{Feasible set.}
Assume without loss of generality that the cropping area $\mathtt C$ is a box with top-left coordinate $(0,0)$ and bottom-right coordinate $\vct d_\mathtt C\in\mathbb R^2$.
Then the feasible set of each dimension-specific minimization problem can be written as:
\begin{itemize}
	\item $\delta-\frac{\omega}{2}\leq -\frac{c_\mathtt A}{d_\mathtt A}$ if $c_\mathtt G\leq\frac{d_\mathtt G}{2}$ else $\delta-\frac{\omega}{2}=\delta_\mathtt G-\frac{\omega_\mathtt G}{2}$ and
	\item $\delta+\frac{\omega}{2}\geq \frac{d_\mathtt C-c_\mathtt A}{d_\mathtt A}$ if $c_\mathtt G\geq d_\mathtt C-\frac{d_\mathtt G}{2}$ else $\delta+\frac{\omega}{2}=\delta_\mathtt G+\frac{\omega_\mathtt G}{2}$,
\end{itemize}
where we dropped the boldface style from the vector-valued variables to emphasize that the constraint is specified for a single dimension.

\paragraph{Optimization problem.}
We will now enumerate the different cases characterizing the feasible set and for each of them we will provide the dimension-specific optimization problem that should be solved. Akin to the feasible set above, all variables involved from here on refer implicitly to a single dimension.
\begin{itemize}
	\item If $\frac{d_\mathtt G}{2}<c_\mathtt G<d_\mathtt C-\frac{d_\mathtt G}{2}$ then $\Delta^\star=(\delta_\mathtt G,\omega_\mathtt G)$ is the solution to the minimization problem in \eqref{eq:new_loss} for the dimension under consideration, since the feasible set is singleton in this case.
\item If $c_\mathtt G>\frac{d_\mathtt G}{2}$ and $c_\mathtt G\geq d_\mathtt C-\frac{d_\mathtt G}{2}$, we obtain an optimization problem in the variable $\omega$ of the form
	\begin{equation}\label{eq:prob_con}\tag{$O_1$}
	\begin{aligned}
		\min_{\omega}\quad&\ell_\beta\left(\frac{\omega-\hat\omega}{2}\right) + \ell_\beta(\log(\omega)-\log(\omega_\mathtt P))\\
		\text{s.t.}\quad&\omega\geq b_1-a_1\,,
	\end{aligned}
\end{equation}
where $a_1=\delta_\mathtt G-\frac{\omega_\mathtt G}{2}$, $b_1=\frac{d_\mathtt C-c_\mathtt A}{d_\mathtt A}$ and $\hat \omega=2(\delta_\mathtt P-a_1)$.
If $w^\star$ is a solution to \eqref{eq:prob_con} then $\Delta^\star=(a_1+\frac{\omega^\star}{2},\omega^\star)$ is a solution to the minimization problem in \eqref{eq:new_loss} for the dimension under consideration.
\item If $c_\mathtt G\leq\frac{d_\mathtt G}{2}$ and $c_\mathtt G< d_\mathtt C-\frac{d_\mathtt G}{2}$, we obtain an optimization problem like \eqref{eq:prob_con} but with
	$a_1=-\frac{c_\mathtt A}{d_\mathtt A}$, $b_1=\delta_\mathtt G+\frac{\omega_\mathtt G}{2}$ and $\hat \omega=2(b_1-\delta_\mathtt P)$.
If $w^\star$ is a solution to \eqref{eq:prob_con} under this parametrization then $\Delta^\star=(b_1-\frac{\omega^\star}{2},\omega^\star)$ is a solution to the minimization problem in \eqref{eq:new_loss} for the dimension under consideration.
\item
If $d_\mathtt C-\frac{d_\mathtt G}{2}\leq c_\mathtt G\leq\frac{d_\mathtt G}{2}$ then we obtain an optimization problem of the form
\begin{equation}\label{eq:prob}\tag{$O_2$}
	\begin{aligned}
		\min_{\delta,\omega}\quad&\ell_\beta(\delta-\delta_\mathtt P) + \ell_\beta(\log(\omega)-\log(\omega_\mathtt P))\\
		\text{s.t.}\quad&\delta-\frac{\omega}{2}\leq a_2\,,\quad \delta+\frac{\omega}{2}\geq b_2\,,
	\end{aligned}
\end{equation}
where $a_2=-\frac{c_\mathtt A}{d_\mathtt A}$ and $b_2=\frac{d_\mathtt C-c_\mathtt A}{d_\mathtt A}$. Solutions to \eqref{eq:prob} map directly to solutions to \eqref{eq:new_loss} for the dimension under consideration.
\end{itemize}

We focus now on finding the solution to the optimization problems \eqref{eq:prob_con} and \eqref{eq:prob}.

\paragraph{Solution to \eqref{eq:prob_con}.} As mentioned before, the optimization problem in \eqref{eq:new_loss} is in general non-convex and might have multiple local minima. The same holds true for the problem in \eqref{eq:prob_con} despite having a single variable.
Nonetheless, we devised an ad-hoc solver for this problem that allows to quickly converge to a global solution under the desired precision.
We provide the details in Appendix~\ref{sec:algos}.

\paragraph{Solution to \eqref{eq:prob}.}
To solve this problem we break it down into cases.
We start by noting that the solution to the unconstrained optimization problem is trivially given by $\delta^\star = \delta_\mathtt P$ and $\omega^\star=\omega_\mathtt P$, because $0$ is the minimizer of $\ell_\beta$. The solution $\Delta^\star=(\delta^\star,\omega^\star)$ is valid for \eqref{eq:prob} if it satisfies the constraints, but this is easy to check by substitution. If this is the case, we found the solution, otherwise
no solution exists in the interior of the feasible set (see Prop.~\ref{prop:unconstrained} in Appendix~\ref{sec:proofs}), but lies along the boundary of the feasible set.
Accordingly, we start by forcing the first constraint to be active. This yields an instance of \eqref{eq:prob_con} with $a_1=a_2$, $b_1=b_2$ and $\hat \omega=2(\delta_\mathtt P-a_2)$, which can be solved using the algorithm from Appendix~\ref{sec:algos}, yielding $\omega_1^\star$. By substituting it into the activated constraint we obtain the other variable $\delta_1^\star=a_2+\frac{\omega_1^\star}{2}$.
Next, we move to activating the second constraint. This yields again an instance of the same optimization problem with the only difference being $\hat \omega=2(b_2-\delta_\mathtt P)$. Again we solve it obtaining $\omega_2^\star$ and by substitution into the activated constraint we get $\delta_2^\star=b_2-\frac{\omega_2^\star}{2}$. We finally retain the solution among $(\delta_1^\star,\omega_1^\star)$ and $(\delta_2^\star,\omega_2^\star)$ yielding the lowest objective. See Alg.~\ref{alg:prob} for further details.

\paragraph{Gradient.}
For the sake of training a neural network, we are interested in computing gradients of the new loss function, which exhibits a nested optimization problem.
The following result shows that the derivative of the new loss function is equivalent to the derivative of the original one, with the ground-truth box replaced (as a constant) by the solution to the internal minimization problem.
In general the solution to the internal minimization problem is a function of $\Delta_\mathtt P$ but the following result states that no gradient term is originated from this dependency.
This is indeed a direct consequence of the envelope theorem~\cite{Afr71}.

\begin{proposition}
	Let $\phi$ be a function returning the minimizer in \eqref{eq:new_loss} given $\Delta_\mathtt P$, \ie $L_\mathtt{CABB}(\Delta_\mathtt P)=L_\mathtt{BB}(\Delta_\mathtt P,\phi(\Delta_\mathtt P))$ holds for any $\Delta_\mathtt P$. Then
	\[
		\frac{d}{d\Delta_\mathtt P}L_\mathtt{CABB}(\Delta_\mathtt P)=\left.\frac{\partial}{\partial\Delta_\mathtt P}L_\mathtt{BB} (\Delta_{\mathtt P},\Delta)\right|_{\Delta=\phi(\Delta_\mathtt P)}\,.
	\]
\end{proposition}

\section{Related Works}
\label{sec:related}

After scrutinizing the literature, we have found no other work directly addressing the specific challenges of training panoptic segmentation networks on high-resolution data, nor the bias introduced by crop-based training.
Indeed, to our knowledge, we are tackling these issues for the first time.
In the literature we find several methods for panoptic segmentation that are architecture-wise compatible with our \CABB loss and \Su, among which we have EfficientPS~\cite{mohan2020efficientps}, AUNet~\cite{Li+2019}, TASCNet~\cite{Li2018}, Panoptic-FPN~\cite{Kirillov19}, UPSNet~\cite{Xiong_UBER_2019} and Seamless Scene Segmentation~\cite{Por+19_cvpr}, to mention a few.
Indeed, those approaches rely on the computation of bounding boxes at some stage, and employ network backbones that produce multi-scale feature pyramids.
Among them, only the first two report crop-based training results in the original work, while the remaining ones report full-image training results.
This however does not mean that the latter approaches would not benefit from crop-based trainings.
Indeed, in this work, we perform experiments using Seamless Scene Segmentation as baseline and show that there is significant improvements deriving from a crop-based training protocol.
Other panoptic segmentation methods that benefit from crop-based training are AdaptIS~\cite{sofiiuk2019adaptis}, DeeperLab~\cite{Yang_Google_2019} SSAP~\cite{Gao+19} and Panoptic-Deeplab~\cite{cheng2020panoptic}.
The latter approaches however are neither based on bounding boxes nor employ feature pyramids, thus our contributions do not directly apply to them.
More broadly, recent works dealing with high-resolution image data include RefineNet~\cite{Lin_2017_CVPR} or CascadePSP~\cite{CascadePSP2020}, which however address the task of conventional semantic segmentation rather than Panoptic segmentation.

\section{Experimental Results}
\label{sec:exp}

\begin{table*}[tbh]
\centering
\begin{tabular}{l|c|c|ccc|cc|ccc|c}
\toprule
  Network & C & Pre-training & PQ & PQ$^\text{th}$ & PQ$^\text{st}$ & mAP & mIoU & PC & PC$^\text{th}$ & PC$^\text{st}$ & PQ$^\dagger$ \\
\midrule
  \rowcolor{mg} TASCNet~\cite{Li2018}
    & \xmark & I   & 32.6 & 31.1 & 34.4 & 18.6 & --   & --   & --   & --   & --   \\
  \rowcolor{mg} AdaptIS~\cite{sofiiuk2019adaptis}
    & \cmark & I   & 35.9 & 31.5 & --   & --   & --   & --   & --   & --   & --   \\
  \rowcolor{mg} Seamless~\cite{Por+19_cvpr}
    & \xmark & I   & 37.7 & 33.8 & 42.9 & 16.4 & 50.4 & --   & --   & --   & --   \\
  \rowcolor{mg} Deeplab, X71~\cite{cheng2020panoptic}
    & \cmark & I   & 37.7 & 30.4 & \textbf{47.4} & 14.9 & 55.3 & --   & --   & --   & --   \\
  \rowcolor{mg} EfficientPS~\cite{mohan2020efficientps}
    & \cmark & I   & 38.3 & 33.9 & 44.2 & 18.7 & 52.6 & --   & --   & --   & --   \\
  \rowcolor{mg} Deeplab, HR48~\cite{cheng2020panoptic}
    & \cmark & I   & 40.6 & --   & --   & 17.8 & \textbf{57.6} & --   & --   & --   & --   \\
\midrule
  \rowcolor{mg} Seamless~\cite{Por+19_cvpr} + \Crop
    & \cmark & I   & 39.2 & 36.5 & 42.8 & 19.0 & 50.8 & 48.8 & 41.2 & 59.0 & 41.5 \\
  \rowcolor{mg} Seamless~\cite{Por+19_cvpr} + \CABB + \Su
    & \cmark & I   & 40.5 & 38.0 & 43.7 & 19.4 & 51.0 & 50.7 & 43.1 & 60.8 & 42.9 \\
\midrule
  \rowcolor{mg} \Full
    & \xmark & I   & 39.4 & 34.0 & 46.5 & 16.2 & 54.4 & 55.2 & 49.7 & 62.4 & 39.5 \\
  \rowcolor{mg} \Crop
    & \cmark & I   & 43.6 & 41.9 & 45.9 & 22.3 & 54.9 & 56.2 & 52.4 & 61.2 & 45.7 \\
  \rowcolor{mg} \Crop + \CABB
    & \cmark & I   & 44.5 & 42.5 & 47.0 & 23.0 & 55.4 & 57.4 & 54.2 & 61.6 & 46.3 \\
  \rowcolor{mg} \Crop + \Su
    & \cmark & I   & 44.7 & 43.1 & 46.9 & 23.0 & 56.3 & 59.4 & 56.1 & 63.7 & 46.9 \\
  \rowcolor{mg} \Crop + \CABB + \Su
    & \cmark & I   & \textbf{45.1} & \textbf{43.4} & \textbf{47.4} & \textbf{23.9} & 56.3 & \textbf{60.4} & \textbf{57.2} & \textbf{64.6} & \textbf{47.2} \\
\midrule
  \rowcolor{pa} Seamless~\cite{Por+19_cvpr}
    & \xmark & I   & 47.7 & 48.9 & 47.1 & 30.1 & 69.6 & --   & --   & --   & --   \\
  \rowcolor{pa} EfficientPS~\cite{mohan2020efficientps}
    & \cmark & I   & 50.1 & 50.7 & \textbf{49.8} & 31.6 & \textbf{71.3} & -- & -- & -- & -- \\
\midrule
  \rowcolor{pa} \Full
    & \xmark & I   & 49.1 & 51.0 & 48.1 & 32.3 & 69.0 & 71.0 & 76.2 & 68.3 & 50.5 \\
  \rowcolor{pa} \Crop
    & \cmark & I   & 50.3 & 52.5 & 49.1 & 35.3 & 69.7 & 70.8 & 73.8 & 69.2 & 51.4 \\
  \rowcolor{pa} \Crop + \CABB + \Su
    & \cmark & I   & \textbf{50.7} & \textbf{52.9} & 49.5 & \textbf{35.7} & 70.4 & \textbf{72.8} & \textbf{78.1} & \textbf{70.0} & \textbf{51.9} \\
\midrule
  \rowcolor{cs} Seamless~\cite{Por+19_cvpr}
    & \xmark & I, V & 65.0 & 60.7 & 68.0 & --   & 80.7 & --   & --   & --   & --   \\
  \rowcolor{cs} Deeplab, X71~\cite{cheng2020panoptic}
    & \cmark & I, V & 65.3 & --   & --   & 38.8 & 82.5 & --   & --   & --   & --   \\
  \rowcolor{cs} EfficientPS~\cite{mohan2020efficientps}
    & \cmark & I, V & 66.1 & \textbf{62.7} & 68.5 & 41.9 & 81.0 & --   & --   & --   & --   \\
\midrule
  \rowcolor{cs} \Full
    & \xmark & I, V & 66.0 & 61.7 & 69.1 & 39.5 & 64.2 & 80.8 & 79.9 & 81.4 & 64.2 \\
  \rowcolor{cs} \Crop
    & \cmark & I, V & 66.6 & 61.1 & 69.5 & 42.2 & 81.7 & 81.3 & 80.0 & 82.3 & 64.4 \\
  \rowcolor{cs} \Crop + \CABB + \Su
    & \cmark & I, V & \textbf{66.7} & 62.4 & \textbf{69.9} & \textbf{43.4} & \textbf{82.6} & \textbf{82.6} & \textbf{82.4} & \textbf{82.7} & \textbf{65.1} \\
\bottomrule
\end{tabular}
\caption{State of the art results on \colorbox{mg}{Mapillary Vistas (top)}, the \colorbox{pa}{Indian Driving Dataset (middle)}, and \colorbox{cs}{Cityscapes (bottom)} compared with variants of our network. A \cmark symbol in column ``C'' indicates crop-based training. ``Deeplab'' abbreviates Panoptic Deeplab~\cite{cheng2020panoptic}. ``I'' and ``V'' are used to indicate pre-training on ImageNet and Mapillary Vistas, respectively.}
\label{tab:mvd-detail1}
\end{table*}

We evaluate our proposed \CABB loss on the three largest publicly available, high-resolution panoptic segmentation datasets: Mapillary Vistas~\cite{Neuhold2017} (MVD), the Indian Driving Dataset~\cite{Varma19} (IDD) and Cityscapes~\cite{cordts2016cityscapes} (CS).
MVD comprises 18k training, 2k validation, and 5k testing images, with resolutions ranging from 2 to 22 Mpixels and averaging $8.8$ Mpixels, and annotations covering 65 semantic classes, 37 of which instance-specific.
IDD comprises 7k training, 1k validation, and 2k testing images, most captured at a $2$ Mpixels resolution and annotated with 26 semantic classes, 9 of them instance-specific.
Cityscapes comprises 3k training, 500 validation, and 1.5k testing images, captured at $2$ Mpixels resolutions and annotated with 19 classes, 8 of which instance-specific.
Next, we present detailed ablation studies and a comparison with recent state-of-the-art panoptic segmentation approaches.

\subsection{Network and Training Details}
\label{sec:exp-setup}

\begin{figure}
\centering
\includegraphics[width=\columnwidth]{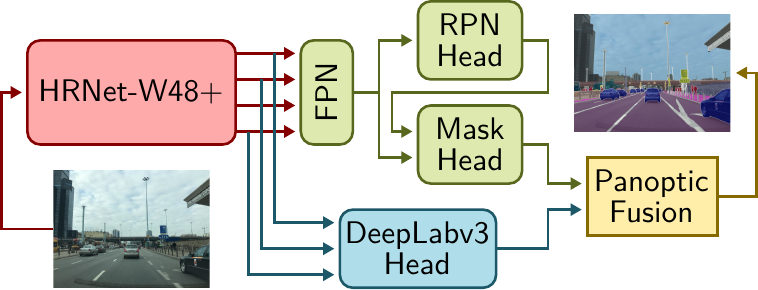}
\caption{Overview of the main functional blocks of our network.
  Red: network body, \ie HRNet-W48+.
  Green: instance segmentation section, composed of an FPN module followed by a Region Proposal Head (RPH) and a mask segmentation head.
  Blue: semantic segmentation section, \ie DeepLabv3 head.
  Yellow: final panoptic fusion step.}
\label{fig:network}
\end{figure}

Our \CABB loss and \Su, described in Sec.~\ref{sec:cabb}, can be used in most top-down panoptic segmentation networks.
To evaluate their effects, however, we focus our attention on a specific architecture, carefully crafted to achieve state-of-the-art performance on high-resolution datasets already \emph{without} using either.
In particular, we follow the general framework of Seamless-Scene-Segmentation~\cite{Por+19_cvpr}, with several modifications described below (see Fig.~\ref{fig:network}).
First, we replace the ResNet-50 ``body'' with HRNetV2-W48+~\cite{WangSCJDZLMTWLX19,cheng2020panoptic}, a specialized backbone which preserves high-resolution information from the image to the final stages of the network.
Second, we replace the ``Mini-DL'' segmentation head from~\cite{Por+19_cvpr} with a DeepLabV3+~\cite{chen2018encoderdecoder} module, connected to the HRNetV2-W48+ body as described in~\cite{cheng2020panoptic}.
As in~\cite{Por+19_cvpr}, we apply synchronized InPlace-ABN~\cite{RotPorKon18a} throughout the network.
Finally, \CABB loss is used to replace the standard bounding box regression loss both in the region proposal and object detection modules.

We train our networks with stochastic gradient descent on 8 NVidia V100 GPUs with 32GB of memory.
The HRNetV2-W48+ backbone is initialized from an ImageNet pre-training in the MVD and IDD experiments, while the Cityscapes networks are fine-tuned from their MVD-trained counterparts.
We fix the crop size to $1024\times 1024$ for MVD, and to $512\times 512$ for IDD and Cityscapes due to their lower resolution, while inference is always performed on full images.
Average inference time on MVD is $\sim1.2$s per image.
To reduce inter-run variability and obtain more comparable results, we fix all sources of randomness that can be easily controlled, resulting in the same sequence of images and initial network weights across all our trainings.
For a detailed breakdown of the training hyper-parameters refer to Appendix~\ref{sec:params}.

\subsection{Comparison with State of the Art}

\label{sec:exp-sota}
We provide a comparison of results in Table~\ref{tab:mvd-detail1}, with baselines including methods trained on full images (TASCNet~\cite{Li2018}, Seamless~\cite{Por+19_cvpr}) and crops (AdaptIS~\cite{sofiiuk2019adaptis}, EfficientPS~\cite{mohan2020efficientps}, Panoptic Deeplab~\cite{cheng2020panoptic}), as well as multiple different backbones (EfficientNet in EffcientPS, ResNet-50 in Seamless and TASCNet, ResNeXt-101 in AdaptIS, Xception-71 and HRNet-W48+ in Panoptic Deeplab).
We consider several different variants of our network:
(i) one using the standard bounding box regression loss and \Cu, trained either on full images (\Full) or crops (\Crop);
(ii) one using our \CABB loss and \Cu, trained on crops (\Crop + \CABB);
(iii) one using the standard bounding box regression loss and \Su, trained on crops (\Crop + \Su); and finally
(iv) one using both our \CABB loss and \Su, trained on crops (\Crop + \CABB + \Su).

The MVD results on top in Table~\ref{tab:mvd-detail1} show that \Crop outperforms \Full on all metrics, attesting to the advantages of crop-based training.
Both our \CABB loss and \Su separately lead to consistent improvements w.r.t. \Crop on all aggregate and pure recognition metrics.
The effects of \CABB and \Su will be explored in more detail in Sec.~\ref{sec:exp-details}.
We also see that even the weakest among our network variants surpasses all PQ baselines, the only exception being the HRNet-W48-based version of Panoptic Deeplab.
After introducing all of our contributions in \Crop + \CABB + \Su, we establish a new state of the art on Mapillary Vistas, surpassing existing approaches by very wide margins (\eg +4.5\% PQ, +5.2\% mAP).

The IDD experiments in the middle of Table~\ref{tab:mvd-detail1} show similar results: \Crop outperforms \Full in most metrics, while \CABB + \Su bring further improvements, most pronounced in PC.
Compared to prior works, we observe much improved mAP scores and state of the art PQ, while segmentation metrics lag a bit behind.
One possible explanation could be the advanced panoptic fusion strategy adopted in EfficientPS, which particularly aims at improving instance segmentation.
We observe the same trends in the Cityscapes results reported in the bottom of Table~\ref{tab:mvd-detail1}, although with reduced margins.
While Cityscapes is smaller than IDD and MVD, and some metrics are already quite saturated, we still get notable +1.5\% gain for mAP in our \Crop + \CABB + \Su setting over previous state-of-the-art.

\subsection{Detailed Analysis}
\label{sec:exp-details}

\begin{figure*}[thb]
\centering
\resizebox{\textwidth}{!}{
\begin{tikzpicture}[baseline]
\pgfplotstableread[col sep=comma]{data/mvd_box_map_by_size.dat}\datatable
\begin{axis}[
    title={Box mAP by size},
    ylabel=mAP,
    width=10cm,
    height=5cm,
    ybar, bar width=.13,
    ymajorgrids=true,
    xticklabels from table={\datatable}{size},
    xtick=data,
    xtick pos=left,
    enlargelimits=0.15,
    legend style={at={(0.01, 0.82),font=\footnotesize},anchor=west},
    cycle list name=colorblind-friendly
]
\addplot table [x expr=\coordindex, y expr=\thisrow{full}*100] {\datatable};
\addplot table [x expr=\coordindex, y expr=\thisrow{baseline}*100] {\datatable};
\addplot table [x expr=\coordindex, y expr=\thisrow{cabb}*100] {\datatable};
\addplot table [x expr=\coordindex, y expr=\thisrow{su}*100] {\datatable};
\addplot table [x expr=\coordindex, y expr=\thisrow{all}*100] {\datatable};
\legend{\Full,\Crop,\Crop + \CABB,\Crop + \Su,\Crop + \CABB + \Su}
\end{axis}
\end{tikzpicture}
\hspace{1em}
\begin{tikzpicture}[baseline]
\pgfplotstableread[col sep=comma]{data/mvd_msk_map_by_size.dat}\datatable
\begin{axis}[
    title={Mask mAP by size},
    ylabel=mAP,
    width=10cm,
    height=5cm,
    ybar, bar width=.13,
    ymajorgrids=true,
    xticklabels from table={\datatable}{size},
    xtick=data,
    xtick pos=left,
    enlargelimits=0.15,
    legend style={at={(0.01, 0.82),font=\footnotesize},anchor=west},
    cycle list name=colorblind-friendly
]
\addplot table [x expr=\coordindex, y expr=\thisrow{full}*100] {\datatable};
\addplot table [x expr=\coordindex, y expr=\thisrow{baseline}*100] {\datatable};
\addplot table [x expr=\coordindex, y expr=\thisrow{cabb}*100] {\datatable};
\addplot table [x expr=\coordindex, y expr=\thisrow{su}*100] {\datatable};
\addplot table [x expr=\coordindex, y expr=\thisrow{all}*100] {\datatable};
\legend{\Full,\Crop,\Crop + \CABB,\Crop + \Su,\Crop + \CABB + \Su}
\end{axis}
\end{tikzpicture}
}
\caption{Mean Average Precision results on Mapillary Vistas, averaged over different size-based subdivisions of the validation instances. The reported ranges are percentiles of the distribution of instance areas in the validation set.}
\label{fig:ap}
\end{figure*}
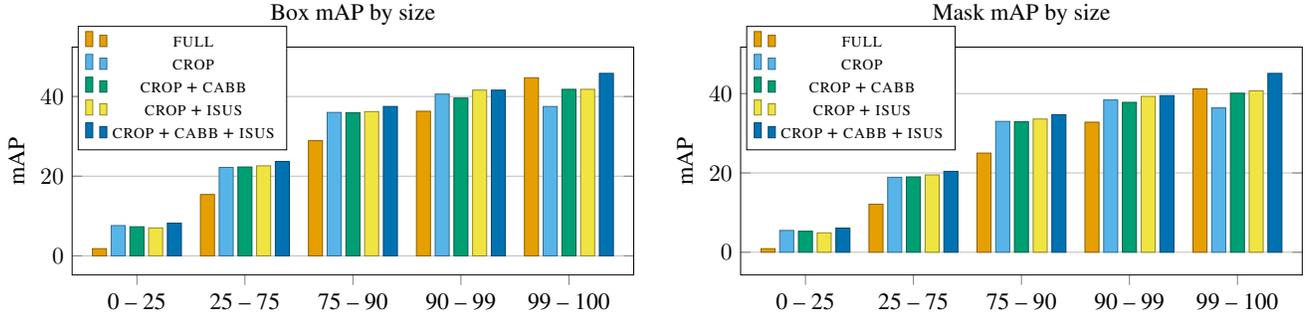

\begin{figure*}
\centering
\includegraphics[width=\textwidth]{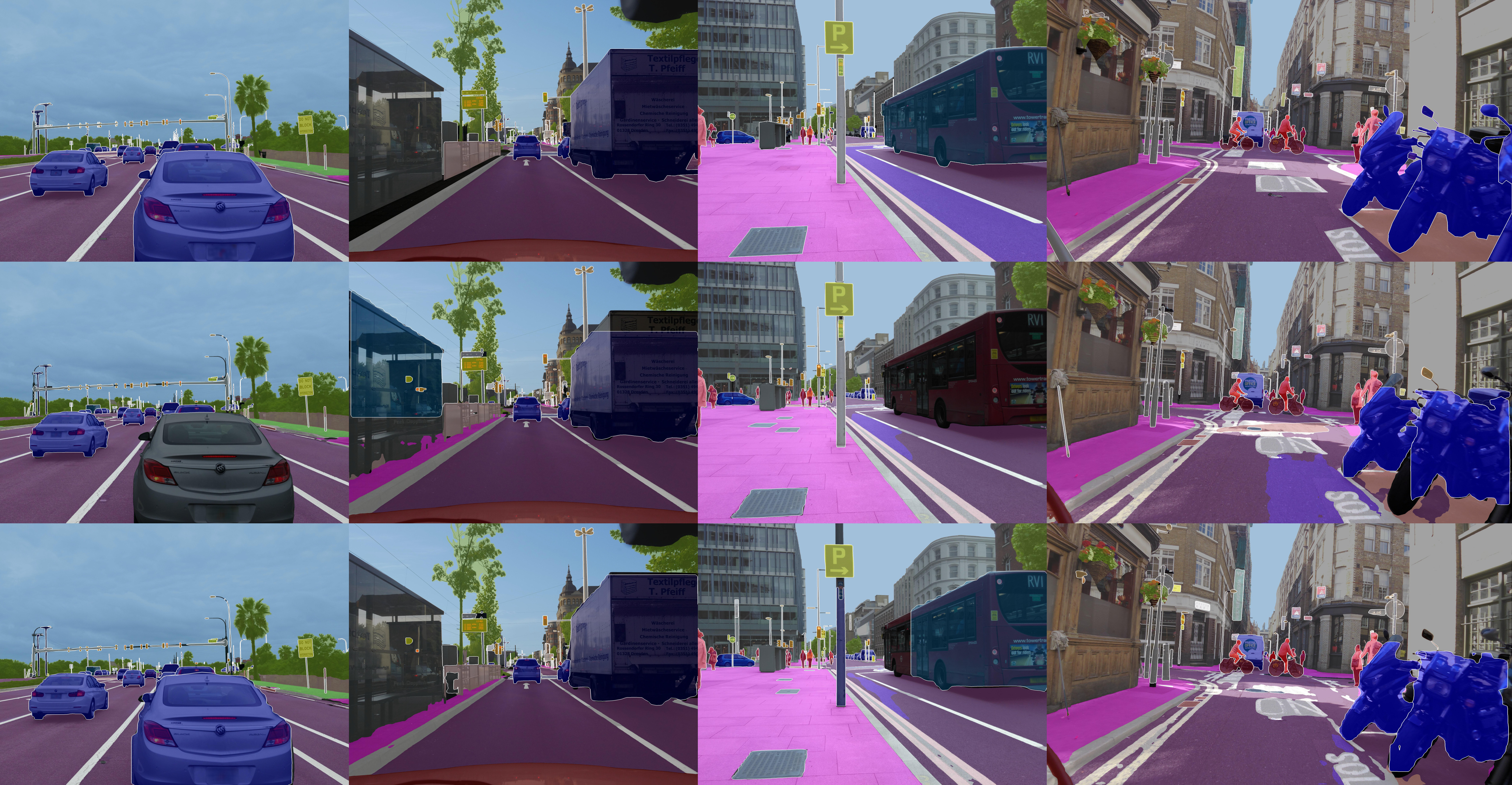}
\caption{Ground truth (first row) and panoptic segmentation results on Mapillary Vistas's validation set obtained with \Crop (second row) and \Crop + \CABB + \Su (third row). Notice how \Crop + \CABB + \Su is able to detect very big instances which are completely missed by \Crop. This figure is best viewed on screen and at magnification.}
\label{fig:qual-results}
\end{figure*}

After showing our new high-scores for MVD, IDD and Cityscapes in the previous section we provde in-depth analyses for \CABB and \Su next.
First, to validate the generality of our proposals, we evaluate crop-based training, our \CABB loss, and \Su when applied to the approach of Porzi \etal~\cite{Por+19_cvpr}.
We report the results in Table~\ref{tab:mvd-detail1} under two settings, both trained on $1024\times 1024$ crops: the unmodified network from~\cite{Por+19_cvpr}, reproduced from their original code (Seamless + \Crop), and the same network combined with our \CABB loss and \Su (Seamless + \CABB + \Su).
Consistent with our other results, the introduction of crop-based training brings consistent improvements over the baseline, particularly in detection metrics, while the \CABB loss and \Su further boost the scores achieving a +2.8\% improvement in PQ \wrt Seamelss.
Further ablations on \Su are reported in Sec.~\ref{sec:su_ablation}.

As discussed in Sec.~\ref{sec:intro} and~\ref{sec:method}, we expect crop-based training to have a negative impact on large objects, which we aim to mitigate with our \CABB loss, while our \Su should bring improvements across all scales by smoothing out the object size imbalance.
To verify this, in Fig.~\ref{fig:ap} we plot box (left) and mask (right) mAP scores as a function of object size (\ie area), splitting the validation instances into five categories according to size percentiles.
As expected, \Crop outperforms \Full by a wide margin on smaller objects, as it is able to work on almost double the input resolution.
On the other hand, the gap between \Crop and \Full shrinks as object size increases, with \Full finally surpassing \Crop on the largest objects.
By adding \CABB the crop-based network is able to fill the gap with \Full when dealing with objects in the 99th size percentile, while maintaining strong performance in all other size categories.
\Su brings generalized improvements over \Crop at most scales, with the exception of the smallest one.
More surprisingly, \Su seems to be similarly beneficial as \CABB on the largest objects.
A possible explanation is that, by increasing generalization across scales, \Su allows the network to properly infer the sizes and positions of objects that are bigger than the training crop.
Finally, when \CABB and \Su are combined, we observe consistent improvements on all sizes.

In Table~\ref{tab:mvd-detail1} we report additional comparisons between our network variants, based on PC and PQ$^\dagger$ (see Appendix.~\ref{sec:metrics}).
In all datasets, we observe a clear improvement in these metrics when the \CABB loss and \Su are introduced in the network.
In particular, the gap between \Crop and \Crop + \CABB + \Su in PC$^\text{th}$ is markedly larger than in PQ$^\text{th}$.
This is unsurprising, as the PC metrics weight image segments proportionally to size, clearly highlighting how the \CABB loss is able to boost the network's accuracy on large instances.
This is also visible from the qualitative results in Fig.~\ref{fig:qual-results}, showing a comparison between the outputs of \Crop and \Crop + \CABB + \Su on 12Mpixels Mapillary Vistas validation images featuring large objects.

\section{Conclusions}
\label{sec:conclusions}

In this paper we have tackled the problem of training panoptic segmentation networks on high resolution images, using crop-based training strategies to enable the use of modern, high-capacity architectures.
Training on crops has a negative impact on the detection of large objects, which we addressed by introducing a novel crop-aware bounding box regression loss.
To counteract the imbalanced distribution of objects sizes, we further proposed a novel data sampling and augmentation strategy which we have shown to improve generalization across scales.
By combining these with a state-of-the-art panoptic segmentation architecture we achieved new top scores on the Mapillary Vistas dataset, surpassing the previous best performing approaches by +4.5\% PQ and +5.2\% mAP.
We also showed state of the art results on the Indian Driving and Cityscapes datasets on multiple detection and segmentation metrics.

\appendix
\section{Proof of Results}\label{sec:proofs}
Let $\omega_0=b_1-a_1$ and let $\xi(\omega)$ be the objective of \eqref{eq:prob_con} with first-order derivative
\begin{equation}\label{eq:deriv1}
\begin{aligned}
	\xi'(\omega)=\frac{1}{2}\ell_\beta'\left(\frac{\omega-\hat\omega}{2}\right)+\frac{1}{\omega}\ell_\beta'(\log(\omega)-\log(\omega_\mathtt P))\,.
\end{aligned}
\end{equation}
The first-order derivative of the Huber loss is given by $\ell_\beta'(x)=\max(\min(\beta^{-1}x,1),-1)$. 
We assume that $\omega_\mathtt P>0$ and $\omega_0>0$.

\begin{proposition}\label{prop:unconstrained}
	If a strictly feasible local solution $(\delta^\star, \omega^\star)$ of \eqref{eq:prob} exists, then $\delta^\star=\delta_\mathtt P$ and $\omega^\star=\omega_\mathtt P$.
\end{proposition}
\begin{proof}
Let $\delta_\lambda=\lambda \delta_\mathtt P + (1-\lambda) \delta^\star$ and $\omega_\lambda=\lambda \omega_\mathtt P + (1-\lambda) \omega^\star$.
By contradiction, assume that a strictly feasible local solution $(\delta^\star, \omega^\star)$ exists such that $(\delta^\star,\omega^\star)\neq(\delta_\mathtt P,\omega_\mathtt P)$. Then, we expect $\frac{d}{d\lambda}\varphi(\delta_\lambda,\omega_\lambda)\big|_{\lambda=0}=0$, where $\varphi(\delta,\omega)$ denotes the objective of \eqref{eq:prob}.
However
\[
\begin{aligned}
&\frac{d}{d\lambda}\varphi(\delta_\lambda,\omega_\lambda)\Big|_{\lambda=0}\\
\,&= \frac{d}{d\lambda}\ell_\beta(\delta_\lambda-\delta_\mathtt P)\Big|_{\lambda=0} + \frac{d}{d\lambda}\ell_\beta(\log(\omega_\lambda)-\log(\omega_\mathtt P))\Big|_{\lambda=0}\\
\,&=(\delta_\mathtt P-\delta^\star)\ell_\beta'(\delta^\star-\delta_\mathtt P) + \frac{\omega_\mathtt P-\omega^\star}{\omega^\star}\ell_\beta'(\log(\omega^\star)-\log(\omega_\mathtt P))\\
\end{aligned}
\]
is negative because $\ell_\beta'(x)<0$ if $x<0$ and $\ell_\beta'(x)>0$ if $x>0$, the logarithm is an ordering-preserving mapping and $\omega^\star>0$. This yields a contradiction thus proving the result.
\end{proof}

\begin{proposition}\label{prop:outer1}
	$\xi'(\omega)<0$ for all $0<\omega<\min\{\hat \omega,\omega_\mathtt P\}$ and $\xi'(\omega)>0$ for all $\omega>\max\{\hat \omega,\omega_\mathtt P\}$.
\end{proposition}
\begin{proof}
For $0<\omega<\hat \omega$ we have that $\ell_\beta'\left(\frac{\omega-\hat \omega}{2}\right)<0$ and for $0<\omega<\omega_\mathtt P$ we have that $\ell_\beta'(\log(\omega)-\log(\omega_\mathtt P))<0$. Accordingly, $\xi'(\omega)<0$
	for $0<\omega<\min\{\hat\omega,\omega_\mathtt P\}$.

	Similarly, for $\omega>\hat\omega$, we have that $\ell_\beta'\left(\frac{\omega-\hat\omega}{2}\right)>0$ and for $\omega>\omega_\mathtt P$ we have that $\ell_\beta'(\log(\omega)-\log(\omega_\mathtt P))>0$. Accordingly, $\xi'(\omega)>0$ for $\omega>\max\{\hat\omega,\omega_\mathtt P\}$.
\end{proof}

\begin{proposition}\label{prop:trivial}
	If $\max\{\hat \omega,\omega_\mathtt P\}\leq \omega_0$ then $\omega_0$ is the solution to \eqref{eq:prob_con}.
\end{proposition}
\begin{proof}
	By Prop.~\ref{prop:outer1}, $\xi'(\omega)>0$ for all $\omega>\max\{\hat\omega,\omega_\mathtt P\}$. Accordingly, the same holds true for all $\omega\geq \omega_0$, which implies that $\xi(\omega_0)$ yields the lowest feasible objective value.
\end{proof}

\begin{proposition}\label{prop:valid}
	A solution to  \eqref{eq:prob_con} exists in $[\max\{\omega_0,\min\{\hat\omega,\omega_\mathtt P\}\},\max\{\hat\omega, \omega_\mathtt P\}]$ if $\omega_0\leq\max\{\hat\omega,\omega_\mathtt P\}$.
\end{proposition}
\begin{proof}
A feasible solution $\omega<\max\{\omega_0,\min\{\hat \omega,\omega_\mathtt P\}\}$ exists only if $\omega_0\leq \omega < \min\{\hat \omega,\omega_\mathtt P\}$. 
If this is the case, $\xi'(\omega)<0$ holds in the latter interval by Prop.~\ref{prop:outer1}. Accordingly, for $\omega\leq \min\{\hat \omega,\omega_\mathtt P\}$ the best objective is attained at $\min\{\hat \omega,\omega_\mathtt P\}$.
Similarly by Prop.~\ref{prop:outer1}, $\xi'(\omega)>0$ if $\omega>\max\{\hat \omega,\omega_\mathtt P\}$ and, therefore, for $\omega\geq\max\{\hat \omega,\omega_\mathtt P\}$ the best objective is attained at $\max\{\hat \omega,\omega_\mathtt P\}$. 
Hence, a solution to \eqref{eq:prob_con} exists in the required interval.
\end{proof}

\begin{proposition}\label{prop:caseA}
	If $\max\{\omega_0,\hat\omega\}<\omega_\mathtt P$ then a solution to~\eqref{eq:prob_con} exists in $[\max\{\omega_0, \hat \omega\},\omega_\mathtt P]$ and there $\xi'$ is strictly increasing.
\end{proposition}
\begin{proof}
	For all $\hat\omega\leq \omega<\omega '$ we have that $0\leq \ell_\beta'(\frac{\omega-\hat \omega}{2})\leq \ell_\beta'(\frac{\omega'-\hat \omega}{2})$. Moreover, for all $0<\omega<\omega'\leq \omega_\mathtt P$,
	both $\ell_\beta'(\log(\omega)-\log(\omega_\mathtt P))\leq \ell_\beta'(\log(\omega')-\log(\omega_\mathtt P))\leq 0$ and $\frac{1}{\omega}>\frac{1}{\omega'}>0$ hold, which imply $\frac{1}{\omega}\ell_\beta'(\log(\omega)-\log(\omega_\mathtt P))<\frac{1}{\omega'}\ell_\beta'(\log(\omega')-\log(\omega_\mathtt P))\leq 0$.
	It follows that $\xi'(\omega)<\xi'(\omega')$ holds in the required interval. 
\end{proof}

\begin{proposition}\label{prop:eta}
	$\eta(\omega)=\frac{1}{\omega}\ell_\beta'(\log(\omega)-\log(\omega_\mathtt P))$ is
	\begin{itemize}
	\item strictly increasing in $(0,e^{\min\{\beta,1\}}\omega_\mathtt P]$, and
	\item strictly decreasing for $\omega\geq e\omega_\mathtt P$.
	\end{itemize}
\end{proposition}
\begin{proof}
	$\eta(\omega)$ is strictly increasing for $0<\omega<e^{-\beta}\omega_\mathtt P$ because in this case $\eta(\omega)=-\frac{1}{\omega}$ and strictly decreasing for $\omega>\omega_\mathtt P e^\beta$ because in this case $\eta(\omega)=\frac{1}{\omega}$. 
	For $e^{-\beta}\omega_\mathtt P \leq \omega \leq e^\beta \omega_\mathtt P$ we have $\eta(\omega)=\frac{1}{\omega\beta}(\log(\omega)-\log(\omega_\mathtt P))$ and 
	\[
		\eta'(\omega)=\frac{1}{\omega^2\beta}\left[1-\log(\omega)+\log(\omega_\mathtt P)\right]\,.
	\]
	Since $\eta'(e^{-\beta}\omega_\mathtt P)>0$, $\eta'(e^{\beta}\omega_\mathtt P)<0$ and $\eta'(\omega)=0$ only at $\omega=e \omega_\mathtt P $, it follows
	by continuity of $\eta'$ that $\eta'(\omega)>0$ for $\omega<e\omega_\mathtt P$ and $\eta'(\omega)<0$ for $\omega>e\omega_\mathtt P$.
	Accordingly, $\eta(\omega)$ is strictly increasing in $[e^{-\beta}\omega_\mathtt P,e^{\min\{\beta,1\}}\omega_\mathtt P]$. Since the same holds for $0<\omega<e^{-\beta}\omega_\mathtt P$ as shown before, by continuity of $\eta$,  we conclude that $\eta(\omega)$ is strictly increasing along the whole interval $(0,e^{\min\{\beta,1\}}\omega_\mathtt P]$. Similarly, 	
	we have that $\eta(\omega)$ is strictly decreasing in $[e\omega_\mathtt P ,  e^\beta\omega_\mathtt P]$ and for $\omega>e^\beta\omega_\mathtt P$ as shown before. Hence, by continuity of $\eta$, we conclude that $\eta(\omega)$ is strictly decreasing for $\omega\geq e\omega_\mathtt P $.
\end{proof}

\begin{proposition}\label{prop:caseB}
	If $\max\{\omega_0,\omega_\mathtt P\} < \hat \omega$ then a solution to~\eqref{eq:prob_con} is either  $\omega_0$ or $\hat\omega$, 
	or it belongs to one of the following intervals:
	\begin{enumerate}[label=(\roman*)]
	\item $J_1=[\max\{\omega_0,\omega_\mathtt P\},\min\{e^{\min\{\beta,1\}}\omega_\mathtt P,\hat \omega\}]$,
	\item $J_2=[\max\{\omega_0,2\sqrt{\beta},\hat\omega-2\beta,e^\beta\omega_\mathtt P \},\hat\omega]$,
	\item $J_3=[\max\{\omega_0,\hat\omega-2\beta,e\omega_\mathtt P\},\min\{e^\beta\omega_\mathtt P,\hat\omega\}]$ if $\hat\omega\leq 4\sqrt{2}$
	\item $J_4=[\max\{\omega_0,\hat\omega-2\beta,e\omega_\mathtt P\}, \min\{e^\beta\omega_\mathtt P,\hat\omega, \nu_1\}]$ if $\hat\omega>4\sqrt{2}$,
	\item $J_5=[\max\{\omega_0,\hat\omega-2\beta, e\omega_\mathtt P, \nu_2\}, \min\{e^\beta\omega_\mathtt P,\hat\omega\}]$ if $\hat\omega>4\sqrt{2}$,
	\end{enumerate}
	where $\nu_{1,2}=\frac{\hat \omega}{4}\left(1\pm\sqrt{1-\frac{32}{\hat \omega^2}}\right)$.

	Moreover, $\xi'$ is strictly increasing in (i)-(ii) and $\sigma(\omega)=\omega\xi'(\omega)$ is strictly increasing in $(iii)-(v)$.
\end{proposition}
\begin{proof}
	By Prop.~\ref{prop:valid} a solution to~\eqref{eq:prob_con} exists in  $I=[\max\{\omega_0,\omega_\mathtt P\}, \hat\omega]$. 
	We partition $I$ into sections where $\xi'$ or $\sigma$ are either strictly increasing or strictly decreasing. We work by cases:
	\begin{itemize}
		\item $J_1$. In this interval $\ell_\beta'(\frac{\omega-\hat \omega}{2})$ is increasing in $\omega$ and $\eta$ is strictly increasing by Prop.~\ref{prop:eta}. Hence, $\xi'(\omega)=\frac{1}{2}\ell_\beta'(\frac{\omega-\hat \omega}{2})+\eta(\omega)$ is strictly increasing as well. 
		\item $[\max\{\omega_0,e\omega_\mathtt P\},\hat \omega-2\beta]$. In this interval $\ell_\beta'(\frac{\omega-\hat \omega}{2})$ is constant and $\eta$ is strictly decreasing by Prop.~\ref{prop:eta}. Hence, $\xi'$ is strictly decreasing as well. 
	\item $[\max\{\omega_0,\hat\omega-2\beta,e^{\beta}\omega_\mathtt P\},\hat \omega]$. In this interval, $\xi'(\omega)=\frac{1}{4\beta}(\omega-\hat\omega)+\frac{1}{\omega}$ and $\xi''(\omega)=0$ holds only in the feasible point $\omega=2\sqrt{\beta}$, while $\xi''(\omega)<0$ for $0<\omega<2\sqrt{\beta}$ and $\xi''(\omega)>0$ for $\omega>2\sqrt{\beta}$. Accordingly $\xi'(\omega)$ is strictly decreasing in the interval $[\max\{\omega_0,\hat\omega-2\beta,e^{\beta}\omega_\mathtt P\},\min\{2\sqrt{\beta},\hat\omega\}]$ and strictly increasing in the interval $J_2$.
	\item $J_3$. In this interval, $\xi'(\omega)=\frac{1}{4\beta}(\omega-\hat \omega)+\frac{1}{\omega\beta}(\log(\omega)-\log(\omega_\mathtt P))$ and by setting $\sigma'(\omega)=0$ we find at most two solutions, namely $\nu_{1,2}$, which are distinct and real for $\hat\omega> 4\sqrt{2}$. Both solutions might potentially belong to the interval under consideration. 
	The sign of $\sigma'(\omega)$ is negative for $\nu_1<\omega<\nu_2$ and positive for $\omega<\nu_1$ and $\omega>\nu_2$. Accordingly $\sigma(\omega)$ is strictly increasing in the intervals 
	$J_4$ and $J_5$, while it is strictly decreasing in the interval $[\max\{\omega_0,\hat\omega-2\beta,e\omega_\mathtt P,\nu_1\}, \min\{e^\beta\omega_\mathtt P,\hat\omega, \nu_2\}]$. If $\hat\omega\leq 4\sqrt{2}$ then $\sigma'(\omega)\geq 0$ in the whole interval $J_3$, with equality only if $\hat\omega=4\sqrt{2}$ and $\omega=\nu_1=\nu_2$. Accordingly, if $\hat\omega\leq 4\sqrt{2}$ we have that $\sigma(\omega)$ is strictly increasing in $J_3$. 
\end{itemize}

Since $\sigma(\omega)$ and $\xi'(\omega)$ share the same sign, 
given an interval $J$ where $\xi'$ or $\sigma$ is strictly decreasing, we have one of the following cases: a) $\xi'$ is strictly positive, b) $\xi'$ is strictly negative or c) $\xi'$ transitions once from a positive to a negative sign. In all three cases, a solution to~\eqref{eq:prob_con} cannot exist in the interior of $J$ but can be at most at one endpoint of $J$. For the same reason, no solution can be at the junction of two intervals where $\xi'$ or $\sigma$ are strictly decreasing. Hence, the endpoint has to be either an endpoint of $I$ or be in common with an interval where either $\xi'$ or $\sigma$ are strictly increasing, which proves the result.
\end{proof}

\section{Optimization Algorithms}

\begin{algorithm}[h]
	\caption{Solves the optimization problem~\eqref{eq:prob_con}.
	}
	\algrenewcommand\algorithmicindent{0.8em}%
	\begin{algorithmic}[1]
	\Function{Solve\_\ref{eq:prob_con}}{$\omega_\mathtt P, \hat \omega, a_1,b_1$}
	\State $\omega_0=b_1-a_1$
	\State $S=\{\omega_0\}$\Comment{Used to collect potential solutions}
	\If{$\max\{\omega_0,\hat \omega\}<\omega_\mathtt P$}\Comment{Prop.~\ref{prop:caseA}}\label{line:caseA_cond}
	\State \Return  \Call{find\_min}{[$\max\{\omega_0,\hat\omega\},\omega_\mathtt P],\xi'$}\label{line:caseA}
	\ElsIf{$\max\{\omega_0,\omega_\mathtt P\}<\hat\omega$}\Comment{Prop.~\ref{prop:caseB}}\label{line:caseB_cond}
		\State $S=S\cup\{\hat \omega,$ \Call{find\_min}{$ J_1,\xi'$}$,$ \Call{find\_min}{$ J_2,\xi'$}$\}$
		\If{$\hat\omega\leq 4\sqrt 2$}
		\State $S=S\cup\{$\Call{find\_min}{$J_3,\sigma$}$\}$
		\Else{}
		\State $S=S\cup\{$\Call{find\_min}{$J_4,\sigma$}$,$ \Call{find\_min}{$J_5,\sigma$}$\}$

		\EndIf
	\EndIf
	\State\Return$\argmin_{\omega\in S} \xi(\omega)$\label{line:end}
	\EndFunction
\end{algorithmic}
	\label{alg:prob_con}
\end{algorithm}

\begin{algorithm}[h]
	\caption{Solves the optimization problem~\eqref{eq:prob}.
	}
	\algrenewcommand\algorithmicindent{0.8em}%
	\begin{algorithmic}[1]
	\Function{Solve\_\ref{eq:prob}}{$\delta_\mathtt P,\omega_\mathtt P, a_2, b_2$}
	\State $\hat\omega_1=2(\delta_\mathtt P-a_2)$
	\State $\hat\omega_2=2(b_2-\delta_\mathtt P)$
	\If{$\omega_\mathtt P\geq \max\{\hat\omega_1,\hat\omega_2\}$}
	\State\Return $(\delta_\mathtt P,\omega_\mathtt P)$ \label{line:unconstrained}
	\EndIf
	\State $\omega_1=$\Call{Solve\_\ref{eq:prob_con}}{$\omega_\mathtt P, \hat\omega_1, a_2,b_2$}\label{line:con1}
	\State $\omega_2=$\Call{Solve\_\ref{eq:prob_con}}{$\omega_\mathtt P, \hat\omega_2, a_2,b_2$}\label{line:con2}
	\If{$\xi(\omega_1)\leq \xi(\omega_2)$}
	\State\Return $(a_2+\frac{\omega_1}{2},\omega_1)$\label{line:out1}
	\Else{}
	\State\Return $(b_2-\frac{\omega_2}{2},\omega_2)$\label{line:out2}

	\EndIf
	\EndFunction
\end{algorithmic}
\label{alg:prob}
\end{algorithm}

\begin{algorithm}[h]
	\caption{
		Finds the minimum of an objective $\xi$ in a given interval $[u,v]$ by leveraging an increasing, continuous function $\varphi$, whose sign agrees with the sign of $\xi'$.
	}
	\algrenewcommand\algorithmicindent{0.8em}%
	\begin{algorithmic}[1]
	\Function{find\_min}{$[u,v], \varphi$}
	\If{ $\varphi(u)\geq 0$}
	\State\Return $u$
	\ElsIf{$\varphi(v)\leq 0$}
	\State\Return $v$
	\Else
	\State $m=\frac{u+v}{2}$
	\If{$v-u<\epsilon$}
	\Comment $\epsilon$ is a tolerance
	\State \Return $m$
	\ElsIf{$\varphi(m)\geq 0$}
	\State \Return \Call{find\_min}{$[u,m],\varphi$}
	\Else{}
	\State \Return \Call{find\_min}{$[m,v],\varphi$}
	\EndIf
	\EndIf
	\EndFunction
\end{algorithmic}
\label{alg:find}
\end{algorithm}

\label{sec:algos}
In this section we provide the optimization algorithms used to solve~\eqref{eq:prob} and \eqref{eq:prob_con}, which exploit the theoretical results given in Sec.~\ref{sec:proofs}.

\paragraph{Algorithm for~\eqref{eq:prob}.}
Alg.~\ref{alg:prob} provides a solution to the optimization problem~\eqref{eq:prob}.
The idea of the algorithm is sketched also in Sec.~\ref{sec:comp_aspects} of the main paper.
The global, unconstrained solution to~\eqref{eq:prob} is attained at $(\delta_\mathtt P,\omega_\mathtt P)$.
Accordingly, if this solution is feasible it is also the global solution to the constrained version of the problem (line \ref{line:unconstrained}).
If it's not feasible, we have two options, the solution is in the interior of the feasible set, or at the boundary.
However, by Prop.~\ref{prop:unconstrained} the former case is not possible, because the only solution would be the one 
we excluded already, namely  $(\delta_\mathtt P,\omega_\mathtt P)$. Hence, the solution has to lie at the boundary 
of the feasible set. Since we have only two constraints, we can apply a brute force approach, and explore two cases where we assume that the solution activates the first constraint (line \ref{line:con1}) or the second one (line \ref{line:con2}). In both cases, we boil down to solving an instance of the optimization problem~\eqref{eq:prob_con} where $\hat\omega=2(\delta_\mathtt P-a_2)$ and $\hat\omega=2(b_2-\delta_\mathtt P)$, respectively. The solution to each of those problems, denoted by $\omega_1$ and $\omega_2$ in the algorithm, is given by applying Alg.~\ref{alg:prob_con}, which is discussed later. Among those two solutions, we retain the one minimizing the objective of~\eqref{eq:prob_con}, where the objective is denoted by $\xi$ in the algorithm. If $\omega_1$ is the best one then the solution to~\eqref{eq:prob} is given by $(\delta_1,\omega_1)$ in line~\ref{line:out1}, where $\delta_1=a_2+\frac{\omega_1}{2}$ is obtained by substituting $\omega_1$ in the first constraint. Otherwise, the solution is given by $(\delta_2,\omega_2)$ in line~\ref{line:out2}, where $\delta_2=b_2-\frac{\omega_2}{2}$ is obtained similarly from the second constraint.

\paragraph{Algorithm for~\eqref{eq:prob_con}.} 
Alg.~\ref{alg:prob_con} provides a solution to the optimization problem~\eqref{eq:prob_con}.
According to Prop.~\ref{prop:trivial}, if $\max\{\hat \omega, \omega_\mathtt P\}\leq \omega_0$ then $\omega_0$ is the solution. Indeed, in this case we return $\omega_0$ in line~\ref{line:end} since it is the only element of $S$. If condition in line~\ref{line:caseA_cond} is hit, then by Prop.~\ref{prop:caseA} we can search a solution in the interval $[\max\{\omega_0,\hat\omega\},\omega_\mathtt P]$ by leveraging the monotonicity of $\xi'$. We do so by exploiting Alg.~\ref{alg:find} in line~\ref{line:caseA}, which will be discussed later. If condition in line~\ref{line:caseB_cond} is hit instead, according to Prop.~\ref{prop:caseB}, we need to search for the best solution within the intervals $J_i$ with $i\in\{1,\ldots,5\}$ eventually satisfying the given conditions. Moreover, we need to include in the pool also $\omega_0$ and $\hat\omega$. The search over each interval $J_i$ is performed via Alg.~\ref{alg:find}, by leveraging the monotonicity of $\xi'$ or $\sigma$. All potential solutions are collected into $S$ and the best one in terms of the objective is retained in line~\ref{line:end}.

\paragraph{Alg.~\ref{alg:find}.} 
Finds the minimum of an objective $\xi$ in a given interval $[u,v]$ by leveraging an increasing, continuous function $\varphi$, whose sign agrees with the sign of $\xi'$.
This can be done by searching the element in $[u,v]$ that is the closest one to a zero of $\varphi$.
Since the function is increasing, if $\varphi(u)$ is non-negative then the closest element to a zero is $u$, while if $\varphi(v)$ is non-positive then the closest element is $v$. Otherwise, we perform a dichotomic search on the half-interval where we have discording signs of $\varphi$ at the extremes until we reach the zero with sufficient accuracy.

\section{Evaluation Metrics}
\label{sec:metrics}

Panoptic Quality (PQ), originally described in~\cite{Kirillov18}, is the most commonly adopted metric to evaluate panoptic segmentation results.
We report it together with semantic Intersection over Union (mIoU) and mask mean Average Precision (mAP), in order to detailedly measure our network's segmentation and detection performance, respectively.
Some recent works~\cite{Por+19_cvpr,Yang_Google_2019} have proposed alternatives to PQ aimed at highlighting different aspects of the panoptic predictions, or overcoming potential pitfalls of PQ.
Note that we denote by PQ$^\text{th}$ and PQ$^\text{st}$ the PQ scores computed only on thing and stuff classes, respectively.\footnote{A similar notation is also used for PC.}

\paragraph{Parsing Covering.}
PQ assigns equal importance to all image segments, a choice which is not always desirable, \eg autonomous driving systems might care more about objects closer to the vehicle, and thus appear bigger in the image, than far away ones.
Motivated by this observation,~\cite{Yang_Google_2019} proposed Parsing Covering (PC) as an alternative panoptic metric that weights image segments in proportion to their areas.
Since our \CABB loss focuses on improving detection results of large objects, PC helps highlighting its impact.

\paragraph{PQ$^\dagger$.}
Porzi \etal~\cite{Por+19_cvpr} discussed a potential limitation of PQ, as it handles all classes in a uniform way, imposing a hard 0.5 threshold on IoUs of both things and stuff.
While this is strictly necessary to obtain a unique matching between thing segments and their respective ground truth, it can result in strong over-penalization of stuff segments.
To solve this, they propose PQ$^\dagger$ as a direct modification of PQ which avoids the thresholding issue, giving a more faithful representation of the quality of stuff predictions.

\section{Training hyper-parameters}
\label{sec:params}

All our networks are trained using stochastic gradient descent with momentum 0.9 and weight decay $10^{-4}$.
The training schedule starts with a warm-up phase, where the learning rate is increased linearly from 0 to a value $\text{lr}_0$ in the first 200 training steps.
Then, the learning rate follows a linear decay schedule given by $\text{lr}_i = \text{lr}_0 (1 - \frac{i}{\text{\#steps}})$, where lr$_i$ is the value at training step $i$.
In all of our experiments we augment the data with random horizontal flipping, and in those involving \Su we fix the maximum ``things'' scale augmentation range to $r_\text{th}=[0.25, 4]$.
The scale augmentation range used in \Cu always matches the $r_\text{st}$ of the corresponding \Su experiments on the same dataset.
In the following we list the dataset-specific hyper-parameters.
Note that all schedules used for a particular dataset result in approximately the same number of training iterations.

\paragraph{Mapillary Vistas.}
All MVD experiments use a ``stuff'' scale augmentation range of $r_\text{st}=[0.8, 1.25]$.
When utilizing full images we set $s_0=1344$, $\text{lr}_0=0.02$, and we train for 75 epochs on batches including a single image per GPU.
In all other experiments we set $s_0=2400$, $\text{lr}_0=0.04$, take crops of size $1024\times 1024$, and train for 300 epochs using batches of 4 crops per GPU.

\paragraph{Indian Driving Dataset.}
In the IDD experiments we fix $s_0=1080$ and $r_\text{st}=[0.5, 2]$. 
We train for 75 epochs with batch size of 1 per GPU and $\text{lr}_0=0.02$ when using full images, and for 600 epochs with batch size of 8 per GPU, $\text{lr}_0=0.08$ and crop size $512\times 512$ when using crops.

\paragraph{Cityscapes.}
Finally, in the Cityscapes experiments we pre-train our networks on Mapillary Vistas, and fix $s_0=1024$ and $r_\text{st}=[0.5, 2]$.
When using full images, we train for 20 epochs with batch size of 1 per GPU and $\text{lr}_0=0.01$.
When using crops, we train for 150 epochs with batch size of 8 per GPU, $\text{lr}_0=0.04$ and crop size $512\times 512$.

\section{Additional \Su ablations}
\label{sec:su_ablation}

In order to validate the efficacy of \Su, we perform an additional ablation experiment where we train our \Crop network variant (with \Cu) using standard scale augmentation in the range $[0.25, 4]$.
Note that this is the same range as the $r_\text{th}$ used in the \Su experiments.
The aim here is to verify whether the instance-aware scale sampling in \Su has any impact on detection compared to a uniform sampling in the same range.
When training on MVD, we obtain the following results: PQ$^\text{th}$=42.3, mAP=22.8.
Compare these to the corresponding \Crop + \Su results: PQ$^\text{th}$=43.1, mAP=23.0.

\section{Qualitative Results}
\label{sec:qual}

In the following we visualize sample outputs of our best performing \Crop + \CABB + \Su networks on Mapillary Vistas (Fig.~\ref{fig:qual-mvd}), Cityscapes (Fig.~\ref{fig:qual-cityscapes}) and the Indian Driving Dataset (Fig.~\ref{fig:qual-idd}).

\begin{figure*}
  \centering
  \includegraphics[width=0.25\textwidth]{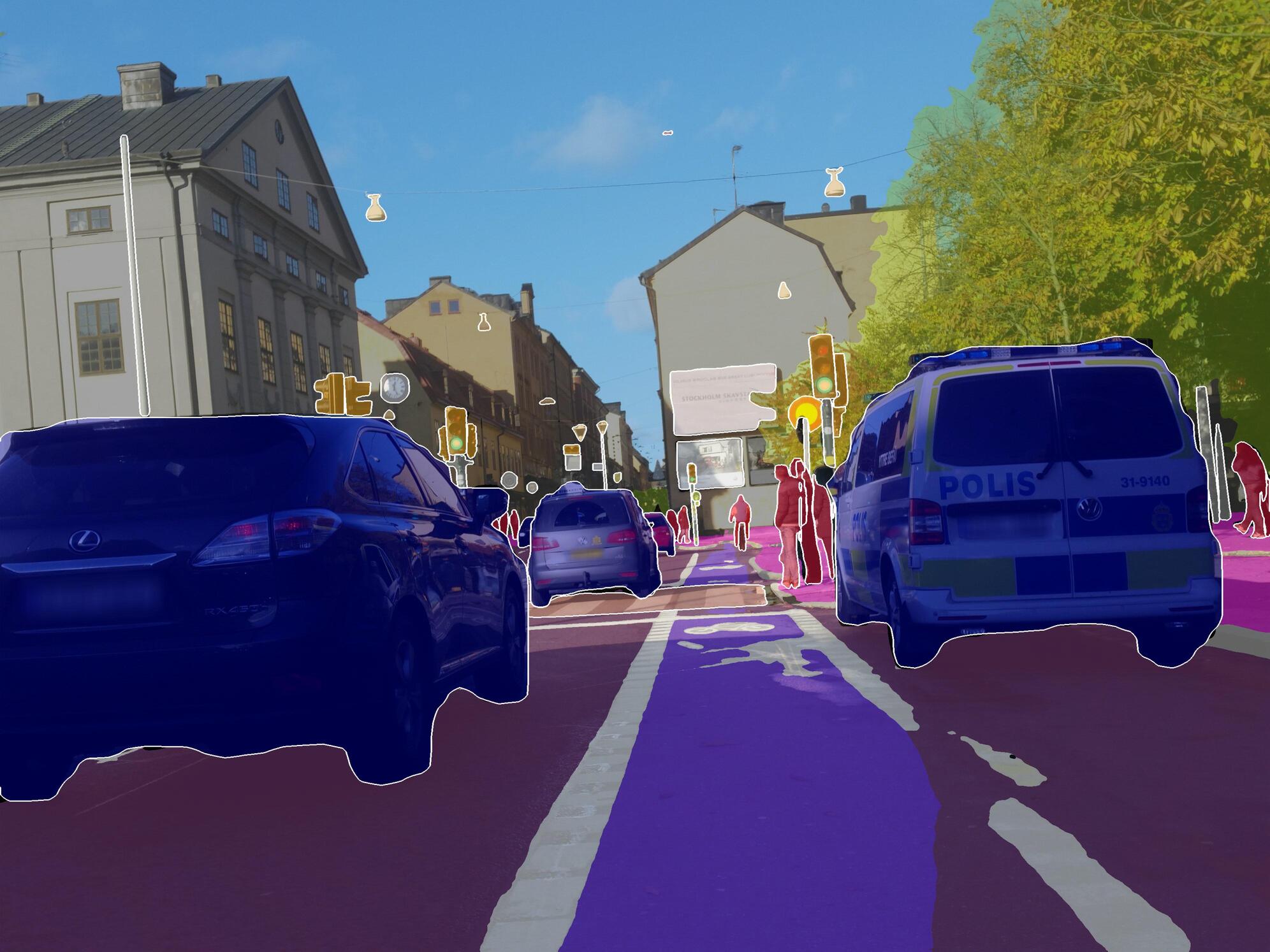}%
  \includegraphics[width=0.25\textwidth]{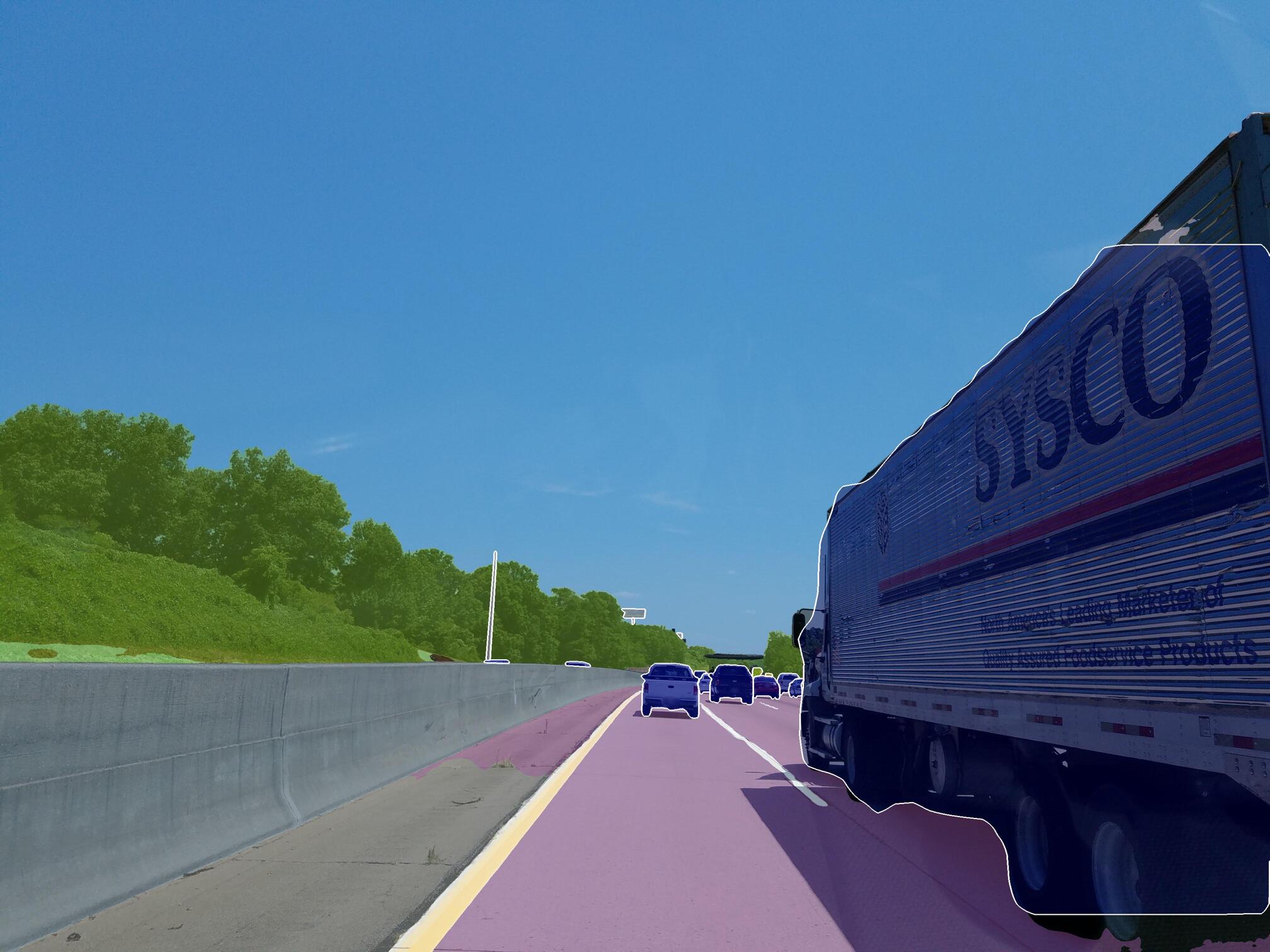}%
  \includegraphics[width=0.25\textwidth]{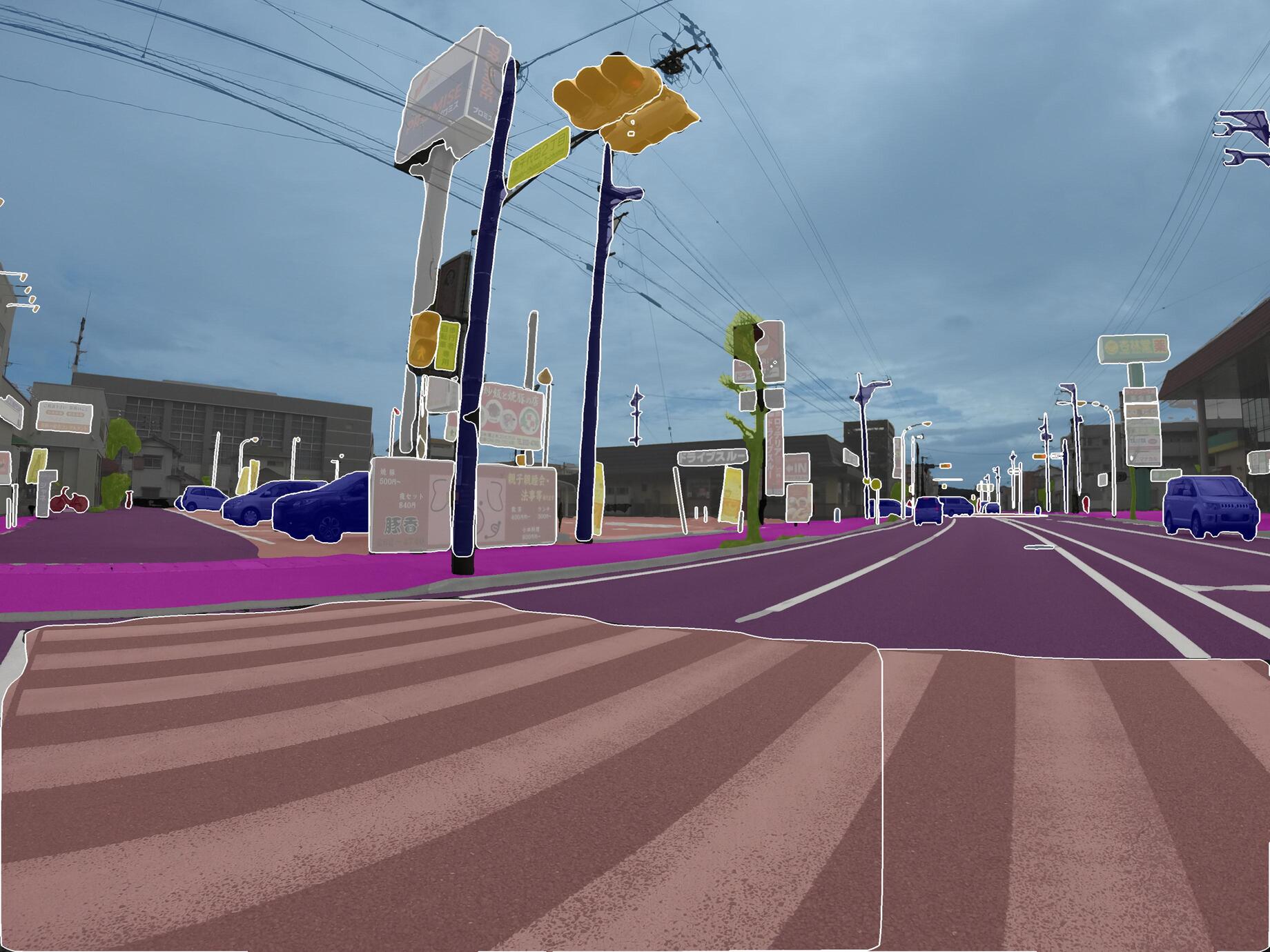}%
  \includegraphics[width=0.25\textwidth]{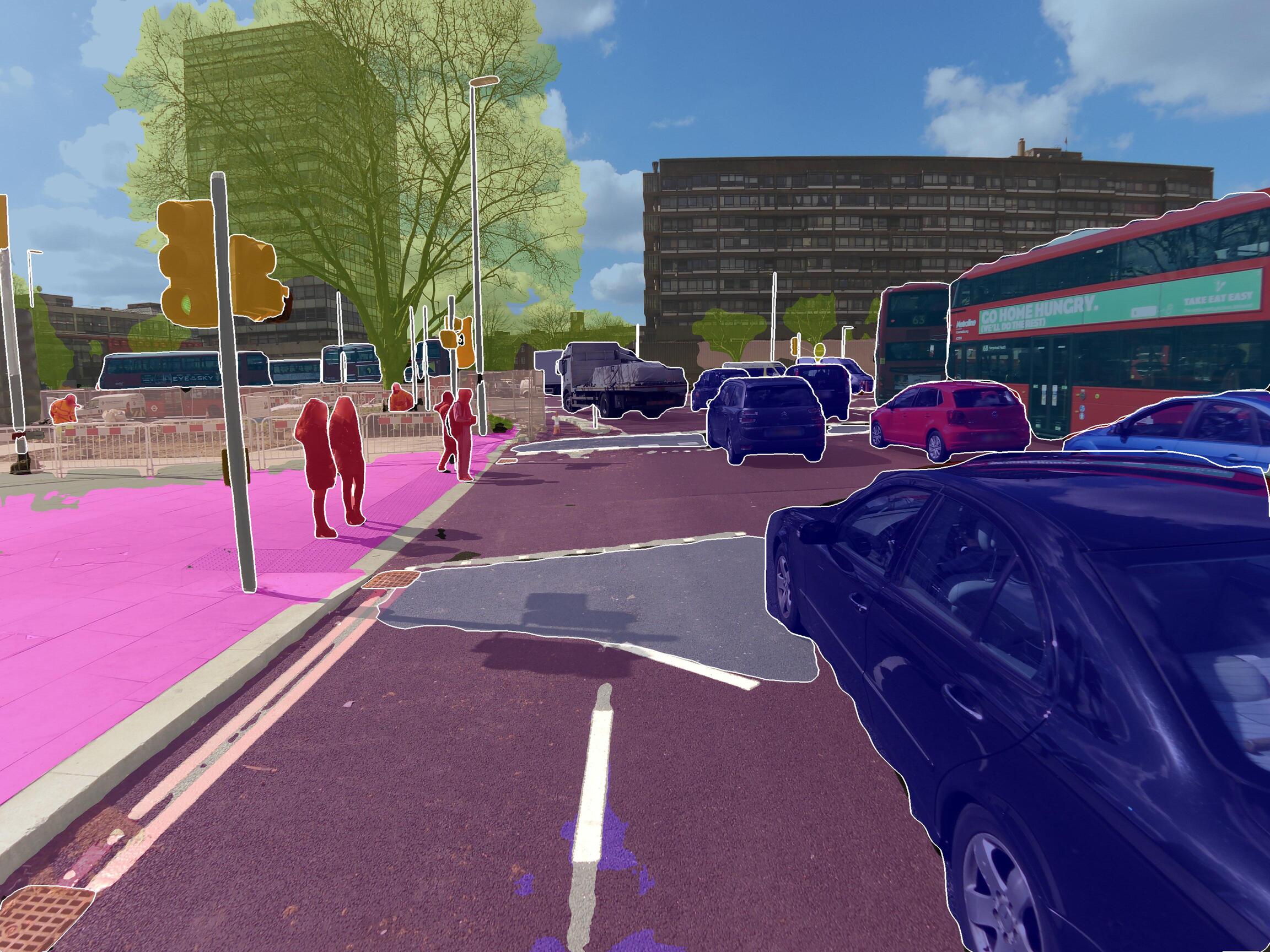}\\
  \includegraphics[width=0.25\textwidth]{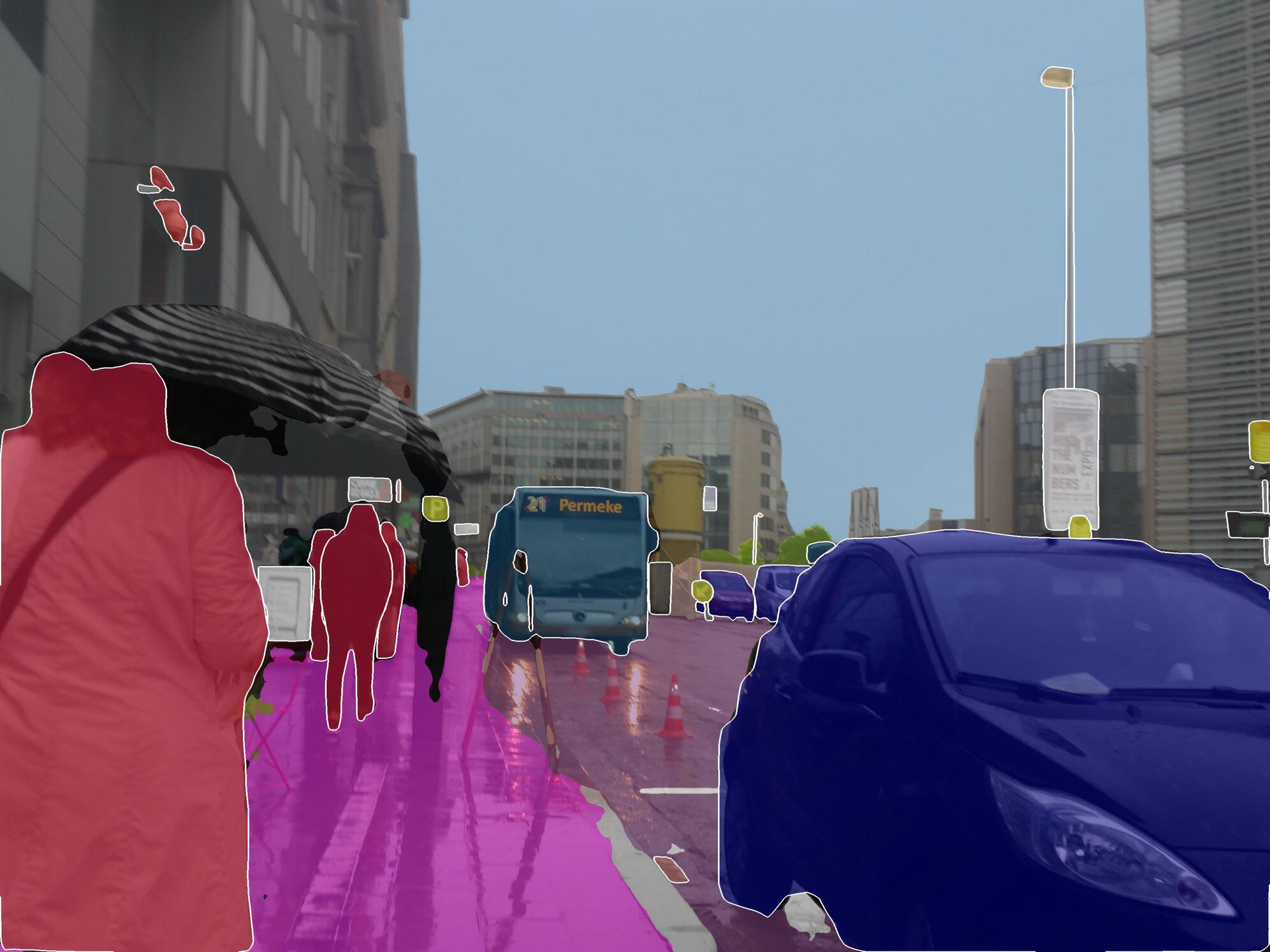}%
  \includegraphics[width=0.25\textwidth]{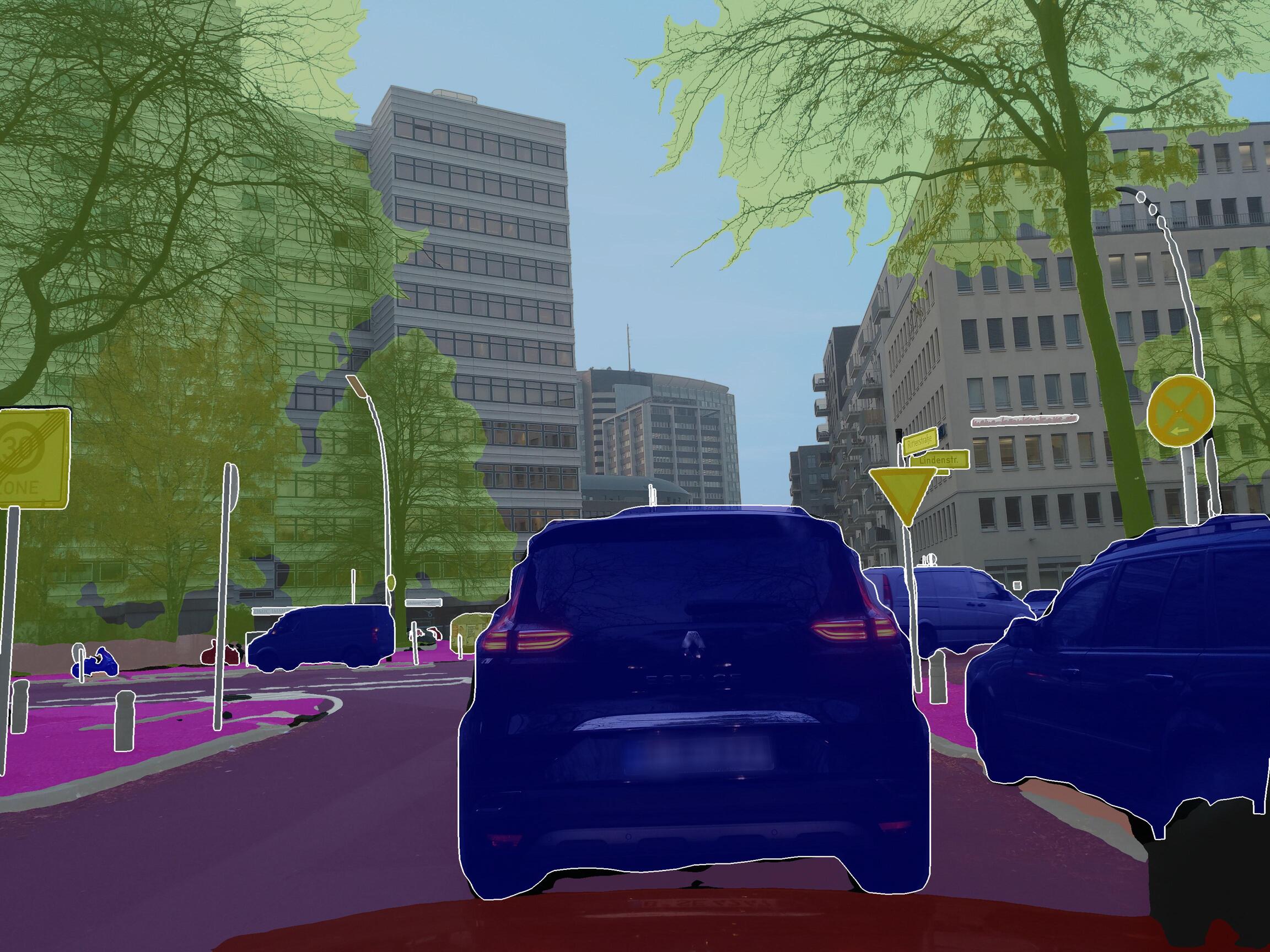}%
  \includegraphics[width=0.25\textwidth]{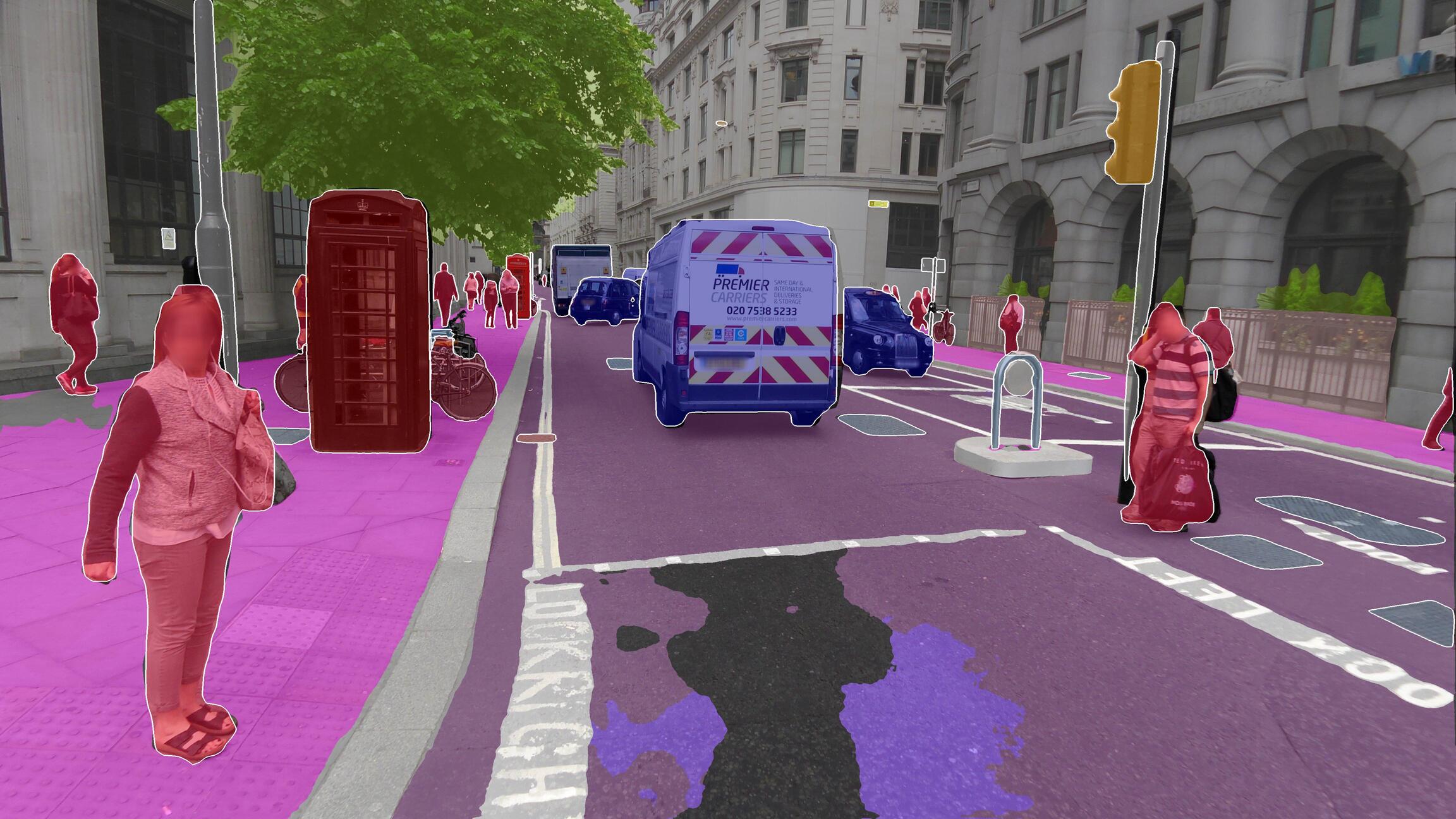}%
  \includegraphics[width=0.25\textwidth]{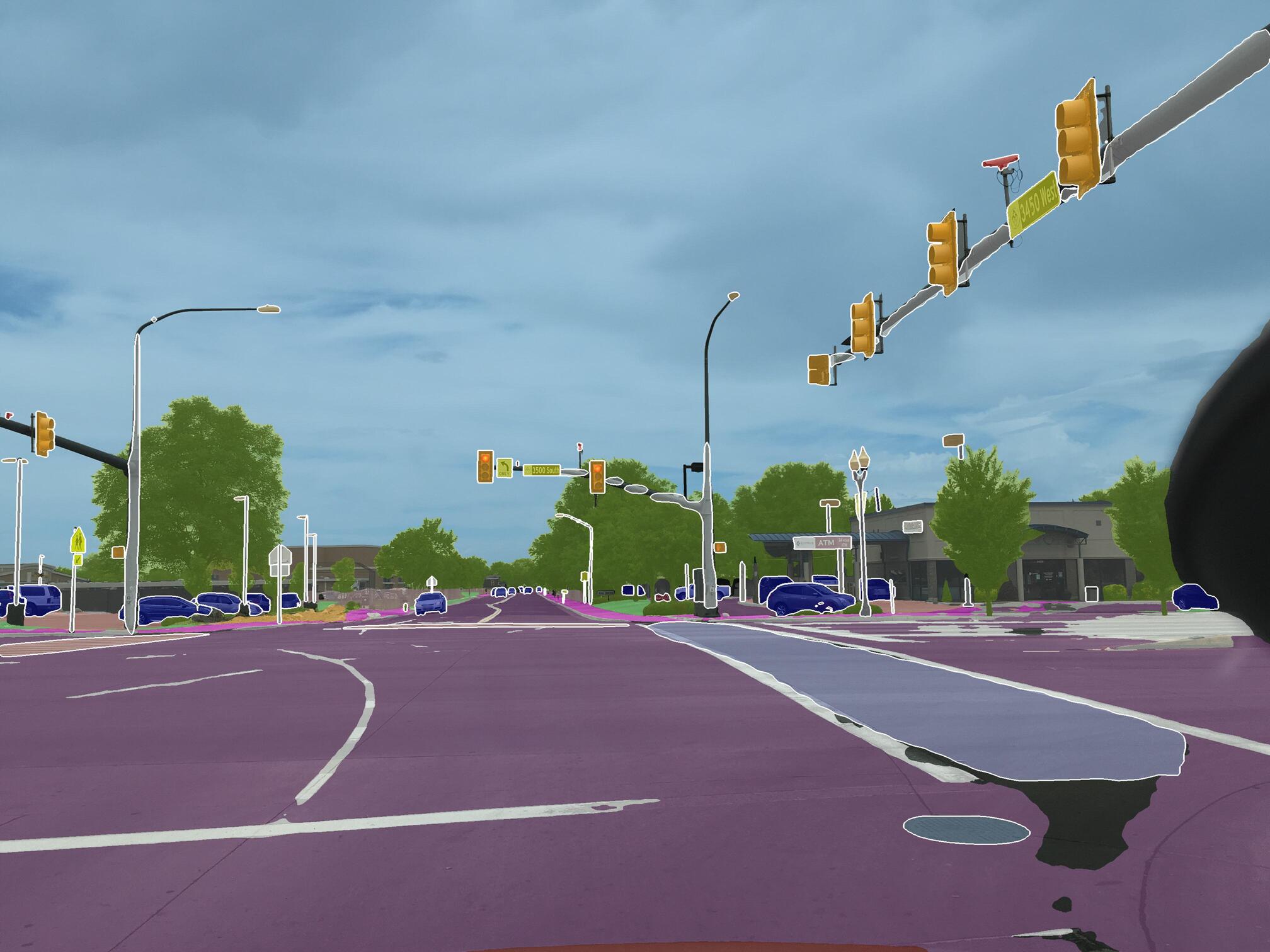}\\
  \includegraphics[width=0.25\textwidth]{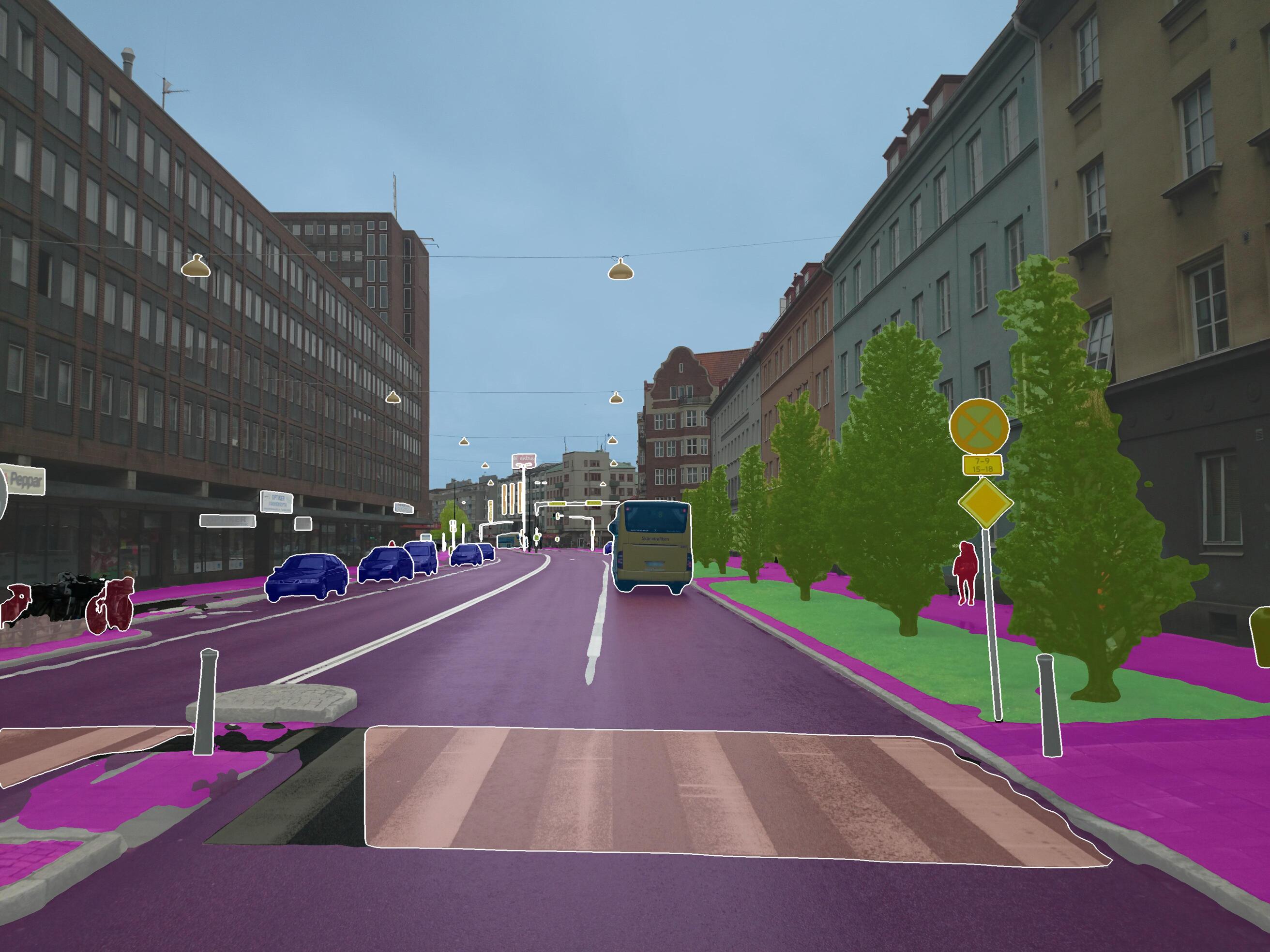}%
  \includegraphics[width=0.25\textwidth]{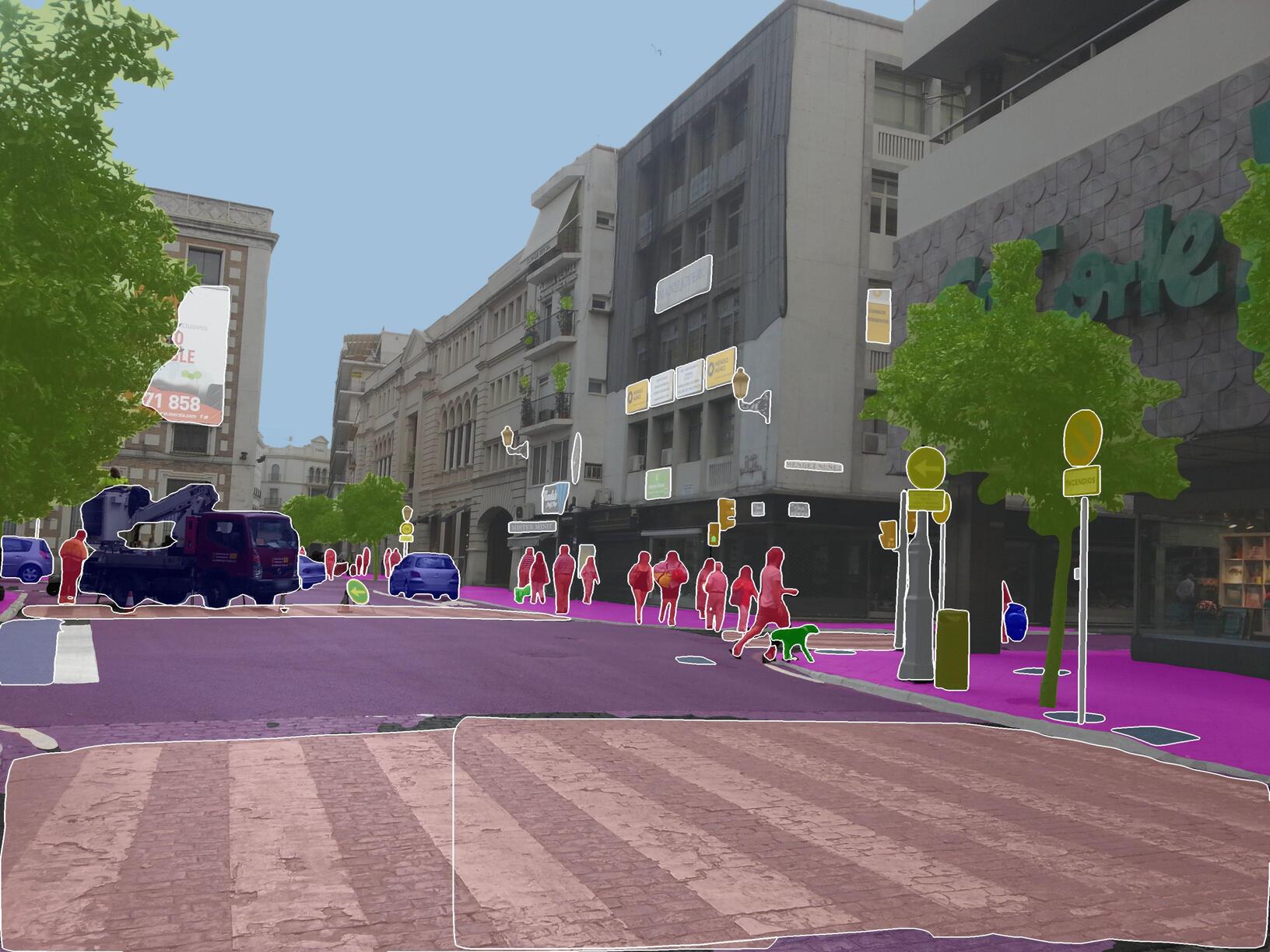}%
  \includegraphics[width=0.25\textwidth]{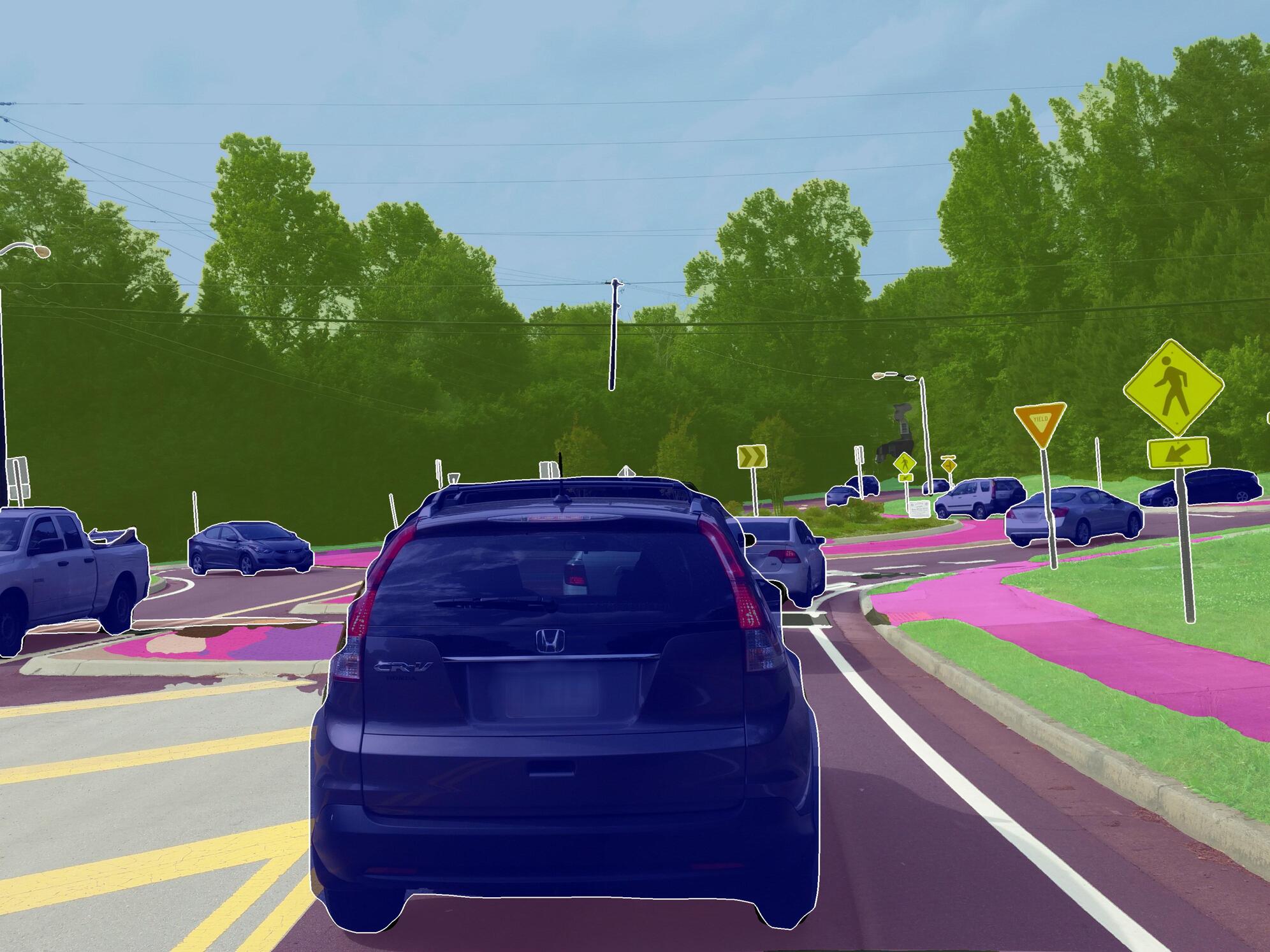}%
  \includegraphics[width=0.25\textwidth]{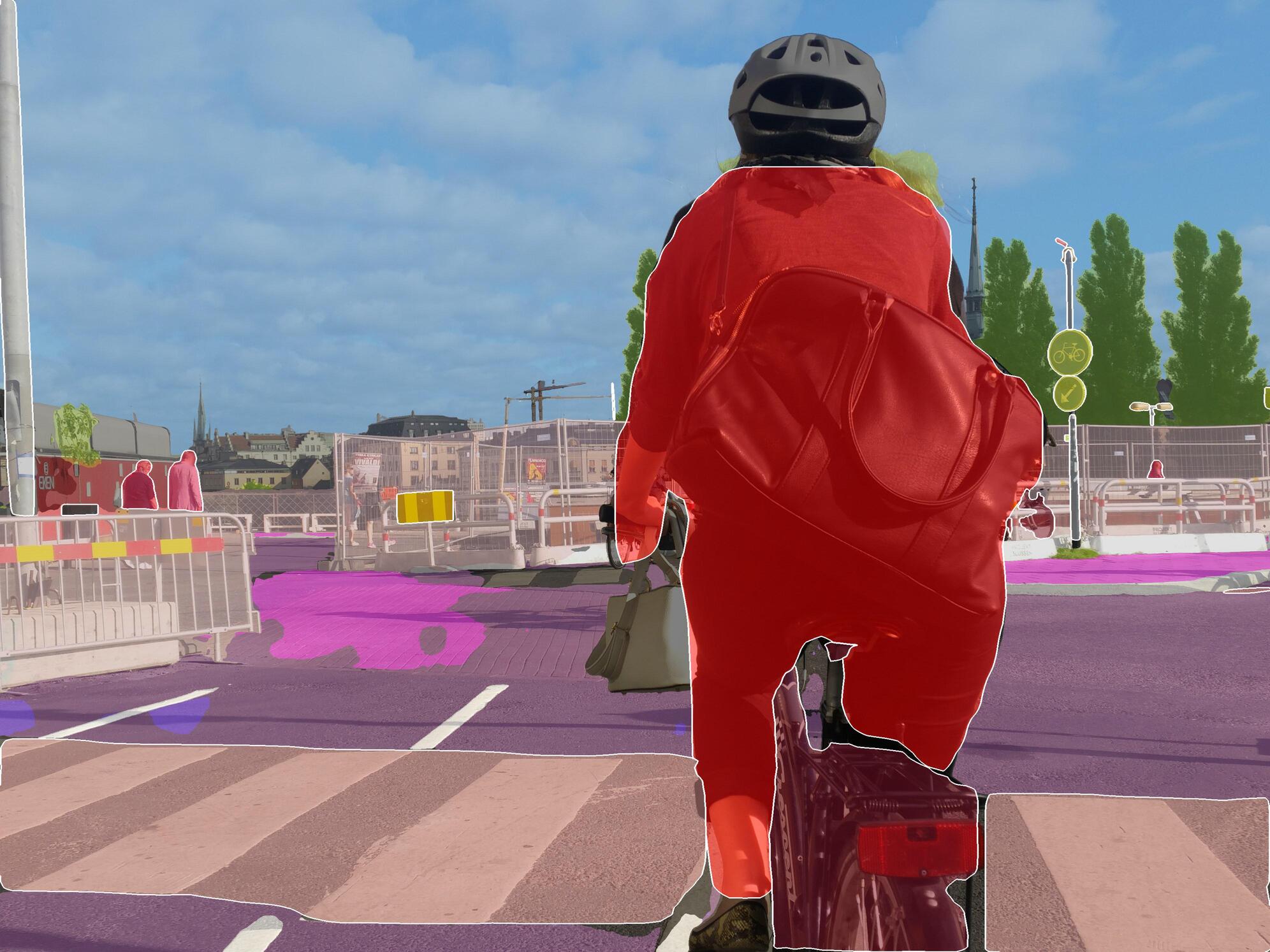}\\
  \caption{Sample outputs of \Crop + \CABB + \Su on Mapillary Vistas. Best viewed on screen.}
  \label{fig:qual-mvd}
\end{figure*}

\begin{figure*}
  \centering
  \includegraphics[width=0.33\textwidth]{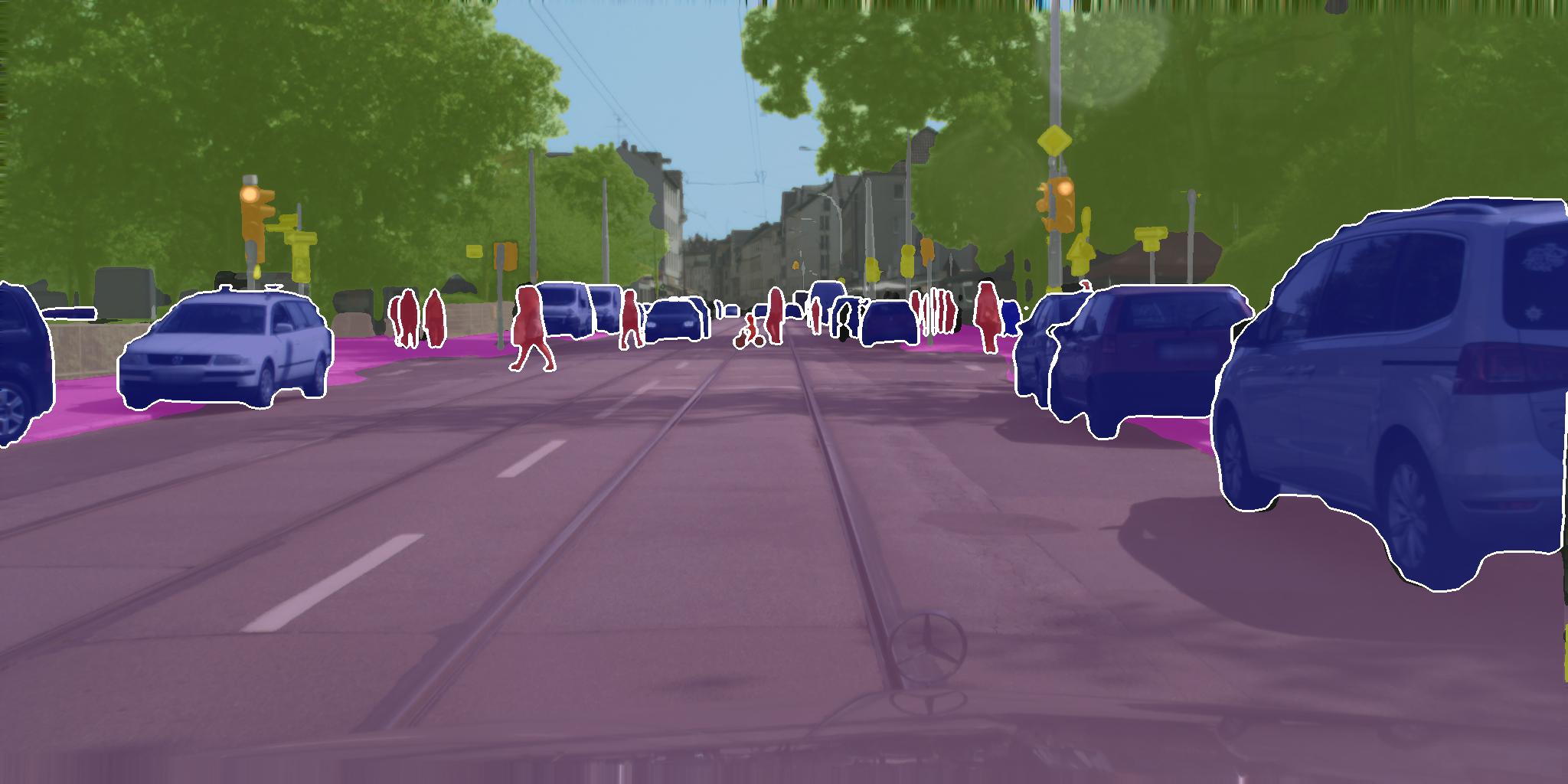}%
  \includegraphics[width=0.33\textwidth]{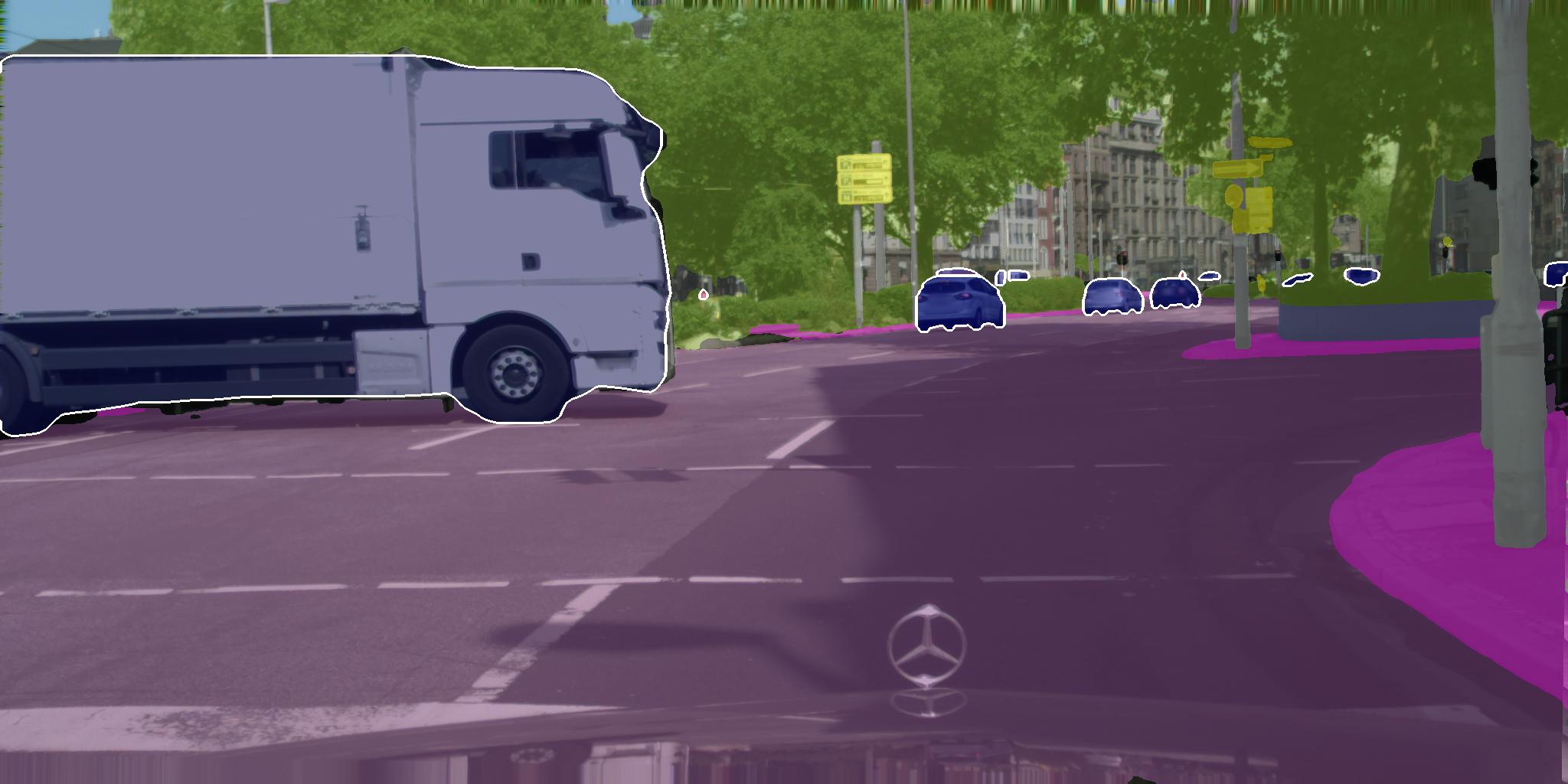}%
  \includegraphics[width=0.33\textwidth]{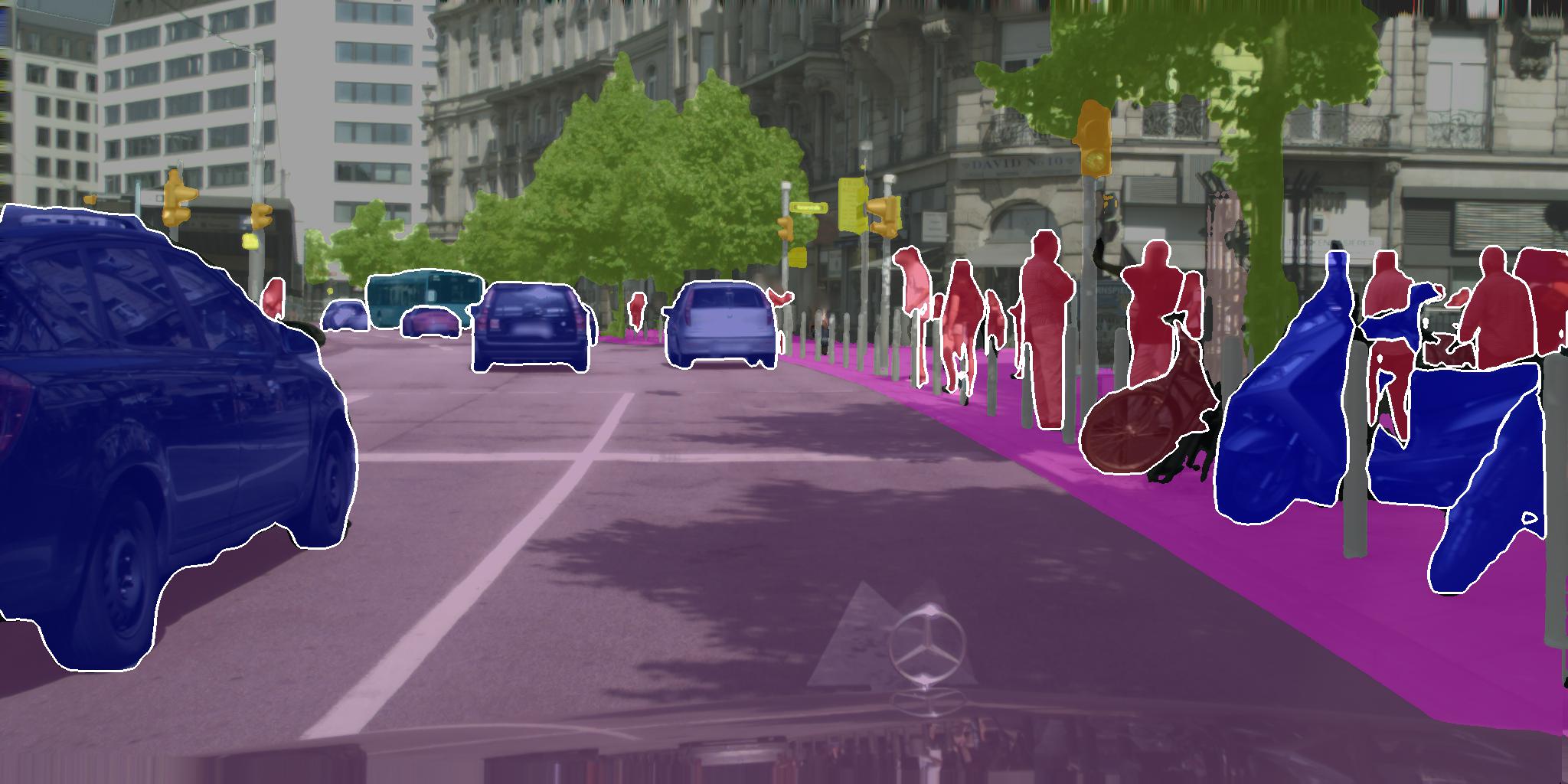}\\
  \includegraphics[width=0.33\textwidth]{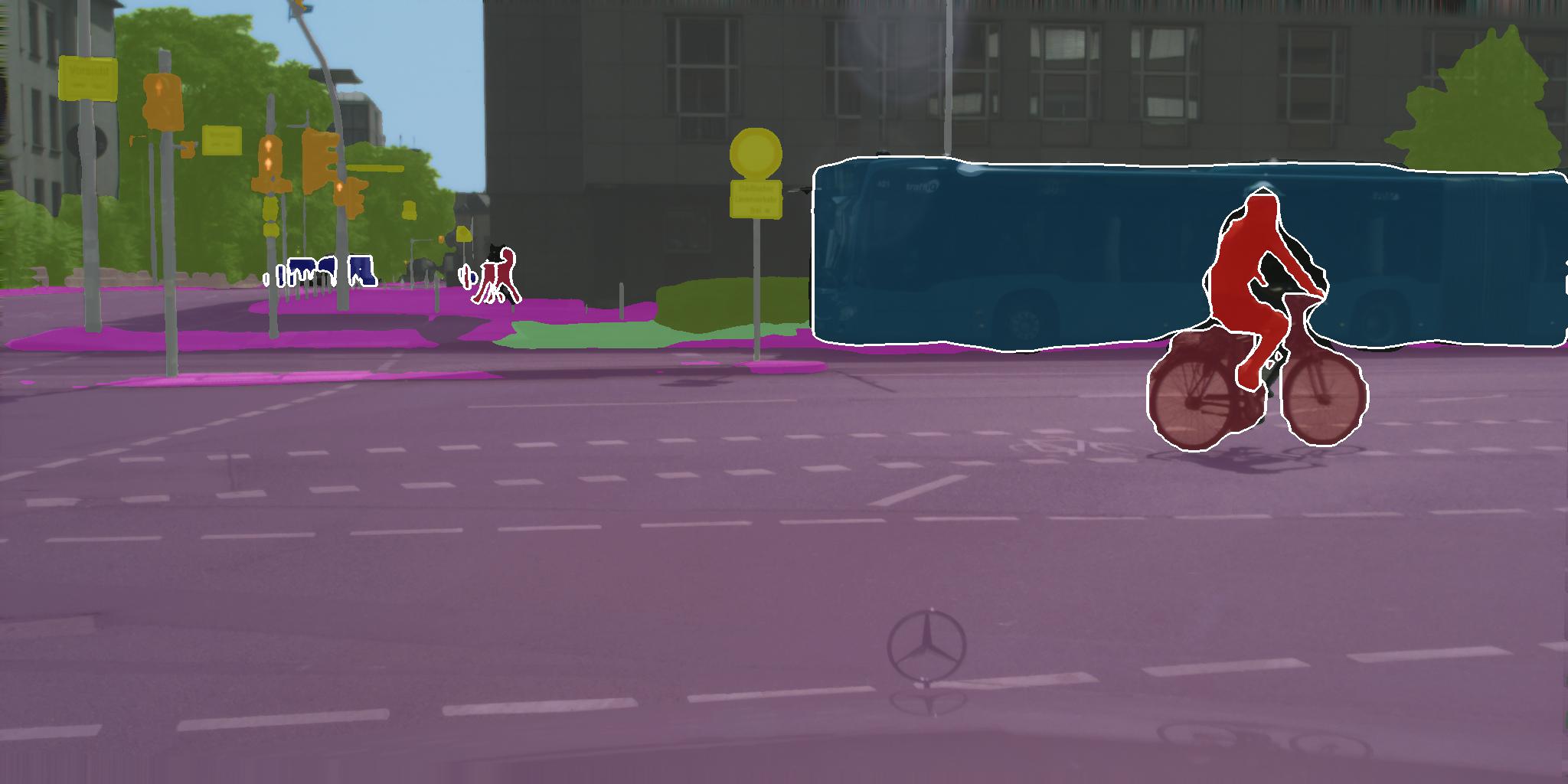}%
  \includegraphics[width=0.33\textwidth]{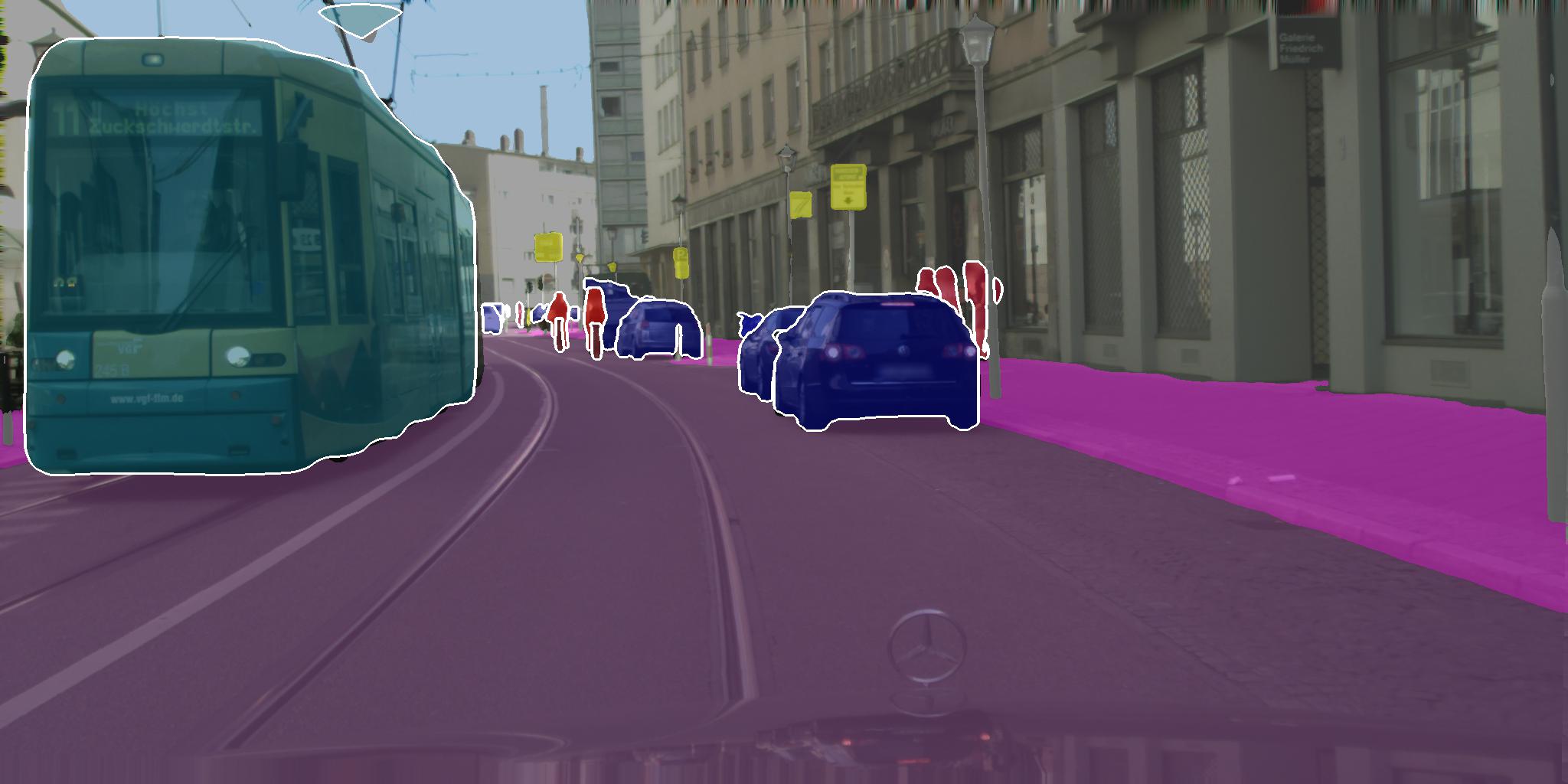}%
  \includegraphics[width=0.33\textwidth]{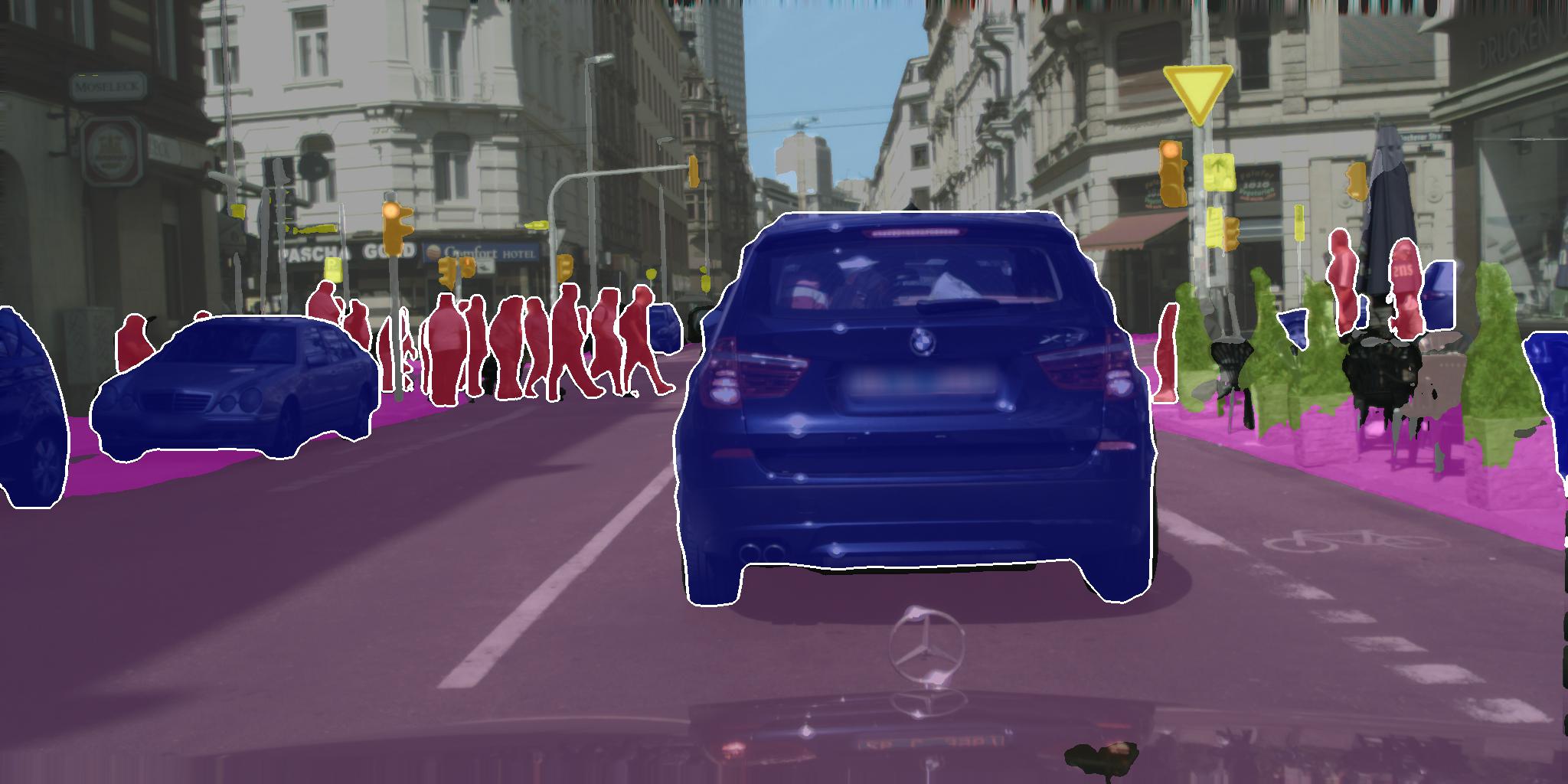}\\
  \includegraphics[width=0.33\textwidth]{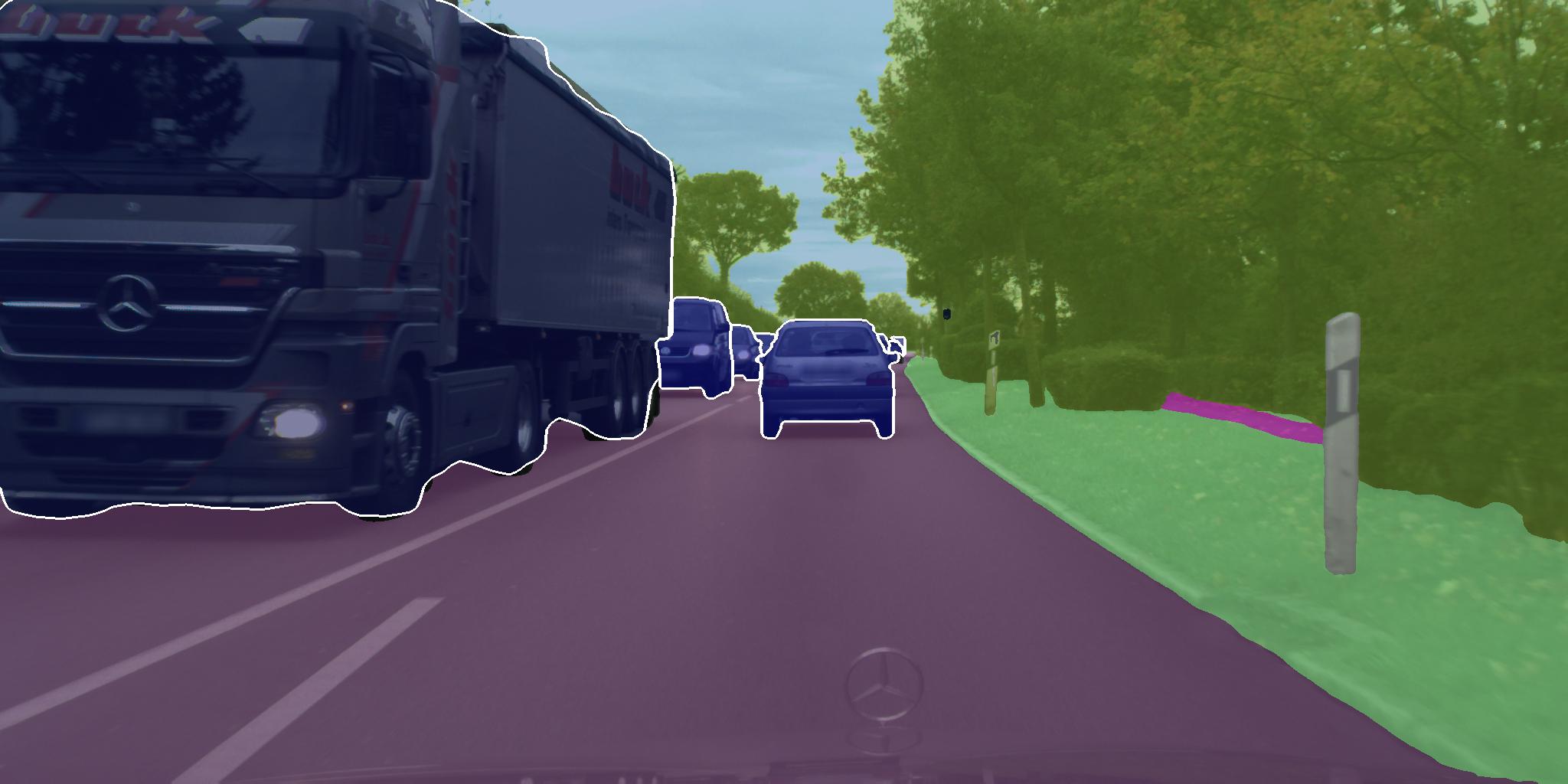}%
  \includegraphics[width=0.33\textwidth]{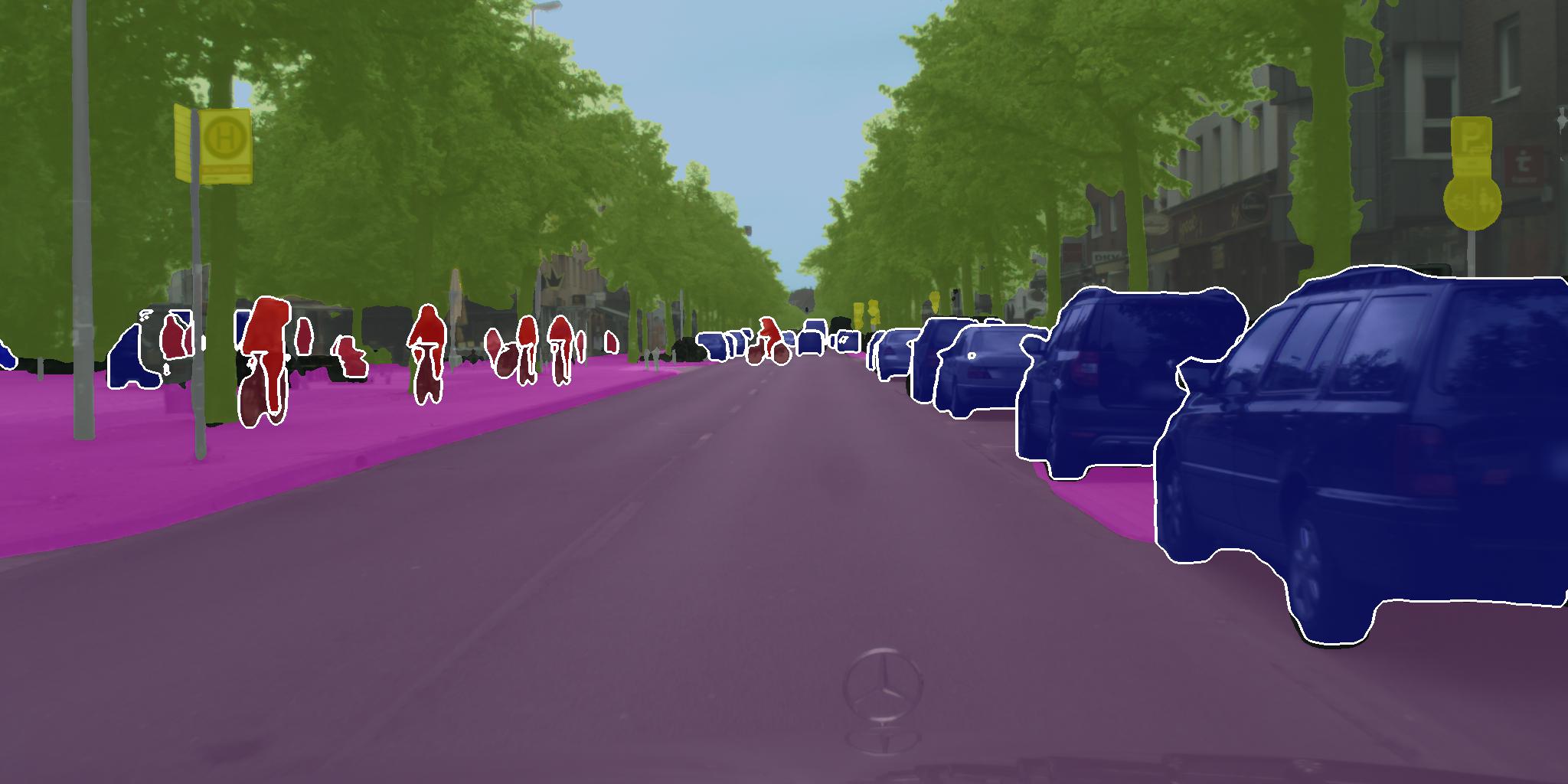}%
  \includegraphics[width=0.33\textwidth]{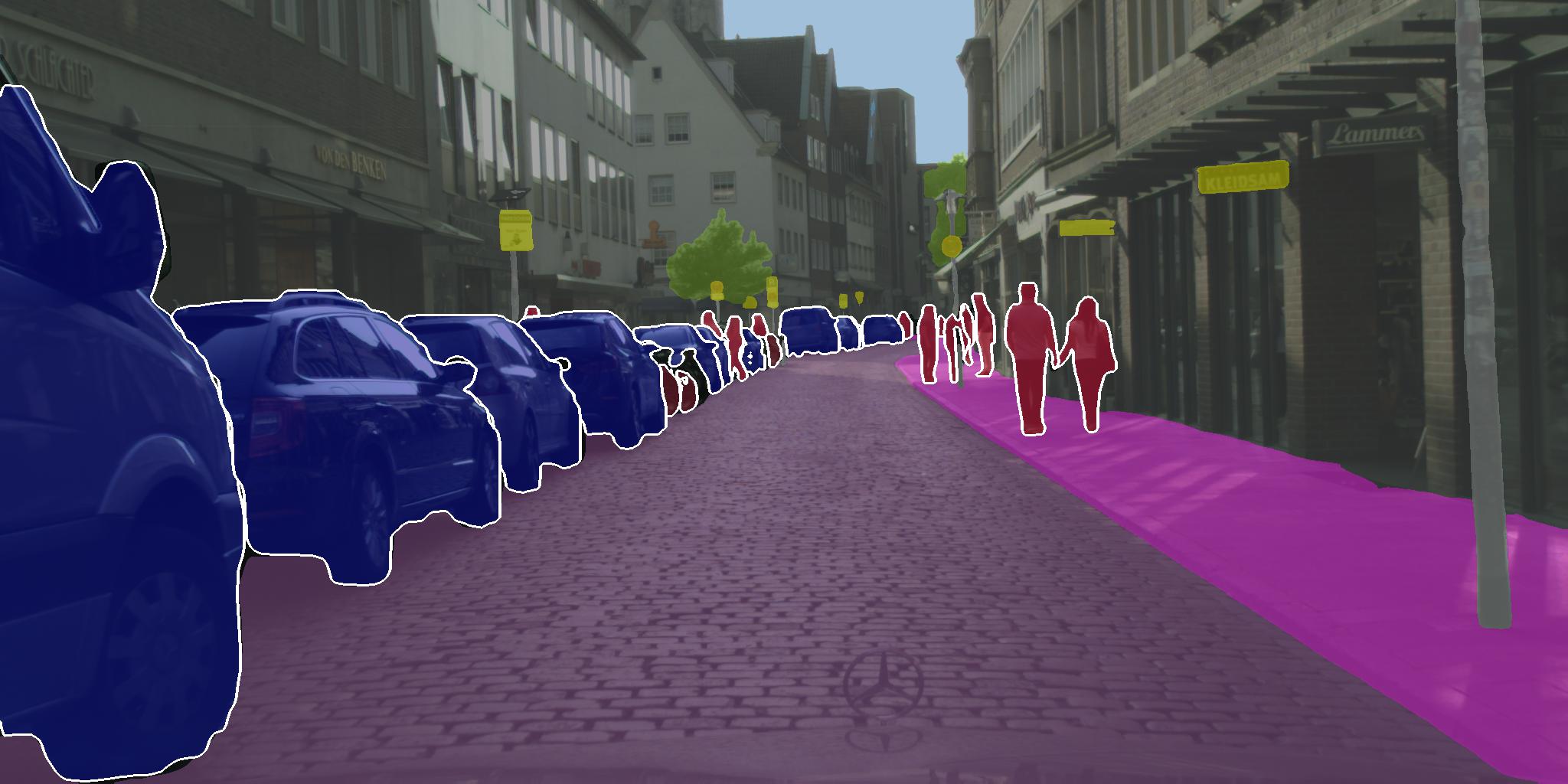}\\
  \caption{Sample outputs of \Crop + \CABB + \Su on Cityscapes. Best viewed on screen.}
  \label{fig:qual-cityscapes}
\end{figure*}

\begin{figure*}
  \centering
  \includegraphics[width=0.33\textwidth]{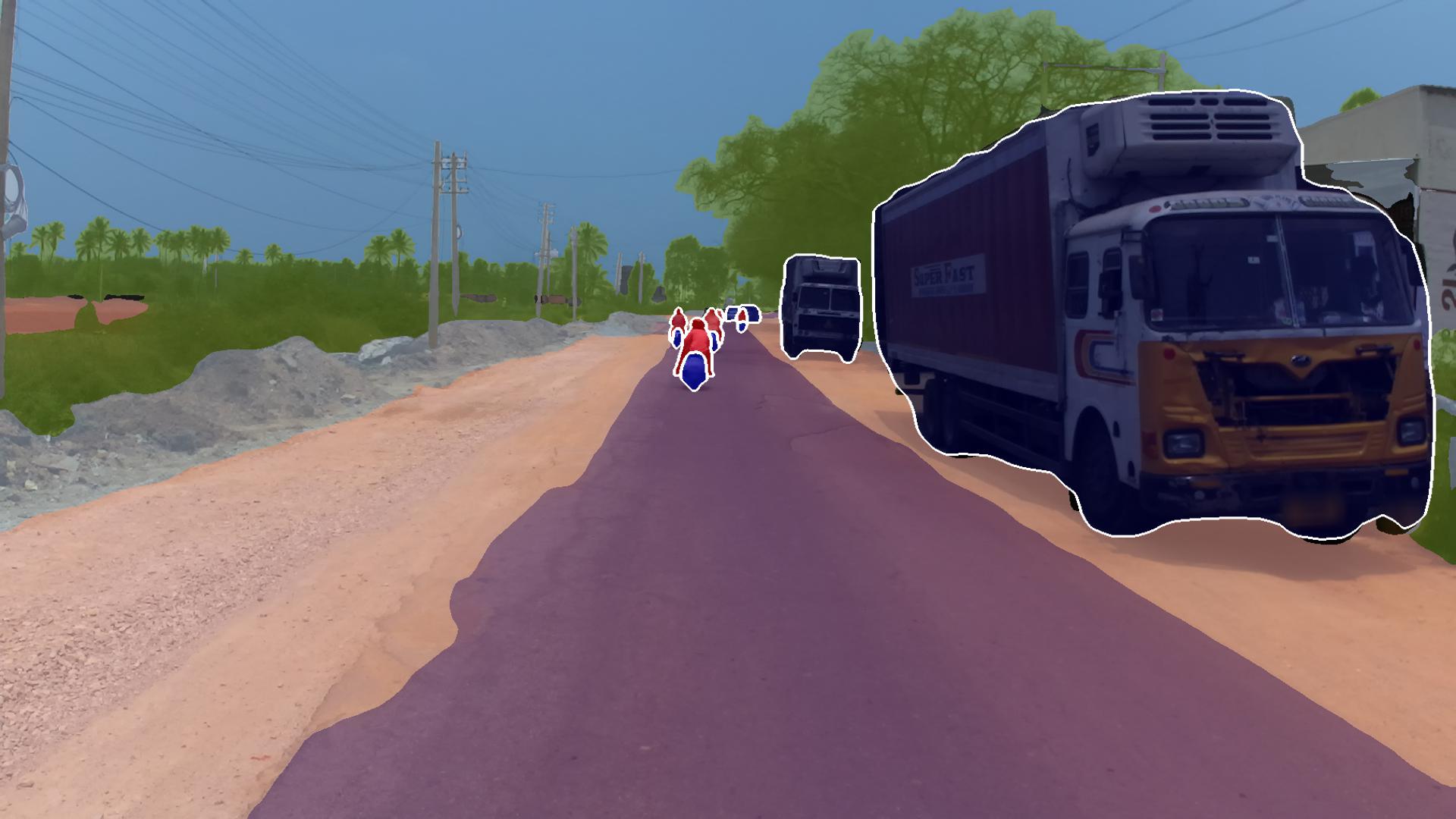}%
  \includegraphics[width=0.33\textwidth]{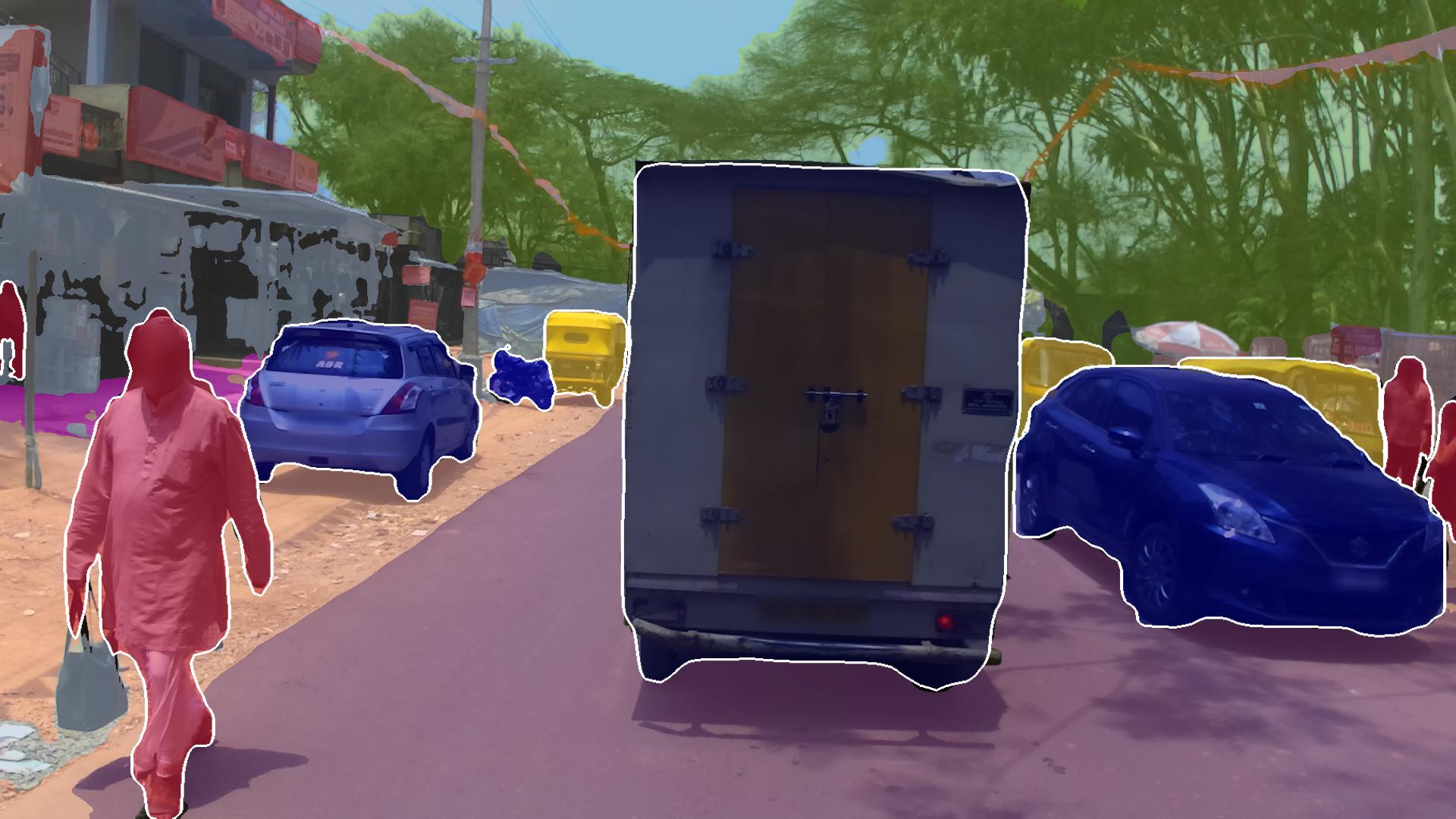}%
  \includegraphics[width=0.33\textwidth]{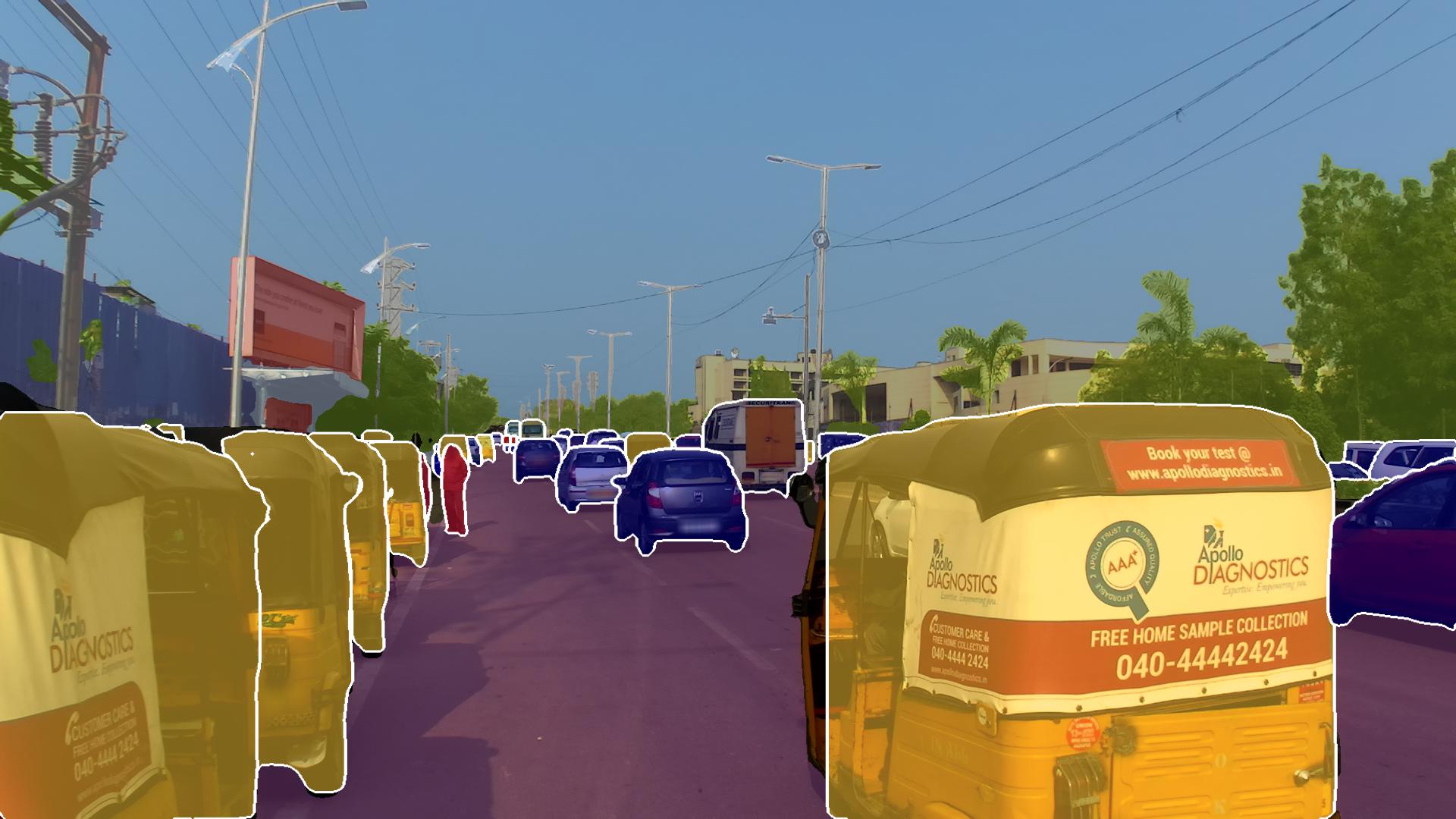}\\
  \includegraphics[width=0.33\textwidth]{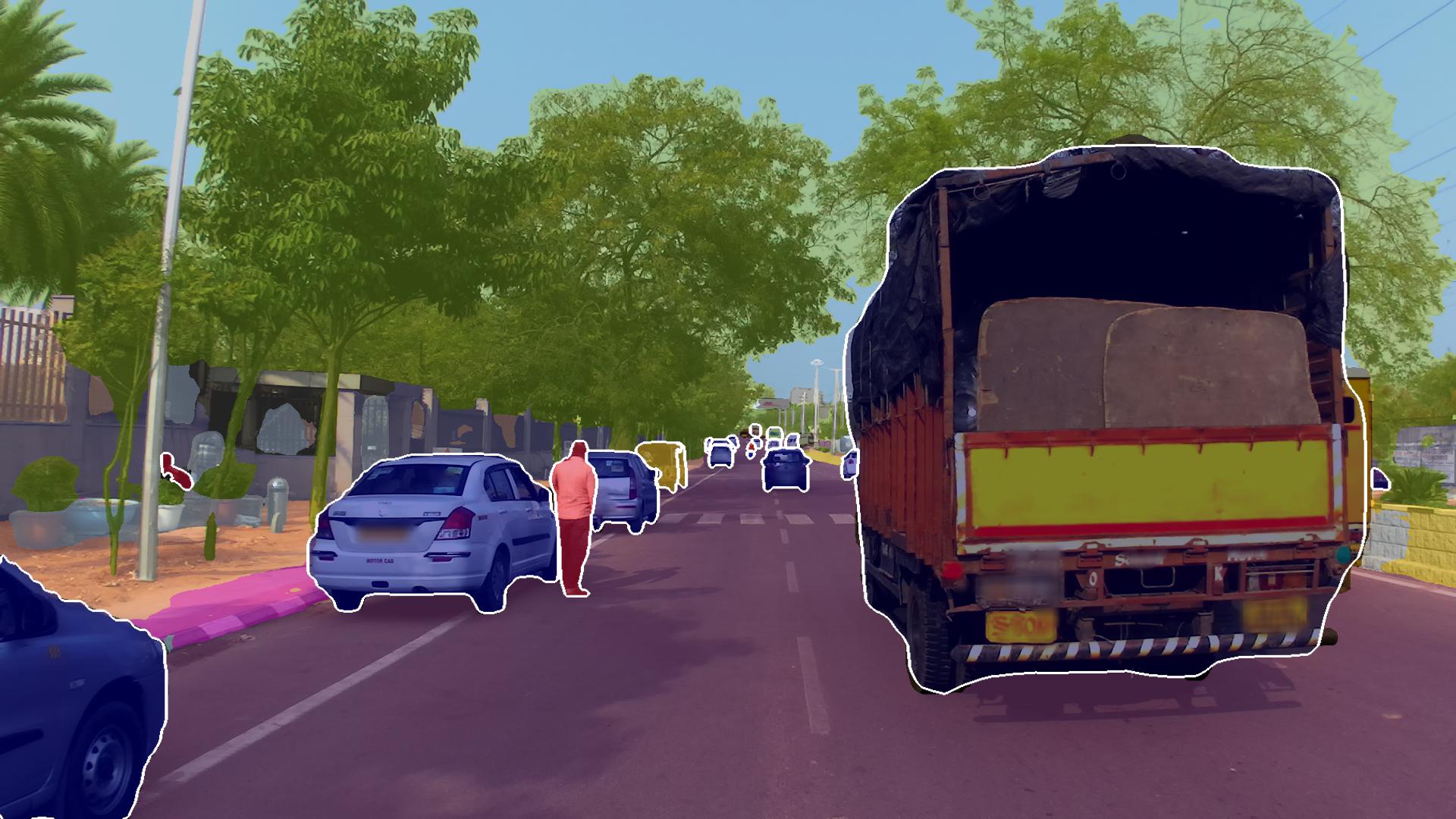}%
  \includegraphics[width=0.33\textwidth]{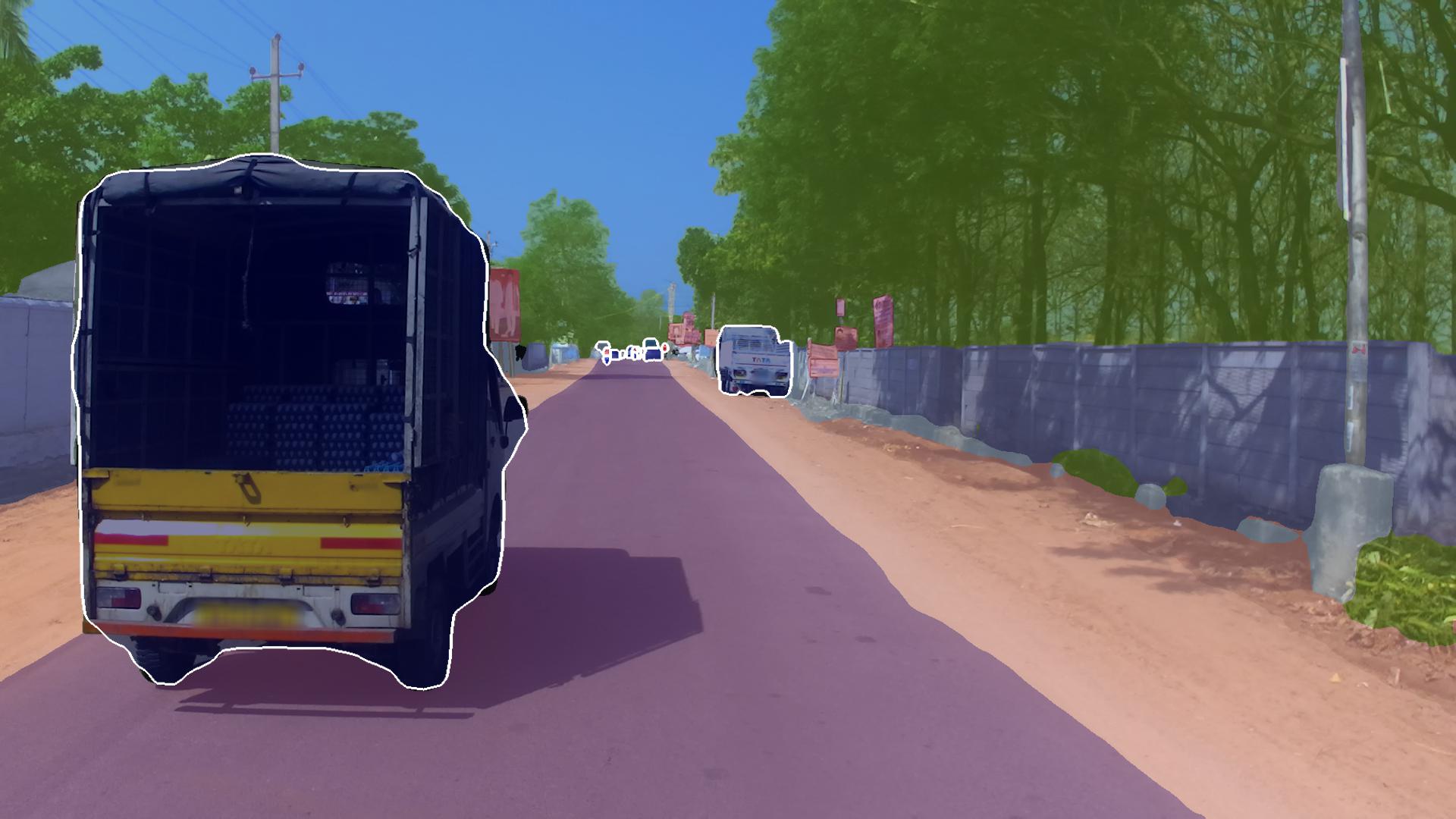}%
  \includegraphics[width=0.33\textwidth]{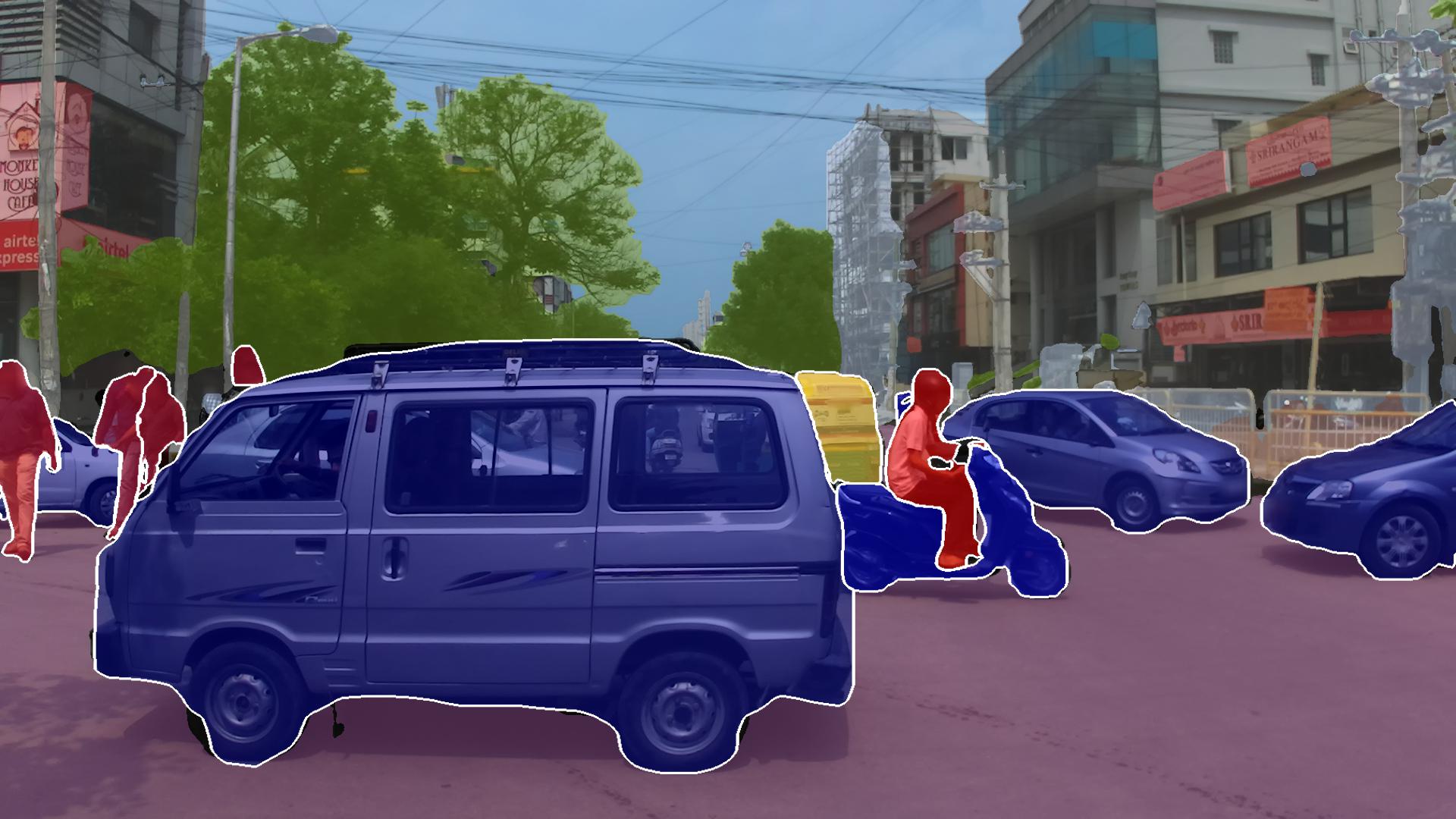}\\
  \includegraphics[width=0.33\textwidth]{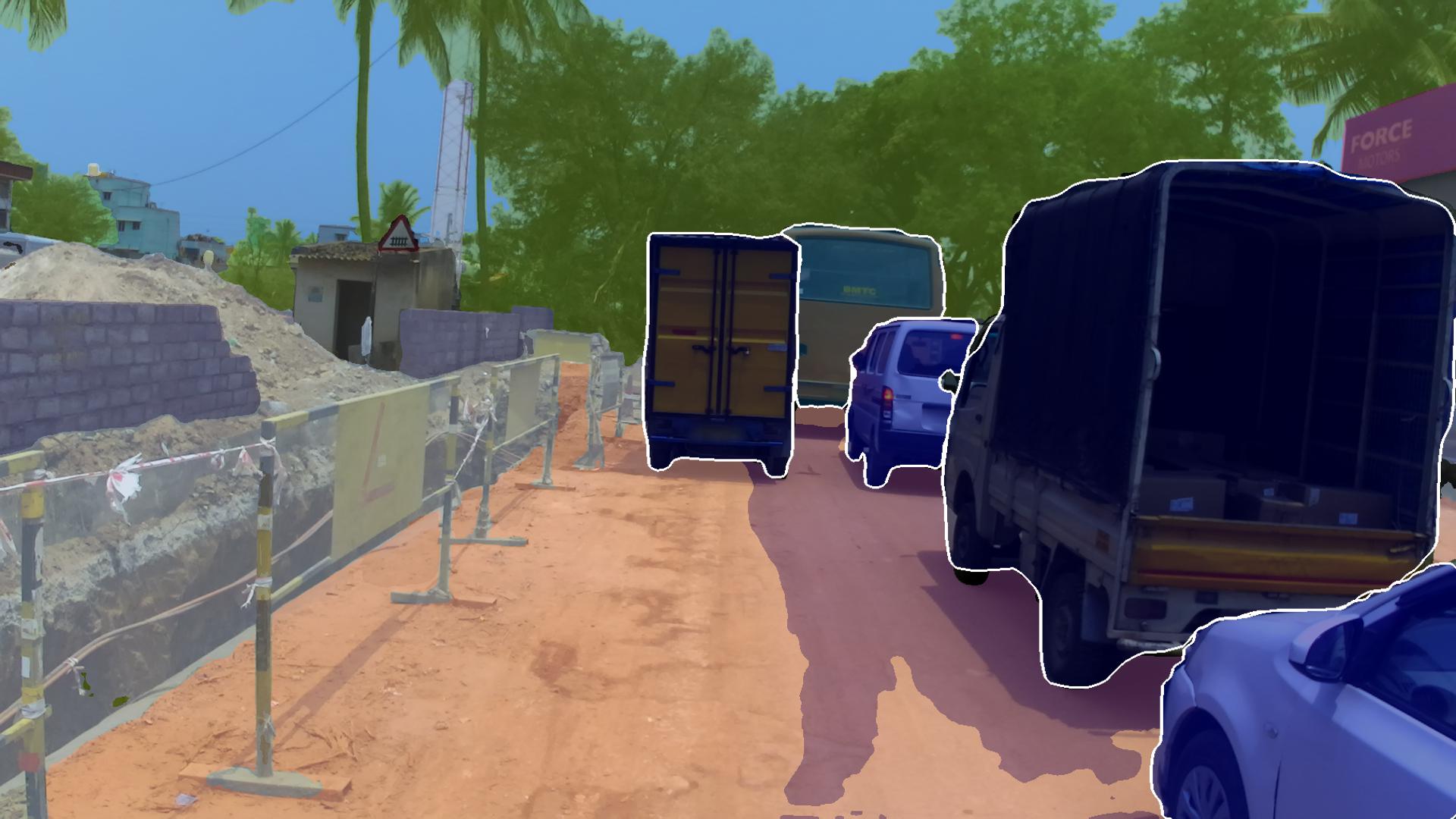}%
  \includegraphics[width=0.33\textwidth]{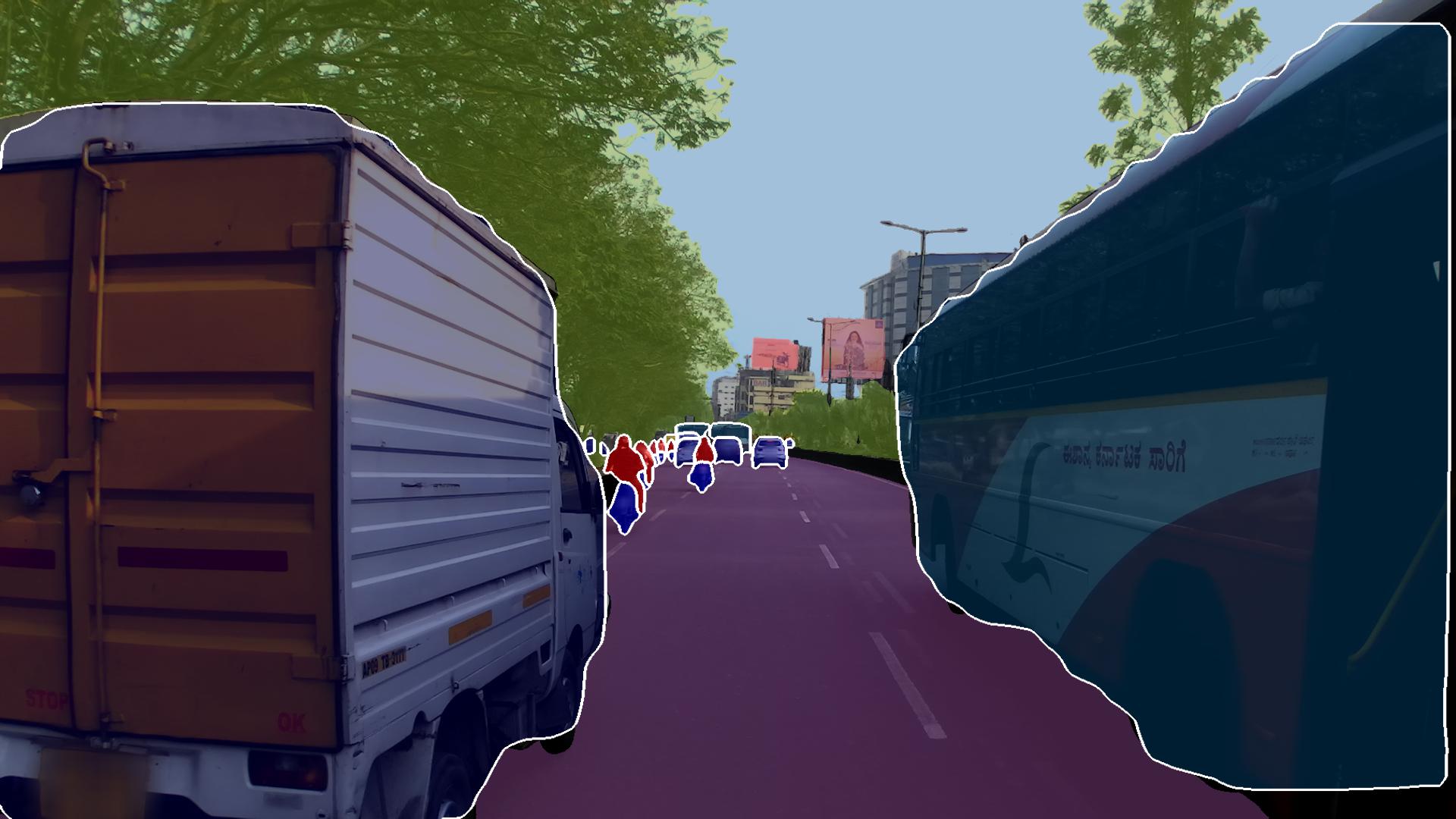}%
  \includegraphics[width=0.33\textwidth]{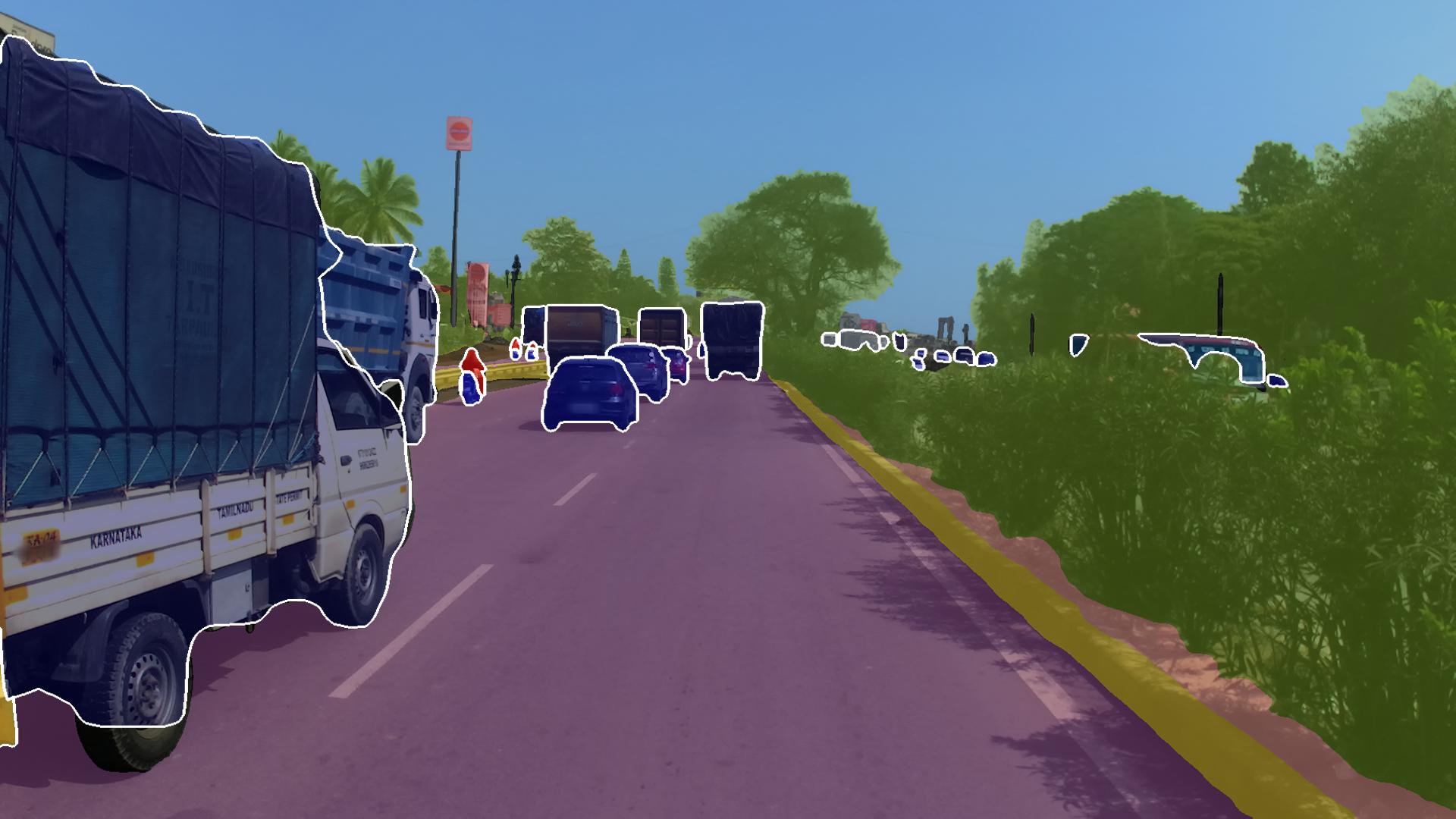}\\
  \caption{Sample outputs of \Crop + \CABB + \Su on the Indian Driving Dataset. Best viewed on screen.}
  \label{fig:qual-idd}
\end{figure*}


{\small
\begin{thebibliography}{10}\itemsep=-1pt

\bibitem{Afr71}
S.~N. Afriat.
\newblock Theory of maxima and the method of lagrange.
\newblock {\em SIAM J. Appl. Math.}, 20(3):343--357, 1971.

\bibitem{Chen2016}
L. {Chen}, G. {Papandreou}, I. {Kokkinos}, K. {Murphy}, and A.~L. {Yuille}.
\newblock Deeplab: Semantic image segmentation with deep convolutional nets,
  atrous convolution, and fully connected crfs.
\newblock {\em (PAMI)}, 40(4):834--848, 2018.

\bibitem{Chen2018ECCV}
Liang-Chieh Chen, Yukun Zhu, George Papandreou, Florian Schroff, and Hartwig
  Adam.
\newblock Encoder-decoder with atrous separable convolution for semantic image
  segmentation.
\newblock In {\em Proceedings of the European Conference on Computer Vision},
  September 2018.

\bibitem{chen2018encoderdecoder}
Liang-Chieh Chen, Yukun Zhu, George Papandreou, Florian Schroff, and Hartwig
  Adam.
\newblock Encoder-decoder with atrous separable convolution for semantic image
  segmentation, 2018.

\bibitem{cheng2019panopticdeeplab}
Bowen Cheng, Maxwell~D. Collins, Yukun Zhu, Ting Liu, Thomas~S. Huang, Hartwig
  Adam, and Liang-Chieh Chen.
\newblock Panoptic-deeplab.
\newblock {\em arXiv:1910.04751}, 2019.

\bibitem{cheng2020panoptic}
Bowen Cheng, Maxwell~D Collins, Yukun Zhu, Ting Liu, Thomas~S Huang, Hartwig
  Adam, and Liang-Chieh Chen.
\newblock Panoptic-deeplab: A simple, strong, and fast baseline for bottom-up
  panoptic segmentation.
\newblock In {\em Proceedings of the IEEE/CVF Conference on Computer Vision and
  Pattern Recognition}, pages 12475--12485, 2020.

\bibitem{CascadePSP2020}
Ho~Kei Cheng, Jihoon Chung, Yu-Wing Tai, and Chi-Keung Tang.
\newblock Cascadepsp: Toward class-agnostic and very high-resolution
  segmentation via global and local refinement.
\newblock In {\em (CVPR)}, 2020.

\bibitem{Cordts2016}
Marius Cordts, Mohamed Omran, Sebastian Ramos, Timo Rehfeld, Markus Enzweiler,
  Rodrigo Benenson, Uwe Franke, Stefan Roth, and Bernt Schiele.
\newblock The {C}ityscapes dataset for semantic urban scene understanding.
\newblock In {\em Proceedings of the IEEE Conference on Computer Vision and
  Pattern Recognition}, 2016.

\bibitem{cordts2016cityscapes}
Marius Cordts, Mohamed Omran, Sebastian Ramos, Timo Rehfeld, Markus Enzweiler,
  Rodrigo Benenson, Uwe Franke, Stefan Roth, and Bernt Schiele.
\newblock The cityscapes dataset for semantic urban scene understanding.
\newblock In {\em Proceedings of the IEEE conference on computer vision and
  pattern recognition}, pages 3213--3223, 2016.

\bibitem{Gao+19}
N. Gao, Y. Shan, Y. Wang, X/ Zhao, Y. Yu, M. Yang, and K. Huang.
\newblock {SSAP}: {S}ingle-shot instance segmentation with affinity pyramid.
\newblock In {\em (ICCV)}, 2019.

\bibitem{Gomez2017}
Aidan~N. Gomez, Mengye Ren, Raquel Urtasun, and Roger~B. Grosse.
\newblock The reversible residual network: Backpropagation without storing
  activations.
\newblock In {\em (NIPS)}, December 2017.

\bibitem{He2017}
Kaiming He, Georgia Gkioxari, Piotr Doll{\'{a}}r, and Ross~B. Girshick.
\newblock Mask {R-CNN}.
\newblock In {\em Proceedings of the IEEE International Conference on Computer
  Vision}, 2017.

\bibitem{IofSze15}
Sergey Ioffe and Christian Szegedy.
\newblock Batch normalization: Accelerating deep network training by reducing
  internal covariate shift.
\newblock {\em CoRR}, abs/1502.03167, 2015.

\bibitem{mlsys2020_196}
Paras Jain, Ajay Jain, Aniruddha Nrusimha, Amir Gholami, Pieter Abbeel, Joseph
  Gonzalez, Kurt Keutzer, and Ion Stoica.
\newblock Breaking the memory wall with optimal tensor rematerialization.
\newblock In {\em Proceedings of Machine Learning and Systems 2020}. 2020.

\bibitem{Kirillov19}
Alexander Kirillov, Ross Girshick, Kaiming He, and Piotr Doll{\'a}r.
\newblock Panoptic feature pyramid networks.
\newblock In {\em (CVPR)}, pages 6399--6408, 2019.

\bibitem{Kirillov18}
Alexander Kirillov, Kaiming He, Ross Girshick, Carsten Rother, and Piotr
  Doll{\'a}r.
\newblock Panoptic segmentation.
\newblock In {\em (CVPR)}, pages 9404--9413, 2019.

\bibitem{Li2018}
Jie Li, Allan Raventos, Arjun Bhargava, Takaaki Tagawa, and Adrien Gaidon.
\newblock Learning to fuse things and stuff.
\newblock {\em CoRR}, abs/1812.01192, 2018.

\bibitem{Li+2019}
Yanwei Li, Xinze Chen, Zheng Zhu, Lingxi Xie, Guan Huang, Dalong Du, and
  Xingang Wang.
\newblock Attention-guided unified network for panoptic segmentation.
\newblock In {\em (CVPR)}, 2019.

\bibitem{Lin_2017_CVPR}
Guosheng Lin, Anton Milan, Chunhua Shen, and Ian Reid.
\newblock Refinenet: Multi-path refinement networks for high-resolution
  semantic segmentation.
\newblock In {\em (CVPR)}, 2017.

\bibitem{micikevicius2018mixed}
Paulius Micikevicius, Sharan Narang, Jonah Alben, Gregory Diamos, Erich Elsen,
  David Garcia, Boris Ginsburg, Michael Houston, Oleksii Kuchaiev, Ganesh
  Venkatesh, and Hao Wu.
\newblock Mixed precision training.
\newblock In {\em (ICLR)}, 2018.

\bibitem{mohan2020efficientps}
Rohit Mohan and Abhinav Valada.
\newblock Efficientps: Efficient panoptic segmentation.
\newblock {\em arXiv preprint arXiv:2004.02307}, 2020.

\bibitem{Neuhold2017}
Gerhard Neuhold, Tobias Ollmann, Samuel Rota~Bul\`o, and Peter Kontschieder.
\newblock The {M}apillary {V}istas dataset for semantic understanding of street
  scenes.
\newblock In {\em (ICCV)}, 2017.

\bibitem{Por+19_cvpr}
Lorenzo Porzi, Samuel {Rota Bul\`o}, Aleksander Colovic, and Peter
  Kontschieder.
\newblock Seamless scene segmentation.
\newblock In {\em Proceedings of the IEEE Conference on Computer Vision and
  Pattern Recognition}, 2019.

\bibitem{Ren+15}
Shaoqing Ren, Kaiming He, Ross Girshick, and Jian Sun.
\newblock Faster {R-CNN}: Towards real-time object detection with region
  proposal networks.
\newblock In {\em (NIPS)}, 2015.

\bibitem{RotPorKon18a}
Samuel Rota~Bul\`o, Lorenzo Porzi, and Peter Kontschieder.
\newblock In-place activated batchnorm for memory-optimized training of {DNN}s.
\newblock In {\em Proceedings of the IEEE Conference on Computer Vision and
  Pattern Recognition}, 2018.

\bibitem{sofiiuk2019adaptis}
Konstantin Sofiiuk, Olga Barinova, and Anton Konushin.
\newblock Adaptis: Adaptive instance selection network.
\newblock In {\em Proceedings of the IEEE International Conference on Computer
  Vision}, pages 7355--7363, 2019.

\bibitem{Varma19}
Girish Varma, Anbumani Subramanian, Anoop Namboodiri, Manmohan Chandraker, and
  C~V Jawahar.
\newblock Indian driving dataset ({IDD}): A dataset for exploring problems of
  autonomous navigation in unconstrained environments.
\newblock In {\em (WACV)}, 2019.

\bibitem{WangSCJDZLMTWLX19}
Jingdong Wang, Ke Sun, Tianheng Cheng, Borui Jiang, Chaorui Deng, Yang Zhao,
  Dong Liu, Yadong Mu, Mingkui Tan, Xinggang Wang, Wenyu Liu, and Bin Xiao.
\newblock Deep high-resolution representation learning for visual recognition.
\newblock {\em TPAMI}, 2019.

\bibitem{Xiong_UBER_2019}
Yuwen Xiong, Renjie Liao, Hengshuang Zhao, Rui Hu, Min Bai, Ersin Yumer, and
  Raquel Urtasun.
\newblock Upsnet: A unified panoptic segmentation network.
\newblock In {\em Proceedings of the IEEE Conference on Computer Vision and
  Pattern Recognition}, pages 8818--8826, 2019.

\bibitem{Yang_Google_2019}
Tien{-}Ju Yang, Maxwell~D. Collins, Yukun Zhu, Jyh{-}Jing Hwang, Ting Liu, Xiao
  Zhang, Vivienne Sze, George Papandreou, and Liang{-}Chieh Chen.
\newblock Deeperlab: Single-shot image parser.
\newblock {\em CoRR}, abs/1902.05093, 2019.

\bibitem{YuanCW20}
Yuhui Yuan, Xilin Chen, and Jingdong Wang.
\newblock Object-contextual representations for semantic segmentation.
\newblock {\em arXiv:1909.11065}, 2020.

\end{thebibliography}
}
\end{document}